\documentclass[11pt]{article}

\usepackage[margin=1in]{geometry}
\usepackage{microtype}

\usepackage{xcolor}
\usepackage{graphicx}
\usepackage{caption}
\usepackage[hyphens]{url}
\usepackage{etoc} 

\usepackage{natbib}

\usepackage{float}   
\usepackage{newfloat}
\usepackage{listings}
\DeclareCaptionStyle{ruled}{labelfont=normalfont,labelsep=colon,strut=off}
\floatstyle{ruled}
\newfloat{listing}{tb}{lst}{}
\floatname{listing}{Listing}

\usepackage{amsmath, amssymb, amsfonts}
\usepackage{amsthm}

\usepackage{mathtools}      
\usepackage{bm}             
\usepackage{nicefrac}       
\usepackage{enumitem}       
\usepackage{xcolor}         
\usepackage{graphicx}       
\usepackage{booktabs}       
\usepackage{microtype}      
\usepackage{url} 
\usepackage{bibentry}         

\usepackage{algorithm}
\usepackage{algorithmicx}
\usepackage{algpseudocode}

\theoremstyle{plain}

\newtheorem{proposition}{Proposition}

\usepackage{multirow}
\theoremstyle{definition}
\newtheorem{definition}{Definition}

\theoremstyle{remark}
\newtheorem{remark}{Remark}

\usepackage[hidelinks]{hyperref}
\usepackage[nameinlink,capitalize,noabbrev]{cleveref}

\newcommand{\SE}{\mathrm{SE}(3)}
\newcommand{\SO}{\mathrm{SO}(3)}

\title{SE(3)-MeanFlow: Few-Step Protein Backbone Generation on Lie Groups}
\author{%
  Yikun Bai$^{1}$ \and
  Binghang Lu$^{2}$ \and
  Yikai Liu$^{3}$ \and
  Elaheh Akbari$^{4}$ \and
  Soheil Kolouri$^{4}$ \and
  Linxuan Wang$^{5}$ \and
  Ping He$^{4}$ \and
  Shuchan Wang$^{6}$ \and
  Ruqi Zhang$^{7}\footnotemark[1]$ \and
  Guang Lin$^{1,8}$\thanks{Corresponding authors: \texttt{ruqiz@purdue.edu, guanglin@purdue.edu}.}\\[0.5em]
  \small $^{1}$Department of Mathematics, Purdue University\\
  \small $^{2}$School of Electrical and Computer Engineering, Purdue University\\
  \small $^{3}$Department of Biochemistry, Purdue University\\
  \small $^{4}$Department of Computer Science, Vanderbilt University\\
  \small $^{5}$Department of Statistics, Purdue University\\
  \small $^{6}$Department of Mathematics and Computer Science, Freie Universit\"at Berlin\\
  \small $^{7}$Department of Computer Science, Purdue University\\
  \small $^{8}$School of Mechanical Engineering, Purdue University\\
}
\date{}

\begin{document}

\maketitle
\begin{abstract}
Generative modeling of protein backbones promises the de novo design of proteins
with prescribed structural and functional properties. Existing diffusion and
flow-matching models produce high-quality backbones on $\SE^N$, but inference
requires numerically integrating an ODE over hundreds of network evaluations,
each involving a Lie group exponential map---a bottleneck for high-throughput
design campaigns. We introduce \textbf{SE(3)-MeanFlow}, a few-step generative
framework that extends MeanFlow from Euclidean space to the Lie group geometry of
protein frames. Working natively in the Lie algebra $\mathfrak{so}(3)$ and in $\mathbb{R}^3$, we
derive closed-form average-velocity identities for rotations and translations,
giving simulation-free training targets. We further introduce an $\SE$ $\alpha$-Flow objective that removes the
Jacobian--vector product from the rotation branch and serves as a warm-up stage,
after which training switches to a small-$t$ stabilized MeanFlow loss that is
used for the remainder of pretraining and for rectification-based post-training.  
In protein backbone generation, SE(3)-MeanFlow matches or exceeds
flow-matching baselines that use several times more sampling steps,
and its advantage widens in the few-step regime, where rectification
lets it lead at every matched budget---at a modest cost in diversity.

\end{abstract}


\section{Introduction}

Proteins are one of the basic building blocks of life. Their complex geometric structure
enables specific inter-molecular interactions that allow for crucial biological
functions---acting as catalysts in chemical reactions, transporters for molecules,
and mediators of immune responses. With the emergence of computational techniques,
it has become possible to rationally design novel proteins with desired structures
that program their functions, opening pathways to solutions for long-standing global
health challenges including influenza~\citep{strauch2017computational},
COVID-19~\citep{cao2020novo} and cancer immunotherapy~\citep{silva2019novo}.

A protein backbone can be modeled as a sequence of $N$ rigid bodies, one per residue,
each associated with a frame under orientation-preserving rigid transformations
---the special Euclidean group $\SE$~\citep{jumper2021highly}. The full backbone
is thus described by the product group $\SE^N$, and the problem of de novo
protein design reduces to sampling from a learned distribution over this space.
Recent work has made substantial progress on generative modeling over $\SE^N$.
Diffusion-based methods such as FrameDiff \citep{yim2023se} and RFDiffusion
\citep{watson2023novo} achieve strong designability, while flow matching
approaches---FoldFlow \citep{bose2024se} and FrameFlow \citep{frameflow}---further
improve training stability and flexibility by learning time-dependent vector fields on
$\SE$ in a simulation-free manner. Despite these advances, all existing
$\SE$ generative models share a common bottleneck: inference still requires numerically integrating an ODE over many steps, typically 100--500 network function
evaluations (NFE). On $\SE$, each step involves evaluating the network and
applying the Lie group exponential map, making inference substantially more expensive
than in Euclidean space. This limits practical deployment in high-throughput drug
discovery pipelines, where millions of candidate structures must be generated per
campaign.

Reducing the number of inference steps has been studied extensively in Euclidean
generative modeling. Consistency models \citep{song2023consistency,song2024improved} and progressive
distillation \citep{salimans2022progressive} compress inference into one or a few steps,
but require a pretrained teacher or carefully staged training curricula.
MeanFlow \citep{geng2026mean} offers a more principled and self-contained alternative.
Rather than modeling the instantaneous velocity $v(t, x_t)$ as in standard flow
matching \citep{lipman2022flow,tong2023improving}, MeanFlow introduces the notion of
\emph{average velocity} $u(s, t, x_t)$---the mean velocity of the flow trajectory over
the interval $[s, t]$. A well-defined identity relating $u$ and $v$ is derived purely
from the definition of average velocity via the product rule and the fundamental theorem
of calculus. This identity yields a tractable, simulation-free training objective that
uses only instantaneous velocity as supervision. At inference, the entire flow path is
approximated in a single network evaluation $u_\theta(0, 1, x_1)$, enabling 1-NFE
generation without distillation or pretraining. However, MeanFlow is formulated in
Euclidean space $\mathbb{R}^d$ and does not account for the non-trivial geometry of
$\SE$. Naively lifting MeanFlow to a curved manifold is ill-posed:
velocities at different points along the trajectory live in different tangent spaces,
and their average requires parallel transport along the path.

Riemannian MeanFlow (denoted as RMF-PT) \cite{zhong2026riemannian} addresses this by extending
MeanFlow to general Riemannian manifolds, defining average velocity via parallel
transport and deriving a corresponding Riemannian MeanFlow identity.
While principled, this general-purpose framework does not exploit the specific algebraic structure of
$\SE$ as a Lie group. Another concurrent work, Riemannian MeanFlow (RMF)~\cite{woo2026riemannian}, defines
the average velocity in the Lie algebra; we compare with it theoretically in
Section~\ref{sec:rmf} and empirically in Section~\ref{sec:experiment}.

In this work, we propose \textbf{SE(3)-MeanFlow}, a few-step generative model for de
novo protein backbone design grounded in the Lie group structure of $\SE$.
Our central insight is that the decomposition
$\SE \cong \SO \times \mathbb{R}^3$ allows the MeanFlow identity
to be derived separately in the Lie algebra $\mathfrak{so}(3)$ and in $\mathbb{R}^3$,
both admitting closed-form, simulation-free training targets without parallel transport.

Our main contributions are:
\begin{itemize}
 \item \textbf{Integration-based SE(3)-MeanFlow with theory.} We propose an integration-based
    SE(3)-MeanFlow formulation conditioned on the current state $(R_t,x_t)$, and derive
    closed-form average-velocity identities in the Lie algebra $\mathfrak{so}(3)$ and in
    $\mathbb{R}^3$, without parallel transport. We provide theoretical validation showing that
    the resulting training objectives are well-defined and that the corresponding losses are
    consistent with the intended MeanFlow targets on $\SE$.

    \item \textbf{Stable training algorithms for proteins.} We introduce an $\SE$
    $\alpha$-Flow objective together with practical numerical stabilization techniques that make training and few-step generation reliable for
    protein backbones.
  \item \textbf{Protein design results.} On the SCOPe benchmark, our approach
matches or exceeds flow-matching baselines that use several times more sampling
steps, with the largest gains in the few-step regime, in both pretraining and
post-training settings, at a modest cost in diversity.
\end{itemize}


\section{Background and Notations}

\subsection{Lie-group notation on SE(3)}
In this work, the $\SE$ space is a collection of rigid motions in $\mathbb{R}^3$:
$$
\SE\cong \SO\times\mathbb{R}^3,
$$
where $\SO$ denotes the space of $3\times3$ rotation matrices, which forms a Lie group:
$$
\SO=\{R\in\mathbb{R}^{3\times 3}:RR^\top=I,\ \det(R)=1\}.
$$
We follow \citet{bose2024se} and use the decoupled $\SO\times\mathbb{R}^3$ geometry when defining metrics and losses.

A (continuous-time) flow on $\SO$ can be defined by a curve $R_t\in \SO$ satisfying the left-trivialized ODE
\begin{equation}
\dot R_t = R_t\,\Omega_t,\qquad \Omega_t\in \mathfrak{so}(3),
\label{eq:so3_ode}
\end{equation}
where $\mathfrak{so}(3)$ denotes the corresponding Lie algebra
$$
\mathfrak{so}(3)=\{\Omega\in\mathbb{R}^{3\times 3}:\Omega^\top=-\Omega\}.
$$

$\mathfrak{so}(3)$ is isomorphic to $\mathbb{R}^3$ space, and we use the hat/vee maps $(\cdot)^\wedge,(\cdot)^\vee$ to identify $\mathbb{R}^3\leftrightarrow\mathfrak{so}(3)$ (e.g., $\omega\in\mathbb{R}^3\mapsto\omega^\wedge\in\mathfrak{so}(3)$).  

Given the angular velocity field $\Omega_t:[0,1]\to\mathfrak{so}(3)$ along the trajectory, the solution of \eqref{eq:so3_ode} can be written as
$$
R_t = R_0\,\mathcal{T}\exp\Big(\int_0^t \Omega_\tau\,d\tau\Big),
$$
where $\mathcal{T}\exp$ composes the infinitesimal rotations along time. Here $\exp$ denotes the (matrix) exponential map from the Lie algebra to the Lie group. In particular, if $\Omega_\tau\equiv\Omega$ is constant, then
$$
R_t = R_0\exp(t\Omega).
$$

For a detailed discussion of $\SO$/$\SE$ flows (including the closed-form/geodesic expressions under common metrics), see Appendix~\ref{sec:bg-so3}\footnote{All appendices referenced in this paper are provided in the technical supplement.}.

\subsection{Protein Backbone Parametrization}
Following the setting in \citet{bose2024se}, we parametrize the protein backbone as a sequence of $N$ rigid bodies. Each residue $i \in [N]$ is associated with a frame $T_i = (R_i, x_i) \in \SE$, where $R_i \in \SO$ represents the orientation and $x_i \in \mathbb{R}^3$ denotes the position of the $C_\alpha$ atom.

The 3D coordinates of the backbone atoms $\{N, C_\alpha, C, O\}_i$ are recovered by applying the rigid transformation $T_i$ to a set of idealized coordinates $\{N^*, C_\alpha^*, C^*, O^*\}$:
\begin{equation}
  [N, C_\alpha, C, O]_i = T_i \circ [N^*, C_\alpha^*, C^*, O^*],
\end{equation}
where $C_\alpha^* = (0, 0, 0)$ is fixed at the origin.
For a self-contained review of $\SO$/$\SE$ geometry and notation (including the definitions of $(\cdot)^\wedge$ and $(\cdot)^\vee$), see Appendix~\ref{sec:bg-so3}--\ref{sec:bg-se3}.

\subsection{MeanFlow in Euclidean space}
One limitation of standard flow matching (Appendix~\ref{sec:bg-fm}) is that, to guarantee accuracy, we generally require the step size $\frac{1}{T}$ to be sufficiently small, equivalently requiring large inference steps $T$. To address this issue, in Euclidean space, \citet{geng2026mean,geng2026improved} propose the MeanFlow method.  
In particular, given a probability path $p_t$ generated by velocity $v_t$, it defines the average velocity:  
$$(t-s)v^{\text{avg}}(s,t,x_t)=\int_s^t v(\tau,x_\tau)d\tau.$$

It yields the training loss:
{\footnotesize 
\begin{equation}
\mathbb{E}
\left[\left\|v^{\text{avg}}_\theta(s,t,x_t)
+(t-s)\,\text{sg}\!\left(\frac{d}{dt}\hat{v}^{\text{avg}}(s,t,x_t)\right)
-v(t,x_t)\right\|^2\right].
\label{eq:mf_loss}
\end{equation}
}
where the randomness is given by 
$t,s\sim \text{Unif}([0,1]),s\leq t$ and $(x_0,x_1)\sim q$ for some joint distribution $q$. By default, $q$ is an independent coupling between $p_{\text{data}}$ and $p_{\text{prior}}$. 
\begin{equation}
\frac{d}{dt}\hat{v}^{\text{avg}}(s,t,x_t)=
\frac{\partial }{\partial t}v^{\text{avg}}_\theta(s,t,x_t)+\nabla_x v^{\text{avg}}_\theta(s,t,x_t)\cdot v_t.
\label{eq:mf_cases}
\end{equation}
In addition, $v_t$ can be replaced by $v^{\text{avg}}_\theta(t,t,x_t)$ in \citet{geng2026improved}.

Here $x_t=(1-t)x_0+tx_1$, $v(t,x_t)=x_1-x_0$, and sg is the stop-gradient operator.
Intuitively, the integration is estimated by the parameterized model during the training process; thus, in the inference step, it yields a T-step generation, where $T\in\mathbb{N}$: 
\begin{equation}
\hat{x}_{s}=x_t-\frac{1}{T} v^{\text{avg}}_\theta(s,t,x_t),\quad t=1-i/T,\ s=t-1/T,\ i=0,1,\ldots,(T-1).
\label{eq:mf_inference}
\end{equation}
When we set $T=1$, it yields a one-step generation. 

\begin{figure}[t]
  \centering
  \includegraphics[width=0.6\linewidth]{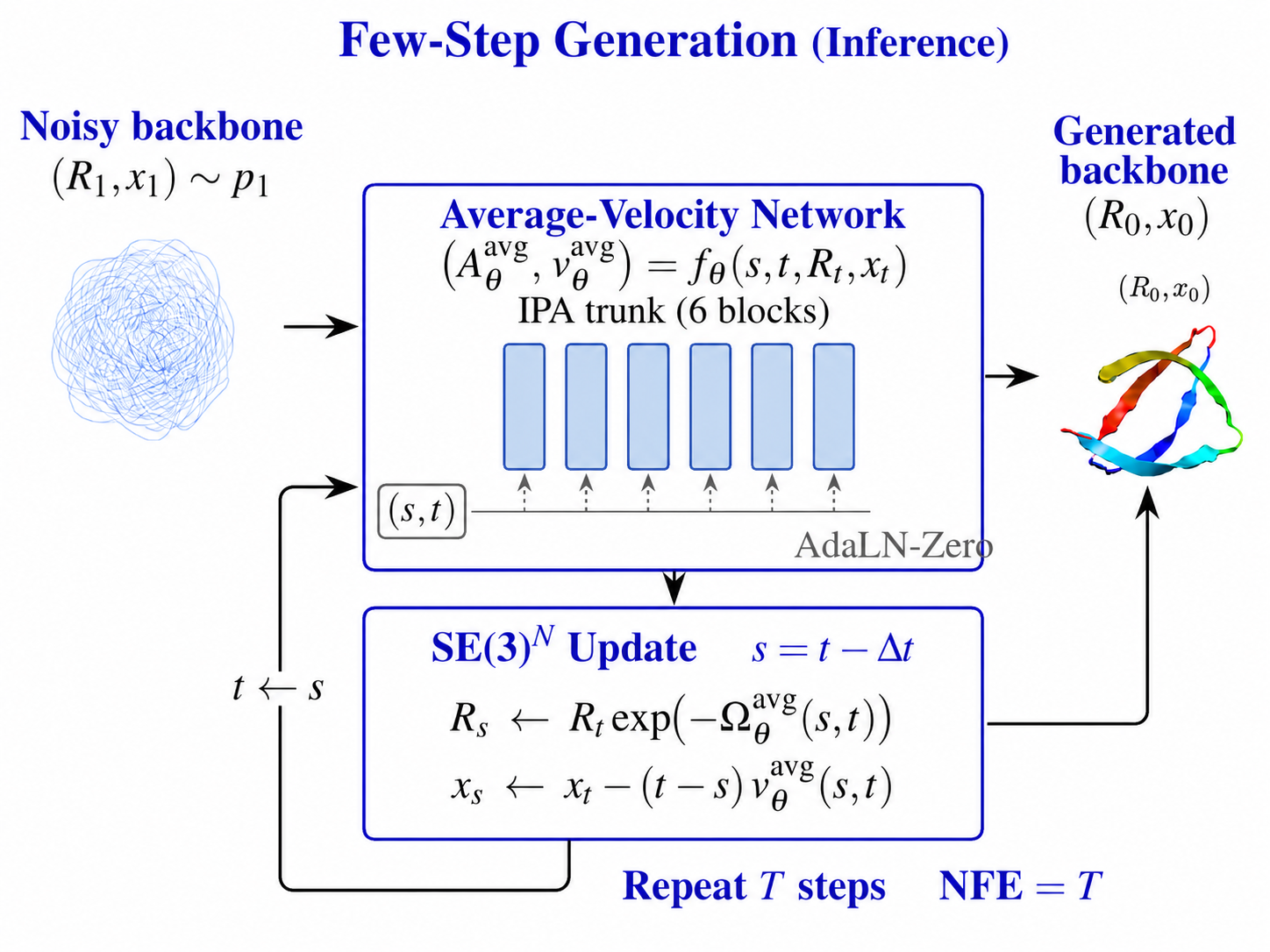}
\caption{Few-step inference with SE(3)-MeanFlow. The two-time network $f_\theta$ (see Appendix \ref{sec:architecture}) 
predicts the average velocity over $[s,t]$ rather than the instantaneous velocity
at $t$, so one evaluation advances the state by the full step---a Lie group
exponential on $\SO$, a Euclidean update on $\mathbb{R}^3$---giving
$\mathrm{NFE}=T$. The objective that makes this accurate is
Proposition~\ref{pro:A_identity_main} and its JVP-free $\alpha$-Flow surrogate
(Section~\ref{sec:alphaflow-main}).}
  \label{fig:overview}
\end{figure}

\section{Method: MeanFlow in SE(3) space}
\label{sec:method}

Since $\SE=\SO\times\mathbb{R}^3$, the $\mathbb{R}^3$ component reduces to the conventional Euclidean MeanFlow; hence, we focus on introducing the MeanFlow on $\SO$ below. \\
Let a distribution path $\{p_t\in\mathcal{P}(\SO):t\in[0,1]\}$ be generated by a
velocity field $v_t:[0,1]\times\SO\to\mathcal{T}\SO$. We define the
\emph{average velocity} over $[s,t]$ through the time-ordered exponential
{\footnotesize
\begin{align}
\exp\!\big((t-s)\,\Omega^{\text{avg}}(s,t,R_t,x_t)\big)
:=\mathcal{T}\exp\!\Big(\int_s^t\Omega(\tau,R_\tau,x_\tau)\,d\tau\Big).
\label{eq:avg_omgea_main}
\end{align}}
We denote the instantaneous (body-frame) angular velocity (vector form) by
\begin{align}
  \omega_t := \Omega(t,R_t,x_t)^{\vee}\in\mathbb{R}^3.
  \label{eq:omega_t-def}
\end{align}
Analogous to Euclidean MeanFlow, differentiating \eqref{eq:avg_omgea_main} gives an identity
linking the average velocity, its trajectory derivative, and the instantaneous velocity.

\begin{proposition}\label{pro:A_identity_main}
Differentiating \eqref{eq:avg_omgea_main} with respect to $t$ yields
{\footnotesize
\begin{align}
&J\big((t-s)A^{\text{avg}}\big)\Big(A^{\text{avg}}(s,t,R_t,x_t)+(t-s)\tfrac{d}{dt}A^{\text{avg}}\Big)
=\omega_t,\label{eq:omgea_identity_main}\\
&\text{where }A^{\text{avg}}=(\Omega^{\text{avg}})^\vee,\nonumber\\
&\qquad\frac{d}{dt}A^{\text{avg}}=\frac{\partial}{\partial t}A^{\text{avg}}
+\langle\nabla_R A^{\text{avg}},\dot{R}_t\rangle+\langle \nabla_x A^{\text{avg}},\dot{x}_t \rangle,
\label{eq:dA/dt}
\end{align}}
and the right Jacobian $J:\mathbb{R}^3\to\mathbb{R}^{3\times3}$ is
{\footnotesize
\begin{align}
J(A):=I-\frac{1-\cos\|A\|_2}{\|A\|_2^2}A^{\wedge}
+\frac{\|A\|_2-\sin\|A\|_2}{\|A\|_2^3}(A^{\wedge})^2,\label{eq:right_jacobian}
\end{align}}
with $\|A\|_2$ the $\ell_2$ norm of $A\in\mathbb{R}^3$; as $\|A\|\to0$, $J(A)\to I$ (we use its
Taylor expansion there). 
\end{proposition}
\noindent This proposition is split into Propositions~\ref{pro:Omega_identity}
and~\ref{pro:dA/dt} in Appendix~\ref{sec:SE3MF-details}, with proofs.

\subsection{Model parametrization and training loss}
\label{sec:model-param}
The network exposes two interchangeable heads: an \emph{endpoint} ($x$-)head predicting the
clean state $(\hat R_0^\theta,\hat x_0^\theta)$, and an \emph{average-velocity} ($u$-)head
$A_\theta^{\text{avg}}:[0,1]\times[0,1]\times\SO\times\mathbb{R}^3\to\mathbb{R}^3$ and
$v_\theta^{\text{avg}}$, with $\Omega_\theta^{\text{avg}}=(A_\theta^{\text{avg}})^\wedge$. The two
carry the same information, related by the invertible map
\begin{align}
&A_\theta^{\text{avg}}=\tfrac1t\log\!\big((\hat R_0^\theta)^\top R_t\big)^{\vee}
\!\iff\!\hat R_0^\theta=R_t\exp\!\big(-t\,A_\theta^{\text{avg}\wedge}\big),\nonumber
\\
&v_\theta^{\text{avg}}=\tfrac1t\big(x_t-\hat x_0^\theta\big)
\!\iff\!\hat x_0^\theta=x_t-t\,v_\theta^{\text{avg}},\nonumber
\end{align}
which, on the constant-velocity ($x$-prediction) geodesic, is consistent with the
$(t-s)$-normalized average of \eqref{eq:avg_omgea_main}. The losses below regress the
velocity head; the endpoint head is produced at inference.

Based on \eqref{eq:omgea_identity_main}, we derive two equivalent mean flow training losses:
\begin{align}
&\mathcal{L}^J_{\SO}:=\mathbb{E}_{\substack{s<t\\ q(R_0,R_1)}}
\Big[\big\|\omega_\theta-\omega_t\big\|_2^2\Big],\label{eq:mf-se3-loss}\\
&\quad \omega_\theta=\text{sg}\big(J((t-s)A_\theta^{\text{avg}})\big)\Big(A_\theta^{\text{avg}}+(t-s)\,\text{sg}(\tfrac{d}{dt}A_\theta^{\text{avg}})\Big),\nonumber\\
&\mathcal{L}^{J^{-1}}_{\SO}:=\mathbb{E}\left[\|A_\theta^{\text{avg}}-\text{sg}(A_\theta^{\mathrm{tgt}})\|^2\right]\label{eq:mf-se3-loss2}\\
&\quad A_\theta^{\mathrm{tgt}}=J^{-1}((t-s)A_\theta^{\text{avg}})\omega_t-(t-s)\frac{d}{dt}A_\theta^{\text{avg}}\nonumber 
\end{align}
where all model arguments are $(s,t,R_t,x_t)$ and $\tfrac{d}{dt}A_\theta^{\text{avg}}$ is
given by \eqref{eq:dA/dt}.  Loss
\eqref{eq:mf-se3-loss2} is well-defined since $J$ is invertible on the principal branch
$\|(t-s)A_\theta^{\text{avg}}\|_2\le\pi$. We refer to Appendix~\ref{sec:SE3MF-details}
for details. 

\paragraph{Manifold derivatives via Euclidean JVPs.}
Although $R_t\in\SO$ is manifold-valued, $\tfrac{d}{dt}A_\theta^{\text{avg}}$ is a
directional derivative along the interpolation trajectory $(R_t,x_t)$, which we evaluate with
a standard Euclidean Jacobian--vector product (e.g.\ \texttt{torch.jvp}) through the chosen
matrix representation of $R_t$. A formal equivalence is given in
Appendix~\ref{sec:SE3MF-details}, Proposition \ref{pro:jvp_equivalent}.

\begin{proposition}[Correctness of the MeanFlow objective, informal]
\label{prop:mf-loss-correctness}
If $\mathcal{L}_{\SO}$ in \eqref{eq:mf-se3-loss} or \eqref{eq:mf-se3-loss2} is zero, the model recovers the correct
relative rotations along the path:
$\exp\!\big(((t-s)A_\theta^{\text{avg}}(s,t,R_t,x_t))^\wedge\big)=R_s^\top R_t$ for all
$0\le s<t\le1$.
\end{proposition}
\noindent A formal statement and proof are in Proposition~\ref{pro:loss_zero} (Appendix).

\paragraph{Inference.}
We step backward using the learned average flow,
{\footnotesize$$
R_{s}\gets R_t\exp\!\big(-(t-s)\,\Omega_\theta^{\text{avg}}(s,t,R_t,x_t)\big),
\qquad t=1,\ldots,\tfrac1T,
$$}
and analogously $x_s\gets x_t-(t-s)\,v_\theta^{\text{avg}}$. See Figure \ref{fig:overview}. 
Reported results use a variant of this update, the exponential rotation schedule of
Appendix~\ref{sec:stable-rescale}.

\subsection{Practical implementation: SE(3)N adaptation}
Each protein sample is $R^N\in\mathbb{R}^{N\times3\times3}$, $x^N\in\mathbb{R}^{N\times3}$, with $N$
the number of residues. The model is
$f_\theta(s,t,R^N_t,x^N_t)=[A_\theta^{\text{avg},N};\,v_\theta^{\text{avg},N}]
\in\mathbb{R}^{N\times3}\times\mathbb{R}^{N\times3}$, where the $i$-th block
$A_{\theta,i}^{\text{avg},N}\in\mathbb{R}^3$ parametrizes the $\mathfrak{so}(3)$ component via
$\Omega_{\theta,i}^{\text{avg}}=(A_{\theta,i}^{\text{avg},N})^\wedge$. The loss
\eqref{eq:mf-se3-loss} (or \eqref{eq:mf-se3-loss2}) is summed over residues.

\subsection{\texorpdfstring{$\SE$}{SE(3)} \texorpdfstring{$\alpha$}{alpha}-Flow: rotation formulation and MeanFlow limit}
\label{sec:alphaflow-main}
The differential target \eqref{eq:omgea_identity_main} requires the trajectory derivative
$\tfrac{d}{dt}A^{\text{avg}}_\theta$ via a JVP, which could be fragile on the rotation branch. We
therefore also adopt an $\alpha$-Flow formulation~\cite{zhang2025alphaflow} that replaces this
derivative by a two-segment construction using only forward evaluations of
$A^{\text{avg}}_\theta$.

Given $0\le s<t\le1$ and a ratio $\alpha\in(0,1]$, insert an intermediate time
$m=\alpha s+(1-\alpha)t$ with step $\delta:=t-m=\alpha(t-s)$, splitting $[s,t]$ into a
\emph{far} segment $[s,m]$ and a \emph{near} segment $[m,t]$. With
$D(a,b):=R_a^\top R_b\in\SO$ the accumulated relative rotation, the segments compose
\emph{multiplicatively} (rather than additively as in Euclidean space),
\begin{align}
D(s,t)=D(s,m)\,D(m,t)
\qquad\text{(Appendix~\ref{sec:alpha-flow}, Prop.~\ref{prop:af-additivity})}.\nonumber
\end{align}
We anchor the near factor to data and bootstrap the far factor from the model. Stepping back
from $R_t$ along the data angular velocity $\omega_t:=\Omega(t,R_t,x_t)^\vee$ gives the intermediate
state $R_m=R_t\exp(-\delta\,\omega_t^{\wedge})$ and the near factor $D(m,t)=\exp(\delta\,\omega_t^{\wedge})$;
a stop-gradient model query at $(s,m,R_m,x_m)$ returns $A_m:=A^{\text{avg}}_\theta(s,m,R_m,x_m)$
and the far factor $D(s,m)=\exp((m-s)A_m^{\wedge})$. Composing and mapping back to the Lie
algebra gives the target average generator over $[s,t]$:
\begin{align}
A^{\text{avg}}_{\text{tgt}}(s,t)
=\frac{1}{t-s}\,\log\!\Big(
\underbrace{\exp\!\big((m-s)A_m^{\wedge}\big)}_{D(s,m)\ (\text{model})}\;
\underbrace{\exp\!\big(\delta\,\omega_t^{\wedge}\big)}_{D(m,t)\ (\text{data})}
\Big)^{\vee}.
\label{eq:alphaflow-rot-tgt-main}
\end{align}
The scalar $\tfrac{1}{t-s}$ must stay outside the $\log(\exp\cdot\exp)$: since $A_m$ and $\omega_t$ do
not commute, folding it in would corrupt the Baker--Campbell--Hausdorff cross term and no longer
give the generator of $D(s,t)$. We regress
\begin{align}
\mathcal{L}^\alpha_{\text{rot}}
=\frac{1}{\alpha}\sum_{i=1}^{N}
\big\|A^{\text{avg}}_{\theta,i}(s,t,R_t,x_t)-\mathrm{sg}\big(A^{\text{avg}}_{\text{tgt},i}\big)\big\|_2^2 ,
\label{eq:alphaflow-rot-loss-main}
\end{align}
which uses only forward evaluations (one log, two exps; no JVP). The numerically stabilized
form and the abelian translation branch --- where composition reduces to a convex combination
of velocities as in the Euclidean $\alpha$-Flow~\cite{zhang2025alphaflow} --- are given in
Appendix~\ref{sec:af-trans}.

The ratio $\alpha$ interpolates between flow matching and MeanFlow: at $\alpha=1$ the far factor
vanishes ($m=s$) and $A^{\text{avg}}_{\text{tgt}}=\omega_t$, while as $\alpha\to0$ a first-order BCH
expansion recovers the differential MeanFlow objective \eqref{eq:omgea_identity_main} in
gradient (Appendix~\ref{sec:af-meanflow-limit}).

\section{Experiments}\label{sec:experiment}
We evaluate our method on unconditional protein backbone generation and compare it against recent diffusion- and flow-based baselines. Our experiments are designed to assess both generation quality and sampling efficiency, with a particular focus on few-step generation. To this end, we report performance under different numbers of sampling steps, allowing us to examine how well each method maintains designability, diversity, and novelty as the computational budget is reduced.

\subsection{Datasets}

Following prior works~\cite{yue2025reqflow,woo2026riemannian}, we conduct experiments on the SCOPe dataset~\cite{chandonia2022scope,frameflow}, which consists of 3,673 preprocessed protein backbones with residue lengths between 60 and 128. We refer to Figure \ref{fig:scope-length-hist} in the Appendix for a visualization of the backbone length-frequency distribution.

\begin{figure}[t]
  \centering
  \includegraphics[width=0.49\linewidth]{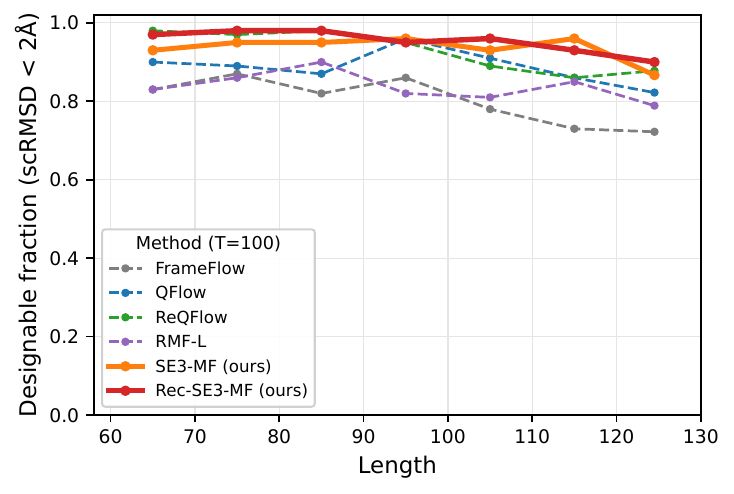}\hfill
\includegraphics[width=0.49\linewidth]{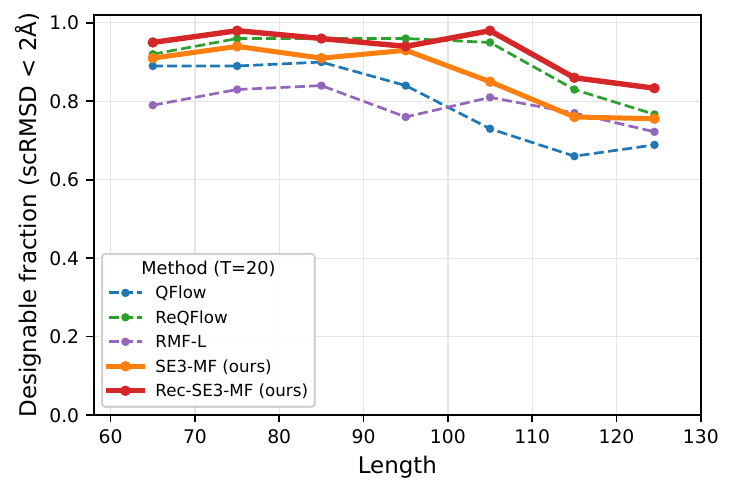}
  \caption{Designable fraction versus length (top: $T{=}100$, bottom: $T{=}20$).}
  \label{fig:fraction-vs-length-T100-T20}
\end{figure}

\begin{figure}[t]
  \centering
  \includegraphics[width=0.49\linewidth]{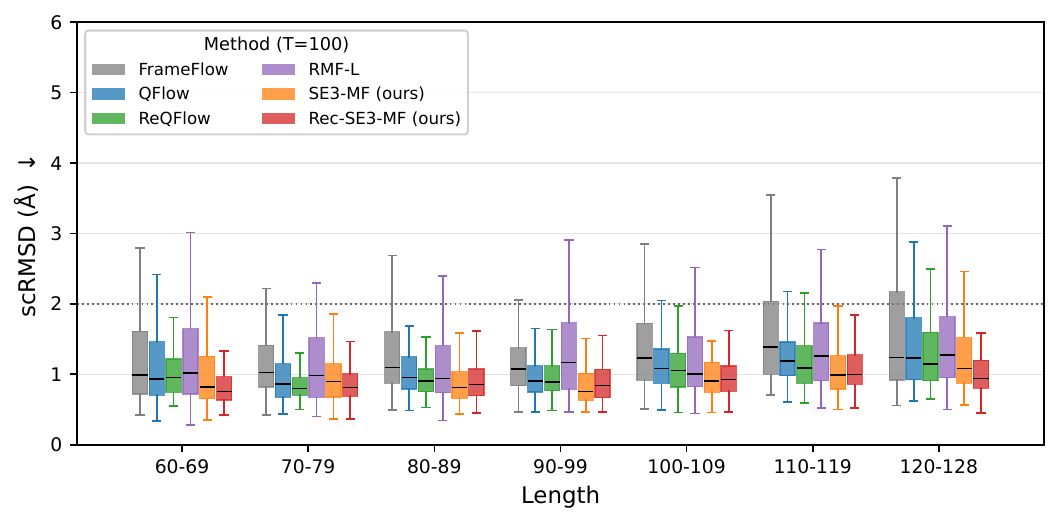}\hfill
\includegraphics[width=0.49\linewidth]{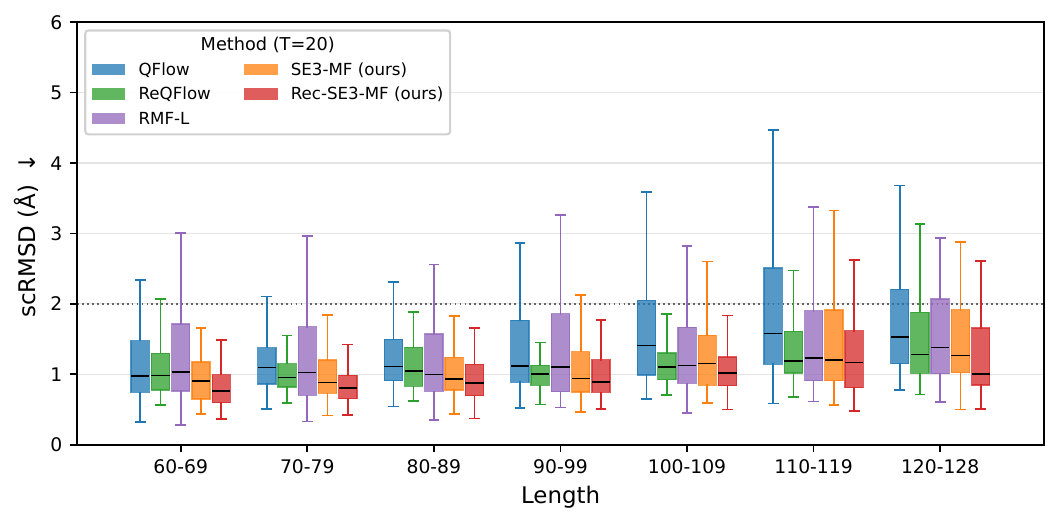}
  \caption{scRMSD distributions (top: $T{=}100$, bottom: $T{=}20$).}
  \label{fig:scrmsd-boxplot-T100-T20}
\end{figure}

\begin{table*}[t]
\centering
\begingroup
\small
\setlength{\tabcolsep}{5pt}
\begin{tabular}{@{}llccccc@{}}
\toprule
\multirow{2}{*}{Steps} & \multirow{2}{*}{Method}
& \multicolumn{3}{c}{Designability}
& \multirow{2}{*}{Diversity $\mathrm{TM}\downarrow$}
& \multirow{2}{*}{Novelty $\mathrm{TM}\downarrow$} \\
\cmidrule(lr){3-5}
& & Fraction $\uparrow$ & scRMSD $\downarrow$ & scTM $\uparrow$ & & \\
\midrule

\multirow{6}{*}{500 / 100}
& FrameFlow (500) & 0.849 & $1.439 \pm 1.137$ & $0.879 \pm 0.084$ & 0.369 & \underline{0.654} \\
& FrameFlow (100) & 0.803 & $1.576 \pm 1.367$ & $0.872 \pm 0.090$ & \underline{0.360} & \textbf{0.638} \\
& QFlow (500)     & \underline{0.900} & \underline{$1.271 \pm 1.113$} & \underline{$0.897 \pm 0.078$} & 0.399 & 0.720 \\
& QFlow (100)     & 0.885 & $1.319 \pm 1.018$ & $0.889 \pm 0.077$ & 0.393 & 0.699 \\
& RMF (100)       & 0.832 & $1.425 \pm 1.145$ & $0.885 \pm 0.081$ & \textbf{0.343} & 0.719 \\
& \textbf{SE3MF} (100) & \textbf{0.936} & $\mathbf{1.106} \pm 0.870$ & $\mathbf{0.909} \pm 0.067$ & 0.415 & 0.736 \\
\midrule

\multirow{3}{*}{50}
& QFlow     & \underline{0.870} & \underline{$1.432 \pm 1.218$} & $0.882 \pm 0.080$ & \underline{0.379} & \textbf{0.684} \\
& RMF       & 0.825 & $1.446 \pm 1.189$ & \underline{$0.885 \pm 0.081$} & \textbf{0.344} & \underline{0.717} \\
& \textbf{SE3MF} & \textbf{0.906} & $\mathbf{1.246} \pm 1.190$ & $\mathbf{0.902} \pm 0.078$ & 0.413 & 0.722 \\
\midrule

\multirow{3}{*}{20}
& QFlow     & 0.778 & $1.703 \pm 1.658$ & $0.853 \pm 0.098$ & \underline{0.377} & \textbf{0.648} \\
& RMF       & \underline{0.806} & \underline{$1.554 \pm 1.438$} & \underline{$0.877 \pm 0.091$} & \textbf{0.345} & 0.722 \\
& \textbf{SE3MF} & \textbf{0.867} & $\mathbf{1.369} \pm 1.144$ & $\mathbf{0.883} \pm 0.087$ & 0.408 & \underline{0.703} \\
\midrule

\multirow{3}{*}{10}
& QFlow     & 0.559 & $2.691 \pm 2.263$ & $0.773 \pm 0.144$ & \underline{0.379} & \textbf{0.615} \\
& RMF       & \textbf{0.778} & $\mathbf{1.641} \pm 1.432$ & $\mathbf{0.870} \pm 0.088$ & \textbf{0.346} & 0.713 \\
& \textbf{SE3MF} & \underline{0.728} & \underline{$1.940 \pm 1.741$} & \underline{$0.832 \pm 0.118$} & 0.402 & \underline{0.659} \\

\bottomrule
\end{tabular}
\endgroup
\caption{Unconditional protein backbone generation on SCOPe, grouped by sampling
budget. Baselines: FrameFlow~\cite{frameflow}, QFlow~\cite{yue2025reqflow},
Riemannian MeanFlow (RMF)~\cite{woo2026riemannian}. Within each budget group,
best in \textbf{bold} and second-best \underline{underlined}.}
\label{tab:scope_comparison}
\end{table*}

\begin{table*}[t]
\centering
\begingroup
\small
\setlength{\tabcolsep}{6pt}
\begin{tabular}{@{}llccccc@{}}
\toprule
\multirow{2}{*}{Steps} & \multirow{2}{*}{Method}
& \multicolumn{3}{c}{Designability}
& \multirow{2}{*}{Diversity $\mathrm{TM}\downarrow$}
& \multirow{2}{*}{Novelty $\mathrm{TM}\downarrow$} \\
\cmidrule(lr){3-5}
& & Fraction $\uparrow$ & scRMSD $\downarrow$ & scTM $\uparrow$ & & \\
\midrule

\multirow{2}{*}{100}
& ReQFlow          & 0.946 & $1.120 \pm 0.635$ & $0.901 \pm 0.062$ & \textbf{0.420} & \textbf{0.695} \\
& \textbf{RecSE3MF} & \textbf{0.968} & $\mathbf{0.974} \pm 0.620$ & $\mathbf{0.920} \pm 0.057$ & 0.448 & 0.734 \\
\midrule

\multirow{2}{*}{50}
& ReQFlow          & 0.946 & $1.131 \pm 0.633$ & $0.900 \pm 0.062$ & \textbf{0.418} & \textbf{0.689} \\
& \textbf{RecSE3MF} & \textbf{0.970} & $\mathbf{0.941} \pm 0.483$ & $\mathbf{0.923} \pm 0.048$ & 0.452 & 0.728 \\
\midrule

\multirow{2}{*}{20}
& ReQFlow          & 0.910 & $1.270 \pm 0.736$ & $0.883 \pm 0.074$ & \textbf{0.416} & \textbf{0.678} \\
& \textbf{RecSE3MF} & \textbf{0.929} & $\mathbf{1.133} \pm 1.314$ & $\mathbf{0.903} \pm 0.073$ & 0.454 & 0.726 \\
\midrule

\multirow{2}{*}{10}
& ReQFlow          & 0.837 & $1.552 \pm 1.113$ & $0.846 \pm 0.092$ & \textbf{0.414} & \textbf{0.662} \\
& \textbf{RecSE3MF} & \textbf{0.894} & $\mathbf{1.269} \pm 1.040$ & $\mathbf{0.891} \pm 0.084$ & 0.453 & 0.710 \\

\bottomrule
\end{tabular}
\endgroup
\caption{\textbf{Post-train comparison on SCOPe.} Within each budget, best in
\textbf{bold}. Training budgets are in Table~\ref{tab:efficiency}.}
\label{tab:scope_posttrain}
\end{table*}

\subsection{Baselines}
We compare recent flow-based backbone generators: FrameFlow~\cite{frameflow},
QFlow and ReQFlow~\cite{yue2025reqflow}, and Riemannian MeanFlow (RMF)~\cite{woo2026riemannian}.
Among these, ReQFlow and RMF are most closely related to our approach, as they also
target efficient few-step or accelerated generation. As a representative of earlier
(pre-2024) diffusion- and flow-matching methods we include FrameFlow, which demonstrated strongest generation capacity on SCOPe; we discuss the earlier baselines (e.g.\ FrameDiff, Genie) in Appendix~\ref{sec:related-classic}.

\subsection{Implementation Details}

For experiments on SCOPe, we use publicly available checkpoints when possible to reproduce baseline results under a consistent evaluation protocol. All experiments are conducted using 4 H-100 GPUs. Unless otherwise stated, we follow the evaluation settings used in prior work \cite{yue2025reqflow}, including the same datasets, sampling protocols, and evaluation metrics, to ensure a fair comparison across methods.

\subsection{Training details}
\label{sec:protein-se3mf-training}

\paragraph{Model and trainer.}
We implement our method within the public QFlow/ReQFlow codebase~\cite{yue2025reqflow}
and inherit its overall training pipeline, adapting it in the following respects to fit the MeanFlow model.\\
\emph{(i) Representation and interpolation.} QFlow/ReQFlow parametrize rotations by unit
quaternions; we instead work directly with rotation matrices under the decoupled
$\SE=\SO\times\mathbb{R}^3$ representation of \citet{yim2023se}. Accordingly, we build the
data--noise interpolation and the mini-batch coupling with our own $\SE$ geodesic
interpolant and optimal-transport (OT) coupling in rotation-matrix form
(Appendix~\ref{sec:bg-fm}; global-OT coupling in Eq.~\eqref{eq:global-ot}), rather than
the quaternion interpolation of QFlow.\\
\emph{(ii) Model.} We keep the ReQFlow IPA trunk but make it consume a rotation-matrix
state $(R_t,x_t)\in\SE^N$ and condition on \emph{two} times $(s,t)$---via a shared
two-time embedding and a per-block AdaLN-Zero gate---and predict endpoints
$(\hat R_0^\theta,\hat x_0^\theta)$ (Appendix~\ref{sec:architecture}).\\
\emph{(iii) Objective.} In place of the (V-)QFlow flow-matching loss, we train with our
$\SE$ MeanFlow objective (Section~\ref{sec:method}) and its $\alpha$-Flow variant
(Section~\ref{sec:alphaflow-main}); derivations and the stable implementation are in
Appendix~\ref{sec:SE3MF-details} and~\ref{sec:training-details}.
The time sampler is the two-time extension ($s\le t$) of the QFlow sampler, with the
marginal schedule of $t$ left unchanged, and
$s\mid t\sim\mathcal{U}[t_{\min},t]$ with $t_{\min}=10^{-6}$
(Table~\ref{tab:train-config}). All remaining pipeline components are inherited
from the QFlow/ReQFlow codebase.

\paragraph{Stable training.}
The differential MeanFlow target requires a time derivative obtained via a
Jacobian--vector product (JVP), which can be numerically brittle through the backbone
network. We use two remedies: the JVP-free $\alpha$-Flow objective
(Section~\ref{sec:alphaflow-main}) as a warm-up, and a small-$t$ stabilized form of the
JVP-based loss (Appendix~\ref{sec:stable-rescale}, Algorithm~\ref{alg:stable-derivative-B}).

\paragraph{Pre-training (two stages).}
Stage~1 is the $\alpha$-Flow warm-up on $\SE^N$
(Section~\ref{sec:alphaflow-main}; Appendix~\ref{sec:stage1}); Stage~2 switches to the
endpoint$+$MeanFlow objective (Appendix~\ref{sec:stage2}).

\paragraph{Post-training.}
We further apply a rectification (self-reflow) stage following ReQFlow's
rectified-flow strategy, keeping the Stage-2 MeanFlow objective
(Appendix~\ref{sec:posttrain}). We denote the resulting model RecSE3MF.
\begin{table}[t]
\centering
\begingroup
\small
\setlength{\tabcolsep}{4pt}
\begin{tabular}{lcccc}
\toprule
Method & Inf Steps & $\alpha$-Helix  & $\beta$-Strand  & $\mathrm{C}\alpha$-valid $\uparrow$\\
\midrule
\multirow{2}{*}{FrameFlow}
& 500 & $0.450\pm0.323$ & $0.248\pm0.218$ & $0.962\pm0.041$\\
& 100 & $0.465\pm0.308$ & $0.234\pm0.212$ & $0.987\pm0.014$\\
\midrule
\multirow{4}{*}{QFlow}
& 100 & $0.502\pm0.302$ & $0.208\pm0.206$ & $0.993\pm0.012$\\
& 50  & $0.469\pm0.291$ & $0.223\pm0.204$ & $0.996\pm0.008$\\
& 20  & $0.476\pm0.273$ & $0.201\pm0.184$ & $0.894\pm0.046$\\
& 10  & $0.461\pm0.246$ & $0.168\pm0.160$ & $0.831\pm0.061$\\
\midrule
\multirow{4}{*}{RMF}
& 100 & $0.346\pm0.232$ & $0.278\pm0.179$ & $0.996\pm0.007$\\
& 50  & $0.342\pm0.232$ & $0.279\pm0.178$ & $0.996\pm0.008$\\
& 20  & $0.340\pm0.230$ & $0.277\pm0.178$ & $0.996\pm0.008$\\
& 10  & $0.341\pm0.229$ & $0.276\pm0.173$ & $0.996\pm0.008$\\

\midrule

\multirow{4}{*}{\textbf{SE3MF}}
& 100 & $0.536\pm0.339$ & $0.200\pm0.238$ & $0.971\pm0.022$\\
& 50  & $0.517\pm0.336$ & $0.213\pm0.236$ & $0.971\pm0.022$\\
& 20  & $0.505\pm0.324$ & $0.210\pm0.224$ & $0.970\pm0.023$\\
& 10  & $0.497\pm0.299$ & $0.193\pm0.203$ & $0.975\pm0.020$\\
\midrule
\midrule
\multirow{4}{*}{ReQFlow}
& 100 & $0.511\pm0.312$ & $0.224\pm0.214$ & $0.993\pm0.009$\\
& 50  & $0.509\pm0.311$ & $0.223\pm0.212$ & $0.986\pm0.014$\\
& 20  & $0.507\pm0.307$ & $0.214\pm0.212$ & $0.860\pm0.066$\\
& 10  & $0.502\pm0.310$ & $0.206\pm0.209$ & $0.712\pm0.105$\\
\midrule
\multirow{4}{*}{\textbf{RecSE3MF}}
& 100 & $0.559\pm0.357$ & $0.207\pm0.249$ & $0.979\pm0.017$\\
& 50  & $0.559\pm0.360$ & $0.206\pm0.249$ & $0.979\pm0.017$\\
& 20  & $0.571\pm0.354$ & $0.197\pm0.246$ & $0.978\pm0.017$\\
& 10  & $0.558\pm0.359$ & $0.198\pm0.245$ & $0.979\pm0.017$\\
\bottomrule
\end{tabular}%
\endgroup
\caption{Evaluation scores at different sampling steps. Methods follow Table~\ref{tab:scope_comparison} and Table~\ref{tab:scope_posttrain}.}
\label{tab:SCOPe_ss}
\end{table}

\begin{table}[t]
\centering
\begingroup
\small
\setlength{\tabcolsep}{1pt}
\begin{tabular}{@{}lccc@{}}
\toprule
Method & Params & Train Steps & Time/step (ms) $\downarrow$ \\
\midrule
\multicolumn{4}{@{}l}{\emph{Pre-trained}}\\
FrameFlow   & $16.8$M & $177$k & $57$ \\
QFlow & $16.8$M & $90$k  & $43$ \\
RMF & $437$M  & $598$k & $454$ \\
\textbf{SE3MF} (ours)        & $16.8$M & $106$k & $43$ \\
\midrule
\multicolumn{4}{@{}l}{\emph{Post-trained (rectification)}}\\
ReQFlow & $16.8$M & $90$k\,$+$\,$4.6$k & $43$ \\
\textbf{RecSE3MF} (ours)      & $16.8$M & $106$k\,$+$\,$6.5$k & $43$ \\
\bottomrule
\end{tabular}
\endgroup
\caption{Model size, training budget, and sampling cost. Post-training steps are
given as pre-training $+$ rectification. Time/step is wall-clock per sampling
step on a single NVIDIA H100 80GB (single process, batch $10$, length $128$,
fp32); each method runs one network forward per step, so total sampling time
$\approx$ steps $\times$ time/step.}
\label{tab:efficiency}
\end{table}
\subsection{Evaluation metrics and settings}
We evaluate generated protein backbones with four metrics, following prior
work~\citep{yue2025reqflow}: designability, diversity, novelty, and efficiency.
For each chain length $N$ (from $60$ to $128$) we generate $10$ backbones.
Designability is the primary measure of sample quality and assesses whether a
generated backbone admits an amino-acid sequence that folds back into a consistent
structure. For each backbone we design $8$ sequences with ProteinMPNN~\citep{dauparas2022robust},
predict the folded structure of each with ESMFold~\citep{esm}, and take the minimum
self-consistency RMSD (scRMSD) over the $8$ designs. A backbone is deemed
\emph{designable} when this scRMSD is at most $2\,\text{\AA}$. We report the fraction
of designable backbones (per length, averaged over lengths), denoted Fraction, the
mean scRMSD, and the mean self-consistency TM-score (scTM), where higher scTM and
Fraction and lower scRMSD are better.

To assess distributional properties, we measure structural diversity and novelty over
the designable subset. Diversity is the pairwise TM-score~\citep{zhang2004scoring}
among designable structures of the same length, averaged across lengths, where lower
values indicate a less redundant set. Novelty compares each designable sample against
the Protein Data Bank (PDB) with Foldseek~\citep{van2022foldseek} and records the
maximum TM-score to the retrieved structures; the average of these maxima summarizes
similarity to known proteins, with lower values indicating greater novelty. Finally,
we quantify efficiency by the number of sampling steps used to generate the backbones.

\subsection{Results and Discussion}
\begin{figure}[t]
  \centering
  \includegraphics[width=0.6\linewidth]{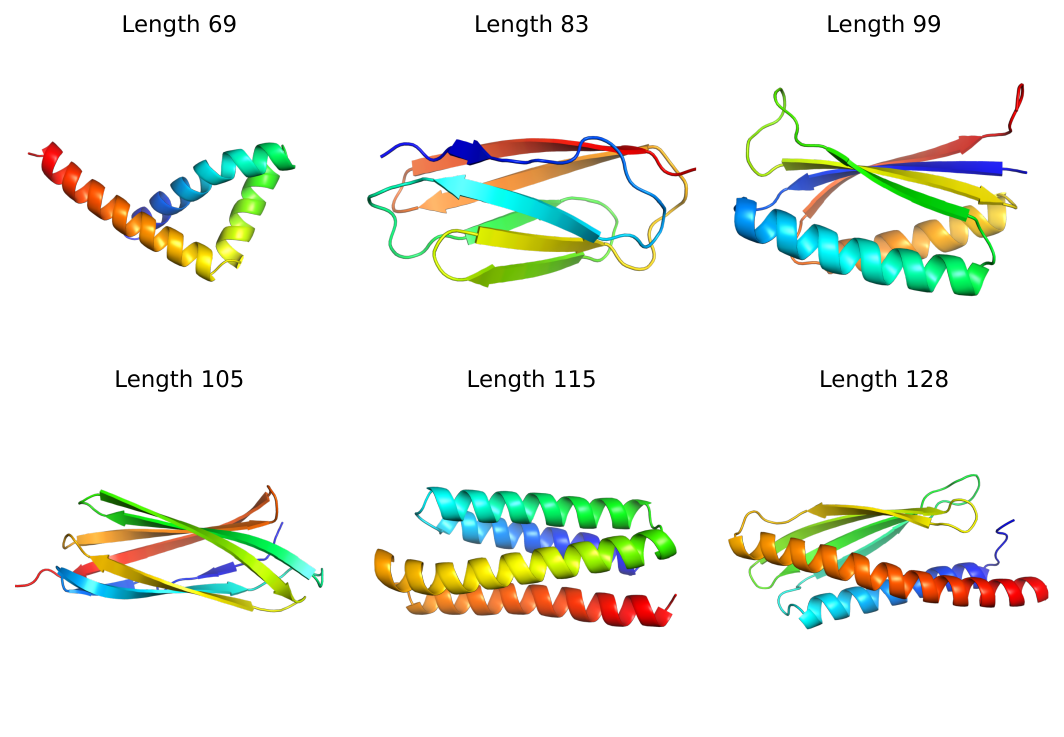}
  \caption{Qualitative samples on SCOPe generated at different sampling budgets. See Appendix for additional diagnostics.}
  \label{fig:samples-gallery}
\end{figure}
SE(3)-MeanFlow (SE3MF) consistently improves few-step protein backbone generation. Across the
20--100-step regime, it achieves the strongest designability among pretrained
baselines, ranking first in designable fraction, scRMSD, and scTM
(Table~\ref{tab:scope_comparison}). The gains are preserved across protein
lengths (Figures~\ref{fig:fraction-vs-length-T100-T20}
and~\ref{fig:scrmsd-boxplot-T100-T20}), indicating that the improvement is not
restricted to short or structurally simple backbones. At 20 steps, SE(3)-MeanFlow
retains a designable fraction of $0.867$, compared with $0.806$ for RMF and
$0.778$ for QFlow, demonstrating a strong designability--efficiency trade-off
under substantial step reduction; Figure~\ref{fig:samples-gallery} shows
designable backbones generated at this budget.

These results support the effectiveness of specializing MeanFlow to the
Lie-group structure of protein frames. Empirically, the complete formulation maintains strong
designability as the sampling budget is reduced while preserving stable local
backbone geometry, with $\mathrm{C}\alpha$-validity remaining near $0.97$ from
100 to 10 steps (Table~\ref{tab:SCOPe_ss}). The secondary-structure
statistics are likewise stable across sampling budgets
(Figure~\ref{fig:ss-scope-T100-T20}), suggesting that accelerated generation
does not substantially alter the structural composition of the samples.

Rectification further strengthens the aggressive few-step regime. RecSE3MF
outperforms ReQFlow on the three designability metrics at nearly all sampling
budgets and reaches a designable fraction of $0.894$ at 10 steps
(Table~\ref{tab:scope_posttrain}). This suggests that rectification and
MeanFlow play complementary roles: rectification simplifies the transport
paths, while MeanFlow learns accurate finite-interval motion along those paths.

The improved designability comes at a cost in coverage. Both SE(3)-MeanFlow and
RecSE3MF have higher diversity TM-scores than the coverage-oriented baselines
(lower is better), and their novelty is weakest at large budgets ($0.736$ at
$100$ steps, against $0.719$ for RMF); the gap narrows as the budget falls, with
novelty second-best in its group at $20$ and $10$ steps. Diversity is nearly flat
in the number of steps ($0.415$ to $0.402$), so the concentration reflects the
learned model rather than step reduction. Overall, SE(3)-MeanFlow advances the
designability--efficiency frontier while preserving stable geometric and
structural statistics under substantially reduced sampling budgets.
\paragraph{Training budget and model size.}
Table~\ref{tab:efficiency} compares model size, training budget and per-step
sampling cost. Our network is the ReQFlow/QFlow trunk with a small ($0.33\%$)
AdaLN addition, hence comparable in size to QFlow and checkpoint-compatible with
it (Appendix~\ref{sec:architecture}), and $\sim\!26\times$ smaller than RMF; its
total budget is on the same order as the flow-matching baselines from the same
codebase, and far below RMF's $\sim\!598$k steps (their cap is $1000$k).
We conjecture that part of this gap reflects RMF's semigroup consistency
objective rather than its model size alone: straightening the transport paths
enough for few-step sampling appears to demand a large budget, and this cost
grows with the number of residues. Appendix~\ref{sec:abl-controlled} tests this by fixing the model, the
initialization and the training budget, and swapping only the objective; under
this control the semigroup target does not reach the few-step regime while ours
does. 

On the low-dimensional $\mathrm{SO}(3)^2$ benchmark of Appendix~\ref{sec:toy},
a second controlled ablation in which only the loss varies, every
average-velocity method reaches the few-step regime under a much smaller budget.
\begin{remark}
Epoch counts are not comparable across methods, as the data loaders differ, whereas
each optimizer step processes the same amount of data. We therefore report budgets
in optimizer steps; see Appendix~\ref{sec:protein-extra} for details. 
\end{remark}

\section{Conclusion}
We introduced \textbf{SE(3)-MeanFlow}, a few-step generative framework on $\SE^N$.
Exploiting the decoupled $\SE=\SO\times\mathbb{R}^3$ structure, we derive
simulation-free average-velocity objectives directly in the Lie algebra, avoiding
parallel transport, and stabilize training with a JVP-free $\alpha$-Flow warm-up and
a small-$t$ MeanFlow scheme supporting both training from scratch and
rectification-based post-training. On SCOPe, SE(3)-MeanFlow advances the
designability--efficiency frontier while using fewer sampling steps and a smaller
model. Its main limitation is a modest reduction in diversity; future work will seek
objectives that better balance coverage and designability, extend to conditional and
motif-scaffolding tasks, and push toward one-step generation.

\section*{Acknowledgments}
Y.B. was supported in part by the Purdue Institute for Physical AI (IPAI) Postdoctoral Fellows Program. G.L. would like to thank the support of National Science Foundation (DMS-2533878, DMS-2053746, DMS-2134209, ECCS-2328241, CBET-2347401 and OAC-2311848), and U.S.~Department of Energy (DOE) Office of Science Advanced Scientific Computing Research program, under the ``Uncertainty Quantification for Multifidelity Operator Learning (MOLUcQ)'' project (Project No. 81739), DE-SC0023161, the SciDAC LEADS Institute, and DOE--Fusion Energy Science, under grant number: DE-SC0024583.

\bibliographystyle{plainnat}
\bibliography{refs}

\newpage
\appendix
\appendix



\makeatletter
\def\addcontentsline#1#2#3{%
  \addtocontents{#1}{\protect\contentsline{#2}{#3}{\thepage}{}}%
}
\makeatother

\etocsettocstyle{\section*{Appendix Contents}}{}
\etocsetnexttocdepth{section}  
\tableofcontents

\section{Background: MeanFlow in Euclidean Space}
Consider a probability path $\{p_t:t\in[0,1]\}$ with $p_0=p_{\text{data}}$ and $p_1=p_{\text{prior}}$, where we assume $p_t$ is generated by the following ODE system:
\begin{align}
  \begin{cases}
  X_0\sim p_0 &\text{initial distribution}\\
  dX_t=v(t,X_t)dt& \text{ODE}\\ 
p_t=\text{Law}(X_t)& \text{Marginal distribution in time}
\end{cases}
\end{align}
Equivalently,  $(v_t,p_t)$ satisfies  the continuity equation:
\begin{align}
  \frac{d}{dt}p_t = -\nabla\cdot(v_tp_t) \label{eq:cont_eq}
\end{align}
with the same initial distribution condition.

\citet{geng2026mean} proposed the mean velocity field $u(s,t,x_t)$ defined as
\begin{equation}
u(s,t,x_t) := \frac{1}{t-s}\int_s^t v(\tau,x_\tau)d\tau \in \mathbb{R}^d.\label{eq:u_Euclidean}
\end{equation}

Multiplying both sides by $(t-s)$ and differentiating with respect to $t$ yields the \emph{MeanFlow identity}:
\begin{align}
&\begin{cases}
\text{L.H.S.}&=\frac{d}{dt}(t-s)u(s,t,x_t)=u(s,t,x_t)+(t-s)\frac{d}{dt}u(s,t,x_t)\nonumber\\
\text{R.H.S.}&=v(t,x_t)=x_1-x_0\nonumber \\
\end{cases}\nonumber 
\end{align}

We compute the R.H.S. from the data and the L.H.S. from the model; the training loss is derived from these calculations: 
\begin{align}
&\mathcal{L}=\begin{cases}
  &\mathbb{E}_{t,(x_0,x_1)\sim \pi_0}\left[\left\|u_\theta+\text{sg}((t-s)(\frac{d}{dt}u_t)-v(t,x_t))\right\|^2\right]\nonumber\\
  &\qquad\text{\cite{geng2026mean}}\\
  &\mathbb{E}_{t,(x_0,x_1)\sim \pi_0}\left[\left\|u_\theta+(t-s)\text{sg}(\frac{d}{dt}u_t)-v(t,x_t)\right\|^2\right]\nonumber\\
  &\qquad\text{\cite{geng2026improved}}
  \end{cases}\label{eq:mf_loss-details}\\
&\text{where }v(t,x_t)=x_1-x_0 \nonumber\\
&\frac{d}{dt}u_\theta(r,t,x_t)=\frac{\partial u_\theta}{\partial x_t}v(t,x_t) + \frac{\partial u_\theta}{\partial t} \nonumber 
\end{align}
and $\text{sg}$ denotes the stop-gradient operation, which is used to prevent backpropagation through the target.

\paragraph{One-step training objective.}
MeanFlow learns a parametric field $u_\theta(s,t,x_t)$ via regression, using the standard linear interpolation
between a data sample $x_0\sim p_{\text{data}}$ and a prior sample $x_1\sim p_{\text{prior}}$:
\begin{equation}
  x_t = (1-t)x_0 + t x_1, \qquad v = x_1 - x_0.
\end{equation}
Using the MeanFlow identity, the target mean flow can be written as
\begin{equation}
  u_{\text{tgt}}(s,t,x_t) = v - (t-s)\frac{d}{dt}u_\theta(s,t,x_t),
\end{equation}
where $\frac{d}{dt}u_\theta$ is the total derivative (implemented as a Jacobian-vector product), and
$u_{\text{tgt}}$ is treated as a stop-gradient target.
The MeanFlow training loss is
\begin{equation}
  \mathcal{L}_{\text{MF}}(\theta)
  = \mathbb{E}_{\substack{x_0\sim p_{\text{data}},\\\ x_1\sim p_{\text{prior}},\\ s<t}}
  \left[\left\|u_\theta(s,t,x_t) - \mathrm{sg}\big(u_{\text{tgt}}(s,t,x_t)\big)\right\|_2^2\right].
\end{equation}

\section{Background: SO(3) Space}\label{sec:bg-so3}

\subsection{Basic concepts in SO(3) space}

The special orthogonal group in three dimensions, denoted as $\SO$, is defined as
$$
\SO
:=
\{ R \in \mathbb{R}^{3\times3} : R R^\top = R^\top R = I_3,\det(R)=1\}.
$$
The constraint $R R^\top = I_3$ consists of smooth polynomial equations, which implies that $\SO$ is a smooth embedded submanifold of $\mathbb{R}^9$.

Equipped with matrix multiplication $\SO^{\times2}\ni(S_1,S_2)\mapsto S_1S_2\in \SO$, $\SO$ forms a Lie group.
The identity element is $I_3$, and the inverse of any element $R \in \SO$ is given by $R^{-1} = R^\top$.

\paragraph{Tangent space.}
Let $R \in \SO$ be fixed.
Consider a smooth curve $R(t) \in \SO$ such that
$$
R(0) = R.
$$
Since $R(t) R(t)^\top = I_3$ holds for all $t$, differentiating at $t=0$ yields
$$
\dot R(0) R^\top + R \dot R(0)^\top = 0.
$$
This implies that $R^\top \dot R(0)$ is skew-symmetric.
Therefore, the tangent space at $R$ is given by
$$
T_R(\SO)
=
\{ R \Omega : \Omega \in \mathfrak{so}(3) \},
$$
where the Lie algebra $\mathfrak{so}(3)$ is defined as
$$
\mathfrak{so}(3)
:=
\{ \Omega \in \mathbb{R}^{3\times3} : \Omega^\top = -\Omega \}.
$$
This representation provides a global linear parameterization of tangent vectors on $\SO$. We also use the standard hat/vee identification between $\mathbb{R}^3$ and $\mathfrak{so}(3)$: for $\omega=(\omega_1,\omega_2,\omega_3)\in\mathbb{R}^3$,
\begin{align}
&\omega^{\wedge}:=
\begin{bmatrix}
0 & -\omega_3 & \omega_2\\
\omega_3 & 0 & -\omega_1\\
-\omega_2 & \omega_1 & 0
\end{bmatrix}\in\mathfrak{so}(3),\nonumber\\
&(\cdot)^\vee:\mathfrak{so}(3)\to\mathbb{R}^3\ \text{is the inverse of }(\cdot)^\wedge.   \nonumber
\end{align}

\paragraph{Lie algebra and Lie bracket.}
The vector space $\mathfrak{so}(3)$ is the Lie algebra associated with the Lie group $\SO$.
The Lie bracket on $\mathfrak{so}(3)$ is defined as the matrix commutator
$$
[\Omega_1,\Omega_2]
=
\Omega_1 \Omega_2 - \Omega_2 \Omega_1,
\quad
\Omega_1,\Omega_2 \in \mathfrak{so}(3).
$$
It is straightforward to verify that this operation preserves skew-symmetry, hence $\mathfrak{so}(3)$ is closed under the Lie bracket.

\subsection{Curves and flows in SO(3)}
In Euclidean space, a curve connecting $x_0$ and $x_1$ can be defined by the ODE
\begin{align}
  \begin{cases}
    X_0 = x_0, \\
    \dot X_t = v(t,X_t).
  \end{cases}
\end{align}
and thus $X_1=\int_0^1 v(t,x_t)dt$. We now extend this construction to the Lie group $\SO$.

\paragraph{ODE-defined curves on $\SO$.}
Let $R_t:[0,1]\to\SO$ be a time-dependent curve. A natural intrinsic definition of its evolution is given by
$$
\begin{cases}
  R_0=r_0\\
  \dot R_t = R_t \Omega_t,
  \quad
  \Omega_t \in \mathfrak{so}(3).
\end{cases}
$$
This ODE guarantees that $R_t$ remains in $\SO$ for all $t$, since
$$
\frac{d}{dt}(R_t R_t^\top)
=
R_t\Omega_t R_t^\top + R_t \Omega_t^\top R_t^\top
=
0.
$$
Given an initial condition $R_0 \in \SO$, the above equation defines a smooth curve on the manifold.

\paragraph{Exponential map and geodesics.}
When the angular velocity is constant, namely $\Omega_t = \Omega$, the ODE admits a closed-form solution
$$
R_t = R_0 \exp(t \Omega),
$$
where $\exp$ denotes the matrix exponential.
Under the canonical bi-invariant Riemannian metric on $\SO$, curves of this form are geodesics.

Given two points $R_0,R_1 \in \SO$, the geodesic connecting them is given by
$$
R_t
=
R_0 \exp\big(t \log(R_0^\top R_1)\big),
\quad
t \in [0,1],
$$
where $\log$ denotes the matrix logarithm mapping $\SO$ to $\mathfrak{so}(3)$.
This curve minimizes the path length among all smooth curves on $\SO$ connecting $R_0$ and $R_1$.
\paragraph{Inner product and geodesic distance in $\SO$.}
Under the canonical bi-invariant Riemannian metric, we identify $T_R\SO=\{R\Omega:\Omega\in\mathfrak{so}(3)\}$ and define the inner product
$$
\langle R\Omega_1, R\Omega_2\rangle_R
=
\tfrac12\,\mathrm{tr}(\Omega_1^\top\Omega_2)
=
\tfrac12\,\mathrm{tr}\big((R^\top\dot R_1)^\top (R^\top\dot R_2)\big),
$$
which is invariant under left and right multiplication.
The induced geodesic distance between $R_0,R_1\in\SO$ is
$$
\mathrm{dist}(R_0,R_1)
=
\frac{1}{\sqrt{2}}\,\big\|\log(R_0^\top R_1)\big\|_F,
$$
where $\|\cdot\|_F$ denotes the Frobenius norm.

\paragraph{Integration in SO(3).}
In Euclidean space, $x_1=\int_0^1 v(t,x_t)dt+x_0$.

Due to the non-commutative group structure, the endpoint in $\SO$ cannot be written as a
simple integral. Instead, the solution at $t=1$ admits the group-valued representation
$$
R(1)=R(0)
\mathcal{T}\exp\Big(\int_0^1 \Omega_t dt\Big),
$$
where $\mathcal{T}$ denotes the time-ordering operator:

\begin{align}
&\mathcal{T}\exp\Big(\int_0^1 A(t) dt\Big)
=
I_3
+\sum_{k=1}^{\infty}
\int_{0 \le t_1 \le \cdots \le t_k \le 1}
A(t_1)\cdots A(t_k)
dt_1 \cdots dt_k,    
\end{align}

where the time-ordering operator enforces the chronological ordering of the
matrix products.

\section{Background: SE(3) Space}\label{sec:bg-se3}

\subsection{Basic concepts in the $\SE$ space}

The special Euclidean group in three dimensions is defined as
$$
\SE=\{(R,x): R\in \SO,\ x\in\mathbb{R}^3\},
$$
representing orientation-preserving rigid motions in $\mathbb{R}^3$. Algebraically, $\SE$ is the
semidirect product
$$
\SE=\SO\ltimes(\mathbb{R}^3,+),
$$
where rotations act on translations. The group operation is given by
\begin{align}
(R_1,x_1)(R_2,x_2)=(R_1R_2,\ x_1+R_1 x_2),
\label{eq:se3_semidirect}
\end{align}
and the inverse by
$$
(R,x)^{-1}=(R^\top,\ -R^\top x).
$$
As a smooth manifold, $\SE$ is six-dimensional and diffeomorphic to $\SO\times\mathbb{R}^3$.

\paragraph{Lie algebra and tangent space.}
The Lie algebra $\mathfrak{se}(3)$ consists of pairs $(\Omega,v)$ with $\Omega\in\mathfrak{so}(3)$ and
$v\in\mathbb{R}^3$, i.e.,
$$
\mathfrak{se}(3)=\{(\Omega,v):\Omega^\top=-\Omega,\ v\in\mathbb{R}^3\}.
$$
The Lie bracket is given by
$$
[(\Omega_1,v_1),(\Omega_2,v_2)]
=
([\Omega_1,\Omega_2],\ \Omega_1 v_2 - \Omega_2 v_1).
$$
For any $(R,x)\in \SE$, the tangent space is obtained by left translation:
$$
T_{(R,x)}\SE=\{(R\Omega,\ Rv):\Omega\in\mathfrak{so}(3),\ v\in\mathbb{R}^3\}.
$$

\subsection{Curves and flows in SE(3)}

A time-dependent rigid motion is represented by a curve $(R_t,x_t)\in \SE$. A natural way to define
its intrinsic evolution is through the left-trivialized velocity:
$$
\begin{cases}
  \dot R_t = R_t\Omega_t,\\
  \dot x_t = R_t v_t,
\end{cases}
\qquad
(\Omega_t,v_t)\in\mathfrak{se}(3),
$$
where $\Omega_t$ governs angular motion and $v_t$ determines translational motion in the body frame.
Note that the translational velocity is expressed in the body frame and is therefore rotated by
$R_t$: in the full semidirect-product geometry the two components are coupled.

\subsection{The decoupled product geometry used in this work}

Following the setting in Section~3 of \citet{bose2024se}, we instead work with the product manifold
$$
\SO\times\mathbb{R}^3
$$
equipped with a product Riemannian metric, under which the composition simplifies to
\begin{align}
\label{eq:dec_met}
  (R_1,x_1)(R_2,x_2)=(R_1R_2,\ x_1+x_2).
\end{align}
Comparing with \eqref{eq:se3_semidirect}, this replaces $x_1+R_1x_2$ by $x_1+x_2$, i.e.\ it drops the
action of $R_1$ on $x_2$, so that the rotational and translational coordinates evolve independently.
The left-trivialized flow correspondingly reduces to $\dot R_t=R_t\Omega_t$ and $\dot x_t=v_t$.

\begin{remark}[Scope of this section]
\label{rmk:se3-scope}
This section is included as background only. Our method does \emph{not} use the semidirect-product
structure \eqref{eq:se3_semidirect} or any property specific to it. Following
FoldFlow~\cite{bose2024se} and ReQFlow~\cite{yue2025reqflow}, we work throughout with the decoupled
product geometry \eqref{eq:dec_met}, in which rotation and translation are independent. Accordingly,
all constructions in this paper---the interpolation path, the MeanFlow identity, the $\alpha$-Flow
target, the semigroup loss, and the training objectives---are derived \emph{separately} on $\SO$ and
on $\mathbb{R}^3$, and the $\SE$ ($\SE^N$) losses are simply the sum of the two branches. The
translational velocity is likewise taken in the ambient frame rather than the body frame. Whenever we
write $\SE\cong\SO\times\mathbb{R}^3$ in the main text, it refers to this decoupled convention rather
than to a group isomorphism.
\end{remark}

\section{Background: SE(3) Flow Matching}\label{sec:bg-fm}
Our method builds on flow matching, which we briefly review here---first on a general
Riemannian manifold, and then in the decoupled $\SE$ geometry used throughout the paper.

\subsection{Riemannian flow matching}
Flow matching \citep{lipman2022flow,tong2023improving} learns a time-dependent velocity
field whose flow transports a prior density $p_1$ to the data density $p_0$. On a
Riemannian manifold $\mathcal{M}$, one fixes a conditional coupling $(z_0,z_1)\sim q$ and
connects the endpoints by the minimizing geodesic
$$
z_t=\exp_{z_0}\!\big(t\,\log_{z_0}(z_1)\big),\qquad t\in[0,1],
$$
whose time derivative is the conditional (target) velocity
$\dot z_t\in\mathcal{T}_{z_t}\mathcal{M}$. Regressing a model field $u_\theta(t,z_t)$ onto
this target,
$$
\mathbb{E}_{t,\,q(z_0,z_1)}\big[\,\|u_\theta(t,z_t)-\dot z_t\|_{g}^2\,\big],
$$
recovers at its minimizer the marginal velocity that generates the interpolating path
$p_t$. Sampling then integrates the learned field, e.g.\ backward from a prior draw
$z_1$ to a data sample $z_0$.

\subsection{SE(3) flow matching}
Following \citet{bose2024se}, we adopt the decoupled product geometry
$\SE\cong\SO\times\mathbb{R}^3$, so that the geodesic and its velocity
split into independent rotational and translational parts. For a data frame $(R_0,x_0)$
and a prior frame $(R_1,x_1)$, the conditional path is the $\SO$ geodesic paired
with the Euclidean straight line,
$$
R_t=R_0\exp\!\big(t\,\log(R_0^\top R_1)\big),\qquad
x_t=(1-t)\,x_0+t\,x_1 .
$$
Differentiating gives the closed-form conditional velocities, which are constant along
each path:
$$
\omega_t:=\log(R_0^\top R_1)^\vee\in\mathbb{R}^3,\qquad
v_t:=x_1-x_0\in\mathbb{R}^3,
$$
so that $\dot R_t=R_t\,\omega_t^{\wedge}$ and $\dot x_t=v_t$. The model predicts the body angular
and translational velocities $A_\theta(t,R_t,x_t)$ and $v_\theta(t,x_t)$---equivalently the
instantaneous ($s{=}t$) evaluation of the two-time head,
$A_\theta(t,R_t,x_t)=A^{\mathrm{avg}}_\theta(t,t,R_t,x_t)$---and the $\SE$
flow-matching objective regresses them onto these targets,
\begin{align}
\mathcal{L}_{\mathrm{FM}}
&=\mathbb{E}_{t,\,q(z_0,z_1)}\sum_{i=1}^{N}\Big[\big\|A_{\theta,i}(t,R_{t,i},x_{t,i})-\omega_{t,i}\big\|_2^2 +\big\|v_{\theta,i}(t,x_{t,i})-v_{t,i}\big\|_2^2\Big],
\label{eq:fw-loss}
\end{align}
summed over the $N$ residues. Equation~\eqref{eq:fw-loss} is the data-anchored objective to
which our MeanFlow and $\alpha$-Flow targets reduce in the appropriate limit---for
instance at $\alpha=1$, where $A^{\mathrm{avg}}_{\mathrm{tgt}}=\omega_t$---and it supplies the
boundary condition that keeps the consistency objective from collapsing.

\subsection{Mini-batch optimal-transport coupling}
\label{sec:bg-ot}
The flow-matching objective \eqref{eq:fw-loss} is defined for any coupling
$q(z_0,z_1)$ of the data and prior marginals. The simplest choice is the
independent coupling $q=p_0\otimes p_1$, under which trajectories from
different pairs cross and the marginal velocity field is far from constant
along each path. Following \citet{tong2023improving,pooladian2023multisample} and
FoldFlow~\cite{bose2024se}, we instead re-pair each batch by discrete optimal
transport. Given $M$ data frames $\{z_0^k\}$ and $M$ prior frames $\{z_1^k\}$
with uniform weights, the Optimal Transport problem \cite{villani2009optimal,flamary2021pot}
\begin{align}
&\min_{\Pi\in\mathcal{U}_M}\ \sum_{k,l}\Pi_{kl}\,c(z_0^k,z_1^l),
\nonumber\\
&\mathcal{U}_M=\{\Pi\ge 0:\Pi\mathbf{1}=\Pi^\top\mathbf{1}=\tfrac1M\mathbf{1}\},
\nonumber
\end{align}
attains its optimum at a vertex of $\mathcal{U}_M$, i.e.\ (Birkhoff) at a
permutation $\pi^\star\in\mathcal{S}_M$, so the plan reduces to a linear
assignment solvable exactly in $O(M^3)$ by the Hungarian algorithm. On the
decoupled product geometry \eqref{eq:dec_met} we take the squared $\SE^N$
geodesic cost,
\begin{align}
c\big((R_0,x_0),(R_1,x_1)\big)
&=\sum_{i=1}^{N}\Big[\lambda_R\,d_{\SO}^2(R_0^{i},R_1^{i})
+\lambda_x\,\|x_0^{i}-x_1^{i}\|_2^2\Big], \nonumber\\
 d_{\SO}(R,R')&=\tfrac{1}{\sqrt2}\|\log(R^\top R')\|_F,
\nonumber
\end{align}
with $(\lambda_R,\lambda_x)=(0.5,0.5)$. Re-pairing shortens the average
transport distance and straightens the induced probability path; the resulting
velocity field is closer to constant along each trajectory, which is exactly
the regime in which few-step average-velocity sampling is accurate. This is
also why the controlled benchmark of Appendix~\ref{sec:toy} disables OT
(Section~\ref{sec:toy-controls}): with OT the instantaneous- and
average-velocity parameterizations nearly coincide, and the comparison would no
longer be informative.

\section{Related Work}\label{sec:related}
\subsection{Diffusion and early flow-matching baselines}
\label{sec:related-classic}
The first wave of $\SE^N$ backbone generators are diffusion models over residue frames.
FrameDiff~\cite{yim2023se} formulates SE(3)-invariant score-based diffusion on multiple
frames and generates designable monomers up to $500$ residues without a pretrained
structure-prediction network ($17.4$M parameters). Genie~\cite{genie} instead diffuses
oriented $\mathrm{C}_\alpha$ residue clouds with SE(3)-equivariant triangle updates; it is
parameter-light ($4.1$M) and attains high diversity and novelty, but its designability is
comparatively low and degrades sharply as the number of sampling steps is reduced.
FrameFlow~\cite{frameflow} keeps the frame representation but replaces diffusion with
$\SE$ flow matching, reporting roughly $2\times$ higher designability than FrameDiff at
about $5\times$ fewer sampling steps, and a $\sim\!23\times$ sampling speed-up over Genie
at markedly higher designability. Because FrameFlow dominates both earlier models on the
designability--efficiency axis that is our focus, we take it as the representative of this
pre-2024 generation in the main-text comparison (Table~\ref{tab:scope_comparison}) and do
not separately tabulate FrameDiff or Genie. The flow-matching methods most closely related
to ours---FoldFlow, ReQFlow, and Riemannian MeanFlow---are discussed next.

\subsection{FoldFlow \cite{bose2024se}}
\citet{bose2024se} introduces $\SE$ stochastic flow matching for protein backbone generation.
Their training objective is a flow-matching regression that fits a time-dependent vector field $(v^R_\theta,v^x_\theta)$ to the drift of a chosen probability path (an $\SE$ bridge) between data $(R_0,x_0)$ and noise $(R_1,x_1)$:
\begin{align}
&\mathcal{L}_{\mathrm{FM}}(\theta)
=\mathbb{E}_{t,(R_0,x_0),(R_1,x_1)}\left[\|v_\theta^R(R_t,t)- v^R(R_t,t\mid R_0,R_1)\|^2_{\SO}+\|v_\theta^x(x_t,t)-v^x_t(x_t,t|x_0,x_1)\|^2\right]\nonumber\\
&\text{where }v^R_t(R_t|R_0,R_1):=\dot{R}_t\mid_{R_0,R_1}:=R_t\Omega_t,\nonumber\\
&\Omega_t=\log(R_0^\top R_1),\quad R_t=R_0e^{t\Omega_t},\nonumber\\ &v^{x}_t(x_t|x_0,x_1)=x_1-x_0,\quad x_t=(1-t)x_0+tx_1.\nonumber
\end{align}
In this work, the authors introduce three variants:
\begin{itemize}
\item In the base setting, $q$ is the independent coupling between $(p_0=p_{\text{data}},p_1=p_{\text{prior}})$.

\item In addition, the authors introduce the (mini-batch) optimal transport coupling obtained by drawing batch-size samples $p_0^B\sim \text{i.i.d. } p_{0},p_1^B\sim\text{i.i.d. } p_{1}$, and the method is denoted as FoldFlow-OT.
\item Furthermore, they introduce perturbation in the rotation interpolation $R_t$ by the IGSO-distribution:
$$\tilde{R}_t\mid_{R_0,R_1}\sim \mathcal{IG}_{SO(3)}(R_t,\gamma^2(t)t(1-t)).$$
where $\gamma^2(t)>0$ is a predefined function.
\end{itemize}

\textbf{Adaptation to protein backbone space: $\SE^N$}
As discussed in the main text, a protein backbone with $N$ residues can be represented as a product space $\SE^N$, i.e., $x=(g_1,\dots,g_N)$ with $g_i=(R_i,x_i)\in \SE$.
The $\SE$ flow-matching loss then extends by summing (or averaging) the per-residue $\SE$ losses, which is equivalent to \emph{concatenating} all residue-wise tangent vectors into a single $6N$-dimensional stacked vector:
\begin{align}
&\mathcal{L}_{\mathrm{FM}}^{\SE^N}(\theta)
=\mathbb{E}\frac{1}{N}\Big[\,\sum_{i=1}^N \|v^R_{\theta,i}(R_{t,i},t)- v^R_{t,i}(R_{t,i}\mid R_{0,i},R_{1,i})\|^2_{\SO}
+\sum_{i=1}^N \|v^x_{\theta,i}(x_{t,i},t)- v^x_{t,i}(x_{t,i}\mid x_{0,i},x_{1,i})\|^2\,\Big]\nonumber
\end{align}

In addition, the auxiliary loss (see section \ref{sec:auxiliary_loss}) is included in the final training loss.

\subsection{FoldFlow-2 \cite{huguet2024sequence}}

\citet{huguet2024sequence} extends FoldFlow to a sequence-conditioned 
generative model. The underlying generative task remains SE(3)$^N$ flow 
matching as in FoldFlow, and the loss structure is identical. The key 
novelty is that the vector field is now conditioned on a (possibly masked) 
amino acid sequence $\bar{a}$:

\begin{align}
&\mathcal{L}_{\mathrm{FM}}^{\mathrm{FF2}}(\theta)
= \mathbb{E}_{t,\,\pi(x_0,x_1),\,\bar{a}}
\frac{1}{N}\Big[
\sum_{i=1}^N \|v^R_{\theta,i}(R_{t,i},t|\bar{a})
- v^R_{t,i}(R_{t,i}\mid R_{0,i},R_{1,i},\bar{a})\|^2_{\SO} +\sum_{i=1}^N \|v^x_{\theta,i}(x_{t,i},t|\bar{a})
- v^x_{t,i}(x_{t,i}\mid x_{0,i},x_{1,i},\bar{a})\|^2
\Big],\nonumber
\end{align}

where $\pi(x_0,x_1)$ is the minibatch OT coupling (as in FoldFlow-OT), 
and $\bar{a}=a\odot m$ is the masked sequence with mask 
$m\sim\mathrm{Bern}(0.5)$ applied uniformly across all residues. 
This stochastic masking enables a single model to handle both 
conditional and unconditional generation:
\begin{itemize}
    \item $\bar{a}=[\varnothing]^N$ (fully masked, probability 0.5): 
    unconditional backbone generation, equivalent to FoldFlow-OT.
    \item $\bar{a}=a$ (unmasked, probability 0.5): 
    sequence-conditioned generation, i.e.\ protein folding.
    \item $\bar{a}=a\odot m$ (partially masked): 
    structure in-painting and motif scaffolding.
\end{itemize}

    
    

\textbf{Reinforced Fine-Tuning (ReFT).}
FoldFlow-2 further introduces a fine-tuning objective to align 
generations towards an auxiliary reward $r_{\mathrm{aux}}$. 
Given a preferential dataset $\mathcal{D}_{\mathrm{pref}}$ 
filtered by $r_{\mathrm{aux}}$, the ReFT objective is:
\begin{align}
    \max_{p_\theta}\;\mathcal{L}_{\mathrm{ReFT}}(\theta)
    = \mathbb{E}_{(x,a)\sim\mathcal{D}_{\mathrm{pref}}}
    \left[r_{\mathrm{aux}}(x)\log p_\theta(x\mid a)\right].
\end{align}
This is applied, for example, to improve secondary structure diversity 
by upweighting samples rich in $\beta$-sheets and coils.

\subsection{ReQFlow \cite{yue2025reqflow}}

\citet{yue2025reqflow} propose a quaternion-based flow model for protein
backbone generation. Similar to FrameFlow and FoldFlow, the protein
backbone is represented as a collection of residue-wise rigid frames in
$\SE^N$. However, instead of representing rotations by matrices
in $\SO$, ReQFlow parameterizes each residue frame as
$$
g_i = (x_i, q_i), \qquad x_i \in \mathbb{R}^3,\;\; q_i \in \mathbb{S}^3,
$$
where $x_i$ is the local translation and $q_i$ is a unit quaternion 
representing the 3D rotation. 

In particular, given a rotation matrix $R\in \SO$, let $\omega=\phi u=(\log(R))^\vee\in\mathbb{R}^3$ be its axis-angle vector, where $\phi=\|\omega\|,u=\frac{\omega}{\|\omega\|}$. The corresponding unit quaternion is given by
\begin{align}
&q=\exp(\frac{1}{2}\omega)=[\cos\frac{\phi}{2},\sin \frac{\phi}{2}u^\top]^\top \in \mathbb{S}^3.  \nonumber
\end{align}
Conversely, given a unit quaternion $q=(w,x,y,z)^\top \in \mathbb{S}^3$, the corresponding rotation matrix is
$$
R=\exp((2\log(q))^\wedge).
$$

Moreover, given 3D rotations $R_1,R_2$ with corresponding quaternion representations $q_1:=(s_1,u_1),q_2:=(s_2,u_2)$, their group action (matrix multiplication) can be expressed by quaternion multiplication:  
$$
q_1\otimes q_2 = \begin{bmatrix}
s_1s_2-u_1^\top u_2 \\ 
s_1u_2+s_2u_1+u_1\times u_2 
\end{bmatrix}.
$$

This quaternion algebra leads to a more numerically stable and efficient
treatment of rotations, especially when the rotation angle is very small
or close to $\pi$.

\paragraph{Quaternion Flow Matching (QFlow).}
In ReQFlow, the authors denote by $\mathcal{T}_0$ the prior distribution
 $$
\mathcal{N}(0,I_3)\times \mathcal{IG}_{SO(3)},
$$
where, with a slight abuse of notation, $\mathcal{IG}_{SO(3)}$ denotes the
isotropic Gaussian distribution in rotation space under the quaternion
representation. The target distribution $\mathcal{T}_1$ corresponds to the
real protein data distribution.

ReQFlow parameterizes the model as
 $$
T_{\theta,1}(x_t,q_t)=(x_{\theta,1},q_{\theta,1}),
$$
where $x_{\theta,1}$ and $q_{\theta,1}$ denote the predicted translation
and rotation at terminal time $t=1$, conditioned on the current state
$(x_t,q_t)$.

At time $t\in[0,1]$, the translation path follows the standard linear
interpolation
 $$
x_t=(1-t)x_0+tx_1,
$$
with translation velocity
 $$
v_t^x=x_1-x_0=\frac{x_1-x_t}{1-t},
\qquad
v_{\theta,t}^x=\frac{x_{\theta,1}-x_t}{1-t}.
$$

For the rotational component, ReQFlow adopts quaternion geodesic
interpolation:
 $$
q_t=q_0\otimes \exp\!\bigl(t\log(q_0^{-1}\otimes q_1)\bigr).
$$
The corresponding angular velocity and its model-based estimate are
\begin{align}
&v_t^q
=
2\log(q_0^{-1}\otimes q_1)
=
\frac{2\log(q_t^{-1}\otimes q_1)}{1-t},\nonumber\\
&v_{\theta,t}^q=\frac{2 \log (q_t^{-1}\otimes q_{\theta,1})}{1-t}.\nonumber 
\end{align}

The inference dynamics are then approximated by
 $$
\begin{cases}
x_{t+\Delta t}=x_t+\Delta t\, v_{\theta,t}^x,\\
q_{t+\Delta t}=q_t\otimes \exp\!\left(\frac{1}{2}\Delta t\, v_{\theta,t}^q\right).
\end{cases}
$$
Accordingly, the flow-matching loss is defined as
\begin{align}
\mathcal{L}_{\mathrm{QFlow}}
&=\mathbb{E}_{t,(x_0,q_0),(x_1,q_1)}\bigl[\|v^{x}_t-v^{x}_{\theta,t}\|^2\bigr]+\mathbb{E}_{t,(x_0,q_0),(x_1,q_1)}\bigl[\|v^{q}_t-v^{q}_{\theta,t}\|^2\bigr].\nonumber
\end{align}
In addition, the auxiliary loss (see section \ref{sec:auxiliary_loss}) is incorporated into the training process.

\subsection{Riemannian Gaussian Variational Flow Matching (RG-VFM) \cite{zaghen2025riemannian}}

\citet{zaghen2025riemannian} approach manifold generation from the \emph{variational}
rather than the velocity-matching side. Building on Variational Flow Matching, they
replace the instantaneous-velocity regression used by CFM/RFM---and, in our setting, by
FrameFlow, FoldFlow, and ReQFlow---with an \emph{endpoint} objective: a network predicts a
terminal mean $\mu_\theta(x_t)\in\mathcal{M}$, and the posterior $q_\theta(x_1\mid x_t)$ is
modeled as a Riemannian Gaussian
\[
\mathcal{N}_{\mathrm{Riem}}\big(x_1\mid \mu_\theta(x_t),\sigma\big)
\;\propto\;
\exp\!\Big(-\tfrac{\mathrm{dist}_g\big(x_1,\mu_\theta(x_t)\big)^2}{2\sigma^2}\Big).
\]
On a homogeneous manifold with closed-form geodesics, the normalizing constant is
independent of $\mu$, and the objective collapses to a squared geodesic (endpoint)
distance:
\begin{align}
\mathcal{L}_{\mathrm{RG\text{-}VFM}}(\theta)
&=\mathbb{E}_{t,x_1,x_t}\big[\|\log_{x_1}(\mu_\theta(x_t))\|_g^2\big]=\mathbb{E}_{t,x_1,x_t}\big[\mathrm{dist}_g\big(x_1,\mu_\theta(x_t)\big)^2\big],\nonumber
\end{align}
which recovers the Euclidean VFM/MSE loss $\|\mu_\theta(x_t)-x_1\|^2$ when
$\mathcal{M}=\mathbb{R}^d$. Unlike vanilla RFM, whose vector field lives in $T_{x_t}\mathcal{M}$
and requires $\mathrm{supp}(p_0)$ to lie on $\mathcal{M}$, the variational objective compares
endpoints in the single tangent space $T_{x_1}\mathcal{M}$ and only needs the local geometry
around $p_1$.

\textbf{Adaptation to the backbone setting (``variational ReQFlow'').}
Instantiating this objective in ReQFlow's frame representation $g_i=(x_i,q_i)\in
\mathbb{R}^3\times\mathbb{S}^3$ yields an endpoint-matching counterpart of QFlow. Rather than
regressing the translation/quaternion velocities as in $\mathcal{L}_{\mathrm{QFlow}}$, the
model predicts the terminal frame $(x_{\theta,1},q_{\theta,1})$ and minimizes the endpoint
geodesic distance directly:
\begin{align}
\mathcal{L}_{\mathrm{v\text{-}ReQFlow}}(\theta)
&=\mathbb{E}_{t}\Big[\,\sum_{i=1}^N\|x_{1,i}-x_{\theta,1,i}\|^2
+\lambda\sum_{i=1}^N \mathrm{dist}_{\SO}\big(q_{1,i},q_{\theta,1,i}\big)^2\,\Big], \nonumber\\
&\quad \mathrm{dist}_{\SO}(q_1,q_{\theta,1})=2\big\|\log(q_{\theta,1}^{-1}\otimes q_1)\big\|,\nonumber
\end{align}
i.e.\ the flow-matching term is replaced by a squared endpoint distance on the 
product space $\SE^N=(\SO \times\mathbb{R}^3)^N$.

\textbf{Relation to our work.}
At the loss level, RG-VFM can be read as a variant of ReQFlow in which the velocity-matching
FM objective is replaced by an endpoint geodesic distance: the network still predicts a
terminal mean $\mu_\theta(x_t)$ (an $x_1$-prediction parameterization), so sampling remains an
ODE integration and inherits the same limited few-step generation behavior as standard flow
matching. Our method is instead a MeanFlow-type model that directly learns the
average-velocity flow map $\Phi_{s,t}$ via the time-ordered exponential
(Eq~\ref{eq:avg_omega}), which is what enables few-step generation. The two are
nonetheless complementary rather than opposed: in the second phase of our training
(Section~\ref{sec:training-details}), we combine this endpoint geodesic loss with the MeanFlow
objective, using the former to stabilize the JVP (total-derivative) term in the MeanFlow loss. Finally, RG-VFM is validated only on a
synthetic spherical dataset and releases no $\SE^N$ protein checkpoint, so we do not include
it as a standalone baseline (see below).

\subsection{Riemannian MeanFlow (RMF) \cite{woo2026riemannian}}\label{sec:rmf}

Concurrent with our work, \citet{woo2026riemannian} generalize MeanFlow from Euclidean space to
Riemannian manifolds, with applications to scientific generative modeling
on manifolds such as the simplex and $\SE^N$. Instead of learning
an instantaneous velocity field and numerically integrating it during
sampling, RMF directly learns the flow map
\[
\Phi_{s,t}:\mathcal{M}\to\mathcal{M},
\]
which transports a point $x_s\sim p_s$ at time $s$ to its corresponding
point $x_t\sim p_t$ at time $t$ along the same integral curve.

The key geometric quantity in RMF is the \emph{average velocity}, defined by
\begin{equation}
    u_{s,t}(x_s)=
\begin{cases}
\dfrac{1}{t-s}\log_{x_s}(x_t), & t\neq s,\\[6pt]
v_s(x_s), & t=s,
\end{cases}
\label{eq:rmf_av}
\end{equation}

where $\log_{x_s}(x_t)\in T_{x_s}\mathcal{M}$ denotes the Riemannian
logarithmic map. Geometrically, $u_{s,t}(x_s)$ is the constant tangent
velocity that transports $x_s$ to $x_t$ along a geodesic over time
interval $t-s$.
By differentiating both sides with respect to $t$, we obtain:
\begin{align}
u_{s,t}(x_s)=d(\log_{x_s})_{x_t}[v_t]-(t-s)\partial_t u_{s,t}(x_s). \label{eq:rmf_loss}    
\end{align}
and it induces the training loss:
$$
\mathcal{L}_{\mathrm{RMF}}(\theta)=\mathbb{E}_{x_s,s,t}[\|u_{s,t}^\theta(x_s)-\text{sg}(\hat{u}_{\mathrm{tgt}})\|^2], 
$$
where $\hat{u}_{\mathrm{tgt}}=(t-s)D_s u_{s,t}^\theta(x_s)-\nabla_{v_s}^1 \log_{x_s} \Phi^\theta(x_s)$
and $\Phi_{s,t}^\theta(x_s)=\exp((t-s)u_{s,t}^\theta(x_s))$.

They further include a cycle-consistency regularizer
$$\mathcal{L}_{\text{cyc}}(\theta)=\mathbb{E}_{x_t,s,t}\big[d_g(\phi_{s,t}^\theta(\phi_{t,s}^\theta(x_t)),x_t)^2\big].$$

\paragraph{Training objective in the protein experiments.}
In the $\SE^N$ protein setting, RMF is trained with a flow-matching loss together with a
semigroup (flow-map composition) consistency loss. The semigroup term enforces
$\Phi_{r,t}=\Phi_{s,t}\circ\Phi_{r,s}$, but measures the residual as an \emph{endpoint
geodesic distance}---the $\SE$ distance between the one-step map $\Phi_{r,t}(x_r)$ and the
composed two-step map $\Phi_{s,t}(\Phi_{r,s}(x_r))$---rather than regressing
log-displacements in the Lie algebra, as in our finite semigroup loss
(Appendix~\ref{sec:semigroup}). The two agree at the minimizer, but differ in
conditioning: the endpoint-distance form couples the rotation and translation branches
through a single $\SE$ metric, whereas our log-displacement form keeps them separated and,
on the rotation branch, makes the BCH/right-Jacobian structure explicit. We compare both
average-velocity formulations on a controlled $\SO^2$ benchmark in
Appendix~\ref{sec:toy}.

\subsection{Riemannian MeanFlow via Parallel Transport (RMF-PT)}
\label{sec:rmf-pt}
Also concurrent with our work, \citet{zhong2026riemannian} extend MeanFlow to general
Riemannian manifolds by defining the average velocity as a parallel-transport integral:
the instantaneous velocities along the trajectory are transported to a common tangent
space before being averaged (Eq.~5 therein). To distinguish it from the Riemannian
MeanFlow of \citet{woo2026riemannian} discussed above, we refer to this method as
\textbf{RMF-PT}. Our approach differs in three respects.

\textbf{(i) Definition.}
RMF-PT defines the mean velocity through a path-dependent parallel-transport integral,
which requires the full trajectory $\{x_\tau\}_{\tau\in[s,t]}$.
We instead define it via the time-ordered exponential \eqref{eq:avg_omega}, which depends
only on the endpoints $R_s$ and $R_t$ and reduces to the logarithmic map
$\Omega^{\mathrm{avg}}=\tfrac{1}{t-s}\log(R_s^\top R_t)$.
The two definitions coincide when the velocity field is constant along the trajectory
(the geodesic case); for general fields they differ by BCH correction terms arising from
the non-commutativity of $\SO$.

\textbf{(ii) Identity and loss.}
Differentiating the two definitions yields structurally different identities. That of
RMF-PT involves the covariant derivative $\nabla_{\dot\gamma}u$ with Christoffel
corrections (Appendix~A of \citealt{zhong2026riemannian}), which their implementation
drops via a log-map approximation. Ours involves the \emph{exact} right Jacobian $J$ of
$\SO$ (Proposition~\ref{pro:Omega_identity}), which admits a closed form and requires no
approximation.

\textbf{(iii) Application.}
RMF-PT targets general manifolds and is evaluated on synthetic data (spheres, tori, and
$\SO$ rotations). Our method is built for $\SE^N$ protein backbone generation: we exploit
the decoupled product geometry $\SO\times\mathbb{R}^3$ (Appendix~\ref{sec:bg-se3}) to
obtain separate, simulation-free objectives for rotations and translations, and validate
them on \emph{de novo} backbone design.


\subsection{Baseline selection}
Since FoldFlow/FoldFlow2 and v-QFlow/v-ReQFlow do not have public checkpoints for the SCOPe
dataset, these methods are not included in the experiments. In addition, FoldFlow and QFlow are
essentially the same in nature---both are flow-matching models---and differ only in the
mathematical representation of rotations (rotation matrices vs.\ quaternions) and in the
backbone architecture. Since QFlow outperforms FoldFlow on the PDB dataset, we consider QFlow
already sufficiently representative of this family. v-QFlow, in turn, can be viewed as a
variant of QFlow whose performance largely matches QFlow's---strong at a moderate number of
sampling steps (e.g.\ $100/500$ steps) but degrading as the step count is reduced. We
therefore select QFlow alone as the representative method.

\section{MeanFlow on $\SE$: Derivation, Theoretical Properties, and $N$-component Implementation}
\subsection{Model and loss function}\label{sec:SE3MF-details}
We first introduce some fundamental results on the Riemannian manifold $\SO$: 

\begin{remark}
\begin{itemize}
\item \textbf{The hat map does not contribute to the derivative.}
Let $A:[0,1]\to\mathbb{R}^3$ be a smooth curve and define
$\Omega(t)=A(t)^{\wedge}$, where $(\cdot)^{\wedge}:\mathbb{R}^3\to\mathfrak{so}(3)$ is the (linear) hat operator, and $(\cdot)^{\vee}:\mathfrak{so}(3)\to\mathbb{R}^3$ is its inverse.
Since $(\cdot)^{\wedge}$ is linear and independent of $t$, the chain rule gives
$$
    \frac{d}{dt}\Omega(t)=\Big(\frac{dA}{dt}\Big)^{\wedge}\in\mathbb{R}^{3\times 3}.
$$

\item \textbf{Fr\'echet derivative of the matrix exponential.}
Let $R(t):[0,1]\to\mathbb{R}^{3\times 3}$ be a smooth matrix curve and define
$F(t)=\exp(R(t))$. By the chain rule on Banach spaces,
$$
    \dot{F}(t)=d(\exp)_{R(t)}[\dot{R}(t)],
$$
where $d(\exp)_R:\mathbb{R}^{3\times 3}\to\mathbb{R}^{3\times 3}$ denotes the
\emph{Fr\'echet derivative of the matrix exponential at $R$} in the direction
$H=\dot{R}(t)$. The derivative admits the integral (Wilcox) representation:
$$
    d(\exp)_R(H)=\int_0^1 e^{(1-\alpha)R}  H  e^{\alpha R} d\alpha.
$$
Equivalently, the derivative of the map $R\mapsto \exp(R)$ is the linear operator
$d(\exp)_R(\cdot)$ defined for all $H\in\mathbb{R}^{3\times 3}$.

\item \textbf{Derivative with respect to a point on a manifold.}
Consider a smooth curve $t\mapsto R(t)$ on the manifold $\SO$ and a smooth function $f:\SO\to\mathbb{R}$. Let $F(t)=f(R(t))$. By the chain rule, we have
$$
    \frac{d}{dt}F(t)
    = df_{R(t)}[\dot{R}(t)]
    = \langle \nabla_R^\mathcal{M} f(R(t)), \dot{R}(t) \rangle,
$$
where $df_{R}$ denotes the Fr\'echet differential of $f$ at $R$, and $\nabla_R^\mathcal{M} f(R)$ denotes the Riemannian gradient at $R$. Equivalently, the differential
$df_R$ is a linear operator defined on the tangent space
$\mathcal{T}_R\SO$ for every $R\in \SO$.
\item \textbf{On $\SO$, the Riemannian gradient coincides with the Euclidean gradient.}

On the Lie group $\SO$, we use the canonical inner product
$\langle\cdot,\cdot\rangle_R:\mathcal{T}_R(\SO)^{\otimes 2}\to\mathbb{R}$ defined by
$$
\langle A,B\rangle = \frac{1}{2} \mathrm{Tr}(A^\top B).
$$
Then, for any smooth function $f:\SO\to\mathbb{R}$ and its smooth extension
$\bar{f}:\mathbb{R}^{3\times 3}\to\mathbb{R}$, we have
$$
\nabla^{\mathcal{M}}_R f = \nabla_R \bar{f},\qquad
\forall R\in \SO\subset \mathbb{R}^{3\times 3}.
$$

\end{itemize}
\end{remark}

With these preliminaries in place, we now define the average angular velocity on $\SO$ and derive the corresponding MeanFlow identity. 

We define the \textbf{average velocity} as
\begin{align}
 \exp(\underbrace{(t-s)\Omega^{\text{avg}}(s,t,R_t,x_t)}_{\Omega^{s\to t}}):=\mathcal{T}\exp\Big(\int_s^t\Omega(\tau,R_\tau,x_\tau)\,d\tau\Big)\label{eq:avg_omega}
\end{align}
\begin{remark}
$\SO$ is not commutative. Thus, in general,
$\mathcal{T}\exp\big(\int_s^t \Omega(\tau,R_\tau,x_\tau)\,d\tau\big)\neq \exp\big(\int_s^t \Omega(\tau,R_\tau,x_\tau)\,d\tau\big)$ unless $\Omega_\tau$ is constant.
Therefore, we cannot directly define $\Omega^{\text{avg}}$ via $\int_s^t \Omega(\tau,R_\tau,x_\tau)\,d\tau$.
\end{remark}
Note that since each $\Omega(\tau,R_\tau,x_\tau)\in \mathfrak{so}(3)$, we have $\Omega^{\text{avg}}(s,t,R_t,x_t)\in \mathfrak{so}(3)$.  

In the extreme case, given $s=0,t=1$, the above mean velocity recovers $R_0$ from $R_1$:
$$\begin{cases}
R_1=R_0\exp(\Omega^{\text{avg}}(0,1,R_1,x_1)) &\text{Forward}\\
R_0=R_1\exp(-\Omega^{\text{avg}}(0,1,R_1,x_1))&\text{Backward}
\end{cases}
$$
We apply this convention to the ground truth as well:
$$
\Omega^{\text{avg}}(s,t,R_t,x_t)=A(s,t,R_t,x_t)^{\wedge}.
$$



 
\begin{remark}
With this model, it is clearer to define the mean velocity via \eqref{eq:avg_omega}, since $\int_s^t\Omega(\tau,R_\tau)\,d\tau$ is identified with an element of $\mathbb{R}^3$, and the above definition is a natural average in Euclidean space.
\end{remark}

With this definition, we can now derive a differential identity relating the average angular velocity to the instantaneous velocity, which will form the basis of our training objective.
\begin{proposition}[Derivative of the averaged angular velocity]\label{pro:Omega_identity}
Under the assumption $s,t$ are independent and using the above notations, set
$A^{s\to t}=(t-s)A^{\text{avg}}(s,t,R_t,x_t)$. 
Taking the derivative of both sides of \eqref{eq:avg_omega} with respect to $t$ yields
\begin{align}
\begin{cases}
\text{L.H.S.}&=R_s^\top R_t \left(J(A^{s\to t})\frac{d}{dt}A^{s\to t}\right)^{\wedge}\\
\text{R.H.S.}&=R_s^\top R_t\Omega(t,R_t,x_t)
\end{cases}.\label{eq:omega_avg_deriv}
\end{align}
Since $R_s^\top R_t$ is invertible, equating the two sides and applying $(\cdot)^\vee$ gives
\begin{align}
J(A^{s\to t})\,\tfrac{d}{dt}A^{s\to t}=\omega_t
\Longleftrightarrow
\tfrac{d}{dt}A^{s\to t}=J^{-1}(A^{s\to t})\,\omega_t ,
\label{eq:omega_identity_two_forms}
\end{align}
the two forms underlying \eqref{eq:mf-se3-loss} and \eqref{eq:mf-se3-loss2}
respectively; the equivalence uses invertibility of $J$
(Remark~\ref{rmk:jacobian-inverse}). Expanding
$\tfrac{d}{dt}A^{s\to t}=A^{\text{avg}}+(t-s)\tfrac{d}{dt}A^{\text{avg}}$ by
\eqref{eq:total_dOmega_dt} recovers \eqref{eq:omgea_identity_main}.
\end{proposition}

\begin{proof}
Differentiate both sides of \eqref{eq:avg_omega} with respect to $t$.

For the right-hand side, the standard derivative rule for the time-ordered exponential gives
\begin{align}
&\frac{d}{dt}\,\mathcal{T}\exp\Big(\int_s^t\Omega(\tau,R_\tau,x_\tau)\,d\tau\Big)\nonumber\\
&=\Big(\mathcal{T}\exp\big(\int_s^t\Omega(\tau,R_\tau,x_\tau)\,d\tau\big)\Big)\,\Omega(t,R_t,x_t)\nonumber\\
&=\exp(\Omega^{s\to t})\,\Omega(t,R_t,x_t),\label{eq:rhs_omega_avg}
\end{align}
where $\exp(\Omega^{s\to t})=R_s^\top R_t$ by definition.

For the left-hand side, the Wilcox formula for the Fr\'echet derivative of the matrix exponential gives
\begin{align}
&\frac{d}{dt}\exp(\Omega^{s\to t})\nonumber\\
&=\exp(\Omega^{s\to t})\int_0^1 \exp\big(-\alpha\Omega^{s\to t}\big)\Big(\frac{d}{dt}\Omega^{s\to t}\Big)\exp\big(\alpha\Omega^{s\to t}\big)\,d\alpha\label{pf:wilcox}\\
&=\exp(\Omega^{s\to t})\int_0^1 \exp\big(-\alpha\Omega^{s\to t}\big)\Big(\frac{d}{dt}A^{s\to t}\Big)^{\wedge}\exp\big(\alpha\Omega^{s\to t}\big)\,d\alpha\nonumber\\
&=\exp(\Omega^{s\to t})\int_0^1 \left(\exp\big(-\alpha (A^{s\to t})^{\wedge}\big)\,\frac{d}{dt}A^{s\to t}\right)^{\wedge} \,d\alpha\nonumber\\
&=\exp(\Omega^{s\to t})\left(\int_0^1 \exp\big(-\alpha (A^{s\to t})^{\wedge}\big)\,d\alpha\;\frac{d}{dt}A^{s\to t}\right)^{\wedge}.\label{eq:lhs_omega_avg}
\end{align}
Here the second line uses $\Omega^{s\to t}=(A^{s\to t})^{\wedge}$ and the linearity of $(\cdot)^{\wedge}$.
The third line uses the conjugation identity $Qx^{\wedge}Q^\top=(Qx)^{\wedge}$ for $Q\in\SO$ and $x\in\mathbb{R}^3$.
The last line follows because $(\cdot)^{\wedge}$ is linear and independent of $\alpha$. 
We complete the proof by the identity of right Jacobian defined in \eqref{eq:right_jacobian}
\begin{align}
  J(A^{s\to t})=\int_0^1\exp(-\alpha(A^{s\to t})^\wedge)d\alpha.\nonumber
\end{align}
\end{proof}

\begin{remark}[Inverse of the right Jacobian]\label{rmk:jacobian-inverse}
Since $A^{\wedge}$ is skew with eigenvalues $0,\pm i\phi$ where $\phi:=\|A\|_2$, the
matrix $J(A)=\int_0^1 e^{-\alpha A^{\wedge}}d\alpha$ has eigenvalues $1$ and
$(1-e^{\mp i\phi})/(\pm i\phi)$, of modulus $2\sin(\phi/2)/\phi$, all nonzero for
$\phi\le\pi$. Hence $J$ is invertible on the principal branch, with
\begin{align}
J(A)^{-1}=I+\tfrac12 A^{\wedge}
+\Big(\frac{1}{\phi^{2}}-\frac{1+\cos\phi}{2\phi\sin\phi}\Big)(A^{\wedge})^{2},
\nonumber
\end{align}
and $J(A)^{-1}=I+\tfrac12A^{\wedge}+\tfrac1{12}(A^{\wedge})^{2}+O(\phi^{4})$ for small
$\phi$, which we use below the numerical threshold. This covers the model output: with
the endpoint parameterization $A^{\text{avg}}_\theta=\tfrac1t\log(\cdot)^{\vee}$ of
Section~\ref{sec:model-param} the principal branch gives
$\|A_\theta^{s\to t}\|_2=\tfrac{t-s}{t}\,\|\log(\cdot)^{\vee}\|_2\le\pi$, so
\eqref{eq:mf-se3-loss2} is well-defined.
\end{remark}

\begin{remark}
If the angular velocity is constant, i.e.\ $\Omega_\tau=\omega^{\wedge}$ for all $\tau$,
then $A^{s\to t}=(t-s)\omega$ and $\tfrac{d}{dt}A^{s\to t}=\omega$. Since
$\omega^{\wedge}\omega=\omega\times\omega=0$, every term of the series for $J$ beyond the
constant one annihilates $\omega$, so
$$
J(A^{s\to t})\,\tfrac{d}{dt}A^{s\to t}=J\big((t-s)\omega\big)\,\omega=\omega,
$$
and likewise $J^{-1}(A^{s\to t})\,\omega=\omega$: the Jacobian does not distort the
velocity when the axis is fixed, and the two forms of
\eqref{eq:omega_identity_two_forms} coincide.
\end{remark}

\begin{figure}[t]
  \centering
  \includegraphics[width=1.0\linewidth]{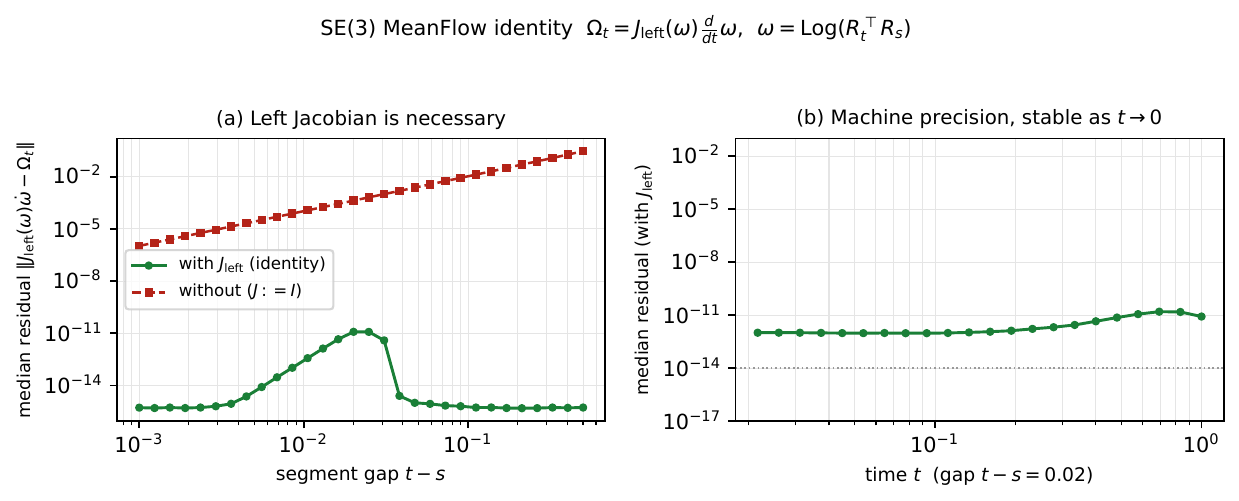}
  \caption{Numerical verification of Proposition~\ref{pro:Omega_identity}.
  On an analytic, non-geodesic $\mathrm{SO}(3)$ path we evaluate both sides of
  \eqref{eq:omega_avg_deriv} and plot the median residual
  $\lVert J(A^{s\to t})\,\tfrac{d}{dt}A^{s\to t}-\omega(t)\rVert$
  against the segment gap $t-s$. With the Jacobian term $J(A^{s\to t})$ (green) the
  residual stays at float64 machine precision ($\sim\!10^{-15}$, at worst
  $10^{-11}$), uniformly in the gap; dropping it ($J:=I$, red) breaks the identity by
  $\mathcal{O}(t-s)$, confirming both the proposition and the necessity of the
  Jacobian term.}
  \label{fig:omega-identity-numerical}
\end{figure}

We verify Proposition~\ref{pro:Omega_identity} numerically. We sample a smooth,
non-geodesic curve $R(\tau)=\exp\big((a+b\tau+c\tau^{2})^{\wedge}\big)$ on
$\mathrm{SO}(3)$ with $b,c$ non-parallel, so the averaged generator
$A^{s\to t}=\log(R_s^{\top}R_t)$ genuinely differs from the instantaneous
one and $J(A^{s\to t})$ acts non-trivially. For a batch of times we compute the
instantaneous body angular velocity $\omega(t):=\big(\Omega(t)\big)^{\vee}=(R_t^{\top}\dot
R_t)^{\vee}$ and the derivative $\tfrac{d}{dt}A^{s\to t}$ by forward-mode automatic
differentiation, and evaluate $J(A^{s\to t})$ from its closed form, all using the
same $\mathrm{SO}(3)$ exp/log/Jacobian routines as the training loss in double
precision. Figure~\ref{fig:omega-identity-numerical} reports the median residual of
\eqref{eq:omega_avg_deriv} over the batch versus $t-s$.

The identity in Proposition~\ref{pro:Omega_identity} involves the total time derivative $\frac{d}{dt}A^{s\to t}$, which we now compute explicitly.
\begin{proposition}\label{pro:dA/dt}
In the above notation, we have the total derivative
\begin{align}
\frac{d}{dt}A(s,t,R_t,x_t)
&=\partial_t A(s,t,R_t,x_t)+\langle\nabla_R A(s,t,R_t,x_t),\dot{R}_t\rangle+\langle\nabla_x A(s,t,R_t,x_t),\dot{x}_t\rangle.
\nonumber
\end{align}
\end{proposition}
\begin{proof}
By the above remark, for $v\in T_{R_t}\SO$,
$$
d_RA(s,t,R_t,x_t)[v]=\nabla_R A(s,t,R_t,x_t)[v].
$$
Similarly, in the Euclidean component, $d_xA(s,t,R_t,x_t)[w]=\langle\nabla_xA(s,t,R_t,x_t),w\rangle$ for $w\in\mathbb{R}^3$.
Applying the chain rule along the curve $t\mapsto (R_t,x_t)$ gives the claimed identity.
\end{proof}

Substituting $A_\theta^{\text{avg}}(s,t,R_t,x_t)$ into the two identities of
    \eqref{eq:omega_identity_two_forms} and defining
    $A_\theta^{s\to t}=(t-s)A_\theta^{\text{avg}}(s,t,R_t,x_t)$, we obtain the two
    training losses of Section~\ref{sec:model-param}:
{\footnotesize\begin{align}
&\mathcal{L}^{J}_{\SO}
      :=
      \mathbb{E}
\Big[\big\|\,\mathrm{sg}\big(J(A_\theta^{s\to t})\big)
      \big(A_\theta^{\text{avg}}+(t-s)\,\mathrm{sg}(\tfrac{d}{dt}A_\theta^{\text{avg}})\big)
      -\omega_t\,\big\|_2^2\Big],
      \label{eq:loss_so3}\\
&\mathcal{L}^{J^{-1}}_{\SO}
      :=
      \mathbb{E}
      \Big[\big\|\,A_\theta^{\text{avg}}
      -\mathrm{sg}\big(J^{-1}(A_\theta^{s\to t})\,\omega_t\big)
      +\mathrm{sg}\big((t-s)\tfrac{d}{dt}A_\theta^{\text{avg}}\big)\,\big\|_2^2\Big],
      \label{eq:loss_so3_inv}
    \end{align}}
    where the expectation is taken over $s<t$ and $((R_0,x_0),(R_1,x_1))\sim q_{0,1}$, with $q_{0,1}$ a coupling between $p_0=p_{\text{data}}$ and $p_1=p_{\text{prior}}=p_{\text{noise}}$
    (e.g.\ the product measure or the optimal-transport coupling), and the stop-gradient is placed
    so that the gradient flows through $A_\theta^{\text{avg}}$ only. The two residuals are
    related by $J(A_\theta^{s\to t})$ and hence share their zero set; we train with
    \eqref{eq:loss_so3_inv}. In addition, the interpolation $R(t)$ (and its velocity
    $\dot{R}(t)$) is obtained from the geodesic interpolation:
\begin{align}
&\Omega(t)=\log(R_0^\top R_1), R(t)=R(0)\exp (t \Omega(t)),\nonumber\\ &\dot{R}_t=R_t\Omega(t).\nonumber 
\end{align}

The following proposition confirms that minimizing this loss recovers the correct relative rotation between any two points along the path.
\begin{proposition}\label{pro:loss_zero}
If $\mathcal{L}^{J}_{\SO}=0$ or $\mathcal{L}^{J^{-1}}_{\SO}=0$, i.e.
      $$
\big(J(A_\theta^{s\to t})\tfrac{d}{dt}A_\theta^{s\to t}\big)^\wedge
      =
      \Omega(t,R_t,x_t)
      \qquad \forall s<t\in[0,1],
      $$
      then
      $$
      \exp\big((A_\theta^{s\to t})^\wedge\big)=R_s^\top R_t.
      $$
    \end{proposition}

    \begin{proof}
      For simplicity we take $s=0,t=1$.
      Define the candidate reconstruction
      $$
      \tilde{R}_t := R_0\exp((A_\theta^{0\to t})^\wedge).
      $$
      By the closed-form Fr\'echet derivative of the exponential on $\SO$,
      $$
      \frac{d}{dt}\exp((A_\theta^{0\to t})^\wedge)
      =
      \exp((A_\theta^{0\to t})^\wedge)\big(J(A_\theta^{0\to t})\tfrac{d}{dt}A_\theta^{0\to t}\big)^\wedge.
      $$
      Using the condition $\mathcal{L}^J_{\SO}=0$ (or $\mathcal{L}^{J^{-1}}_{\SO}=0$) we have
      $$
      \big(J(A_\theta^{0\to t})\tfrac{d}{dt}A_\theta^{0\to t}\big)^\wedge
      =
      \Omega(t,R_t,x_t).
      $$
      Hence
      $$
      \dot{\tilde{R}}_t
      =
      \tilde{R}_t\Omega(t,R_t,x_t).
      $$
      On the other hand the geodesic interpolation satisfies
      $$
      \dot{R}_t = R_t\Omega(t,R_t,x_t),\qquad R_0\ \text{given}.
      $$
      Thus $R_t$ and $\tilde{R}_t$ solve the same ODE with the same initial condition.
      By uniqueness of solutions on $\SO$ we obtain
      $$
      \tilde{R}_t=R_t,\qquad \forall t\in[0,1].
      $$
      Taking $t=1$ yields
      $$
      R_0\exp((A_\theta^{0\to 1})^\wedge)=R_1
      \quad\Longleftrightarrow\quad
      \exp((A_\theta^{0\to 1})^\wedge) = R_0^\top R_1.
      $$
      For general $s<t$ the same argument gives
      $$
      \exp((A_\theta^{s\to t})^\wedge) = R_s^\top R_t,
      $$
      which completes the proof.
    \end{proof}

To implement the loss in \eqref{eq:loss_so3}, it remains to compute the total derivative $\frac{d}{dt}A_\theta^{s\to t}$.

\begin{proposition}\label{pro:jvp_equivalent}
From the product rule and chain rule, we have
\begin{align}
&\frac{d}{dt}(t-s)A_\theta(s,t,R_t,x_t)\nonumber\\
&=A_\theta(s,t,R_t,x_t)+(t-s)\Big(\partial_t A_\theta(s,t,R_t,x_t)+\langle\nabla_R A_\theta(s,t,R_t,x_t),\dot{R}_t\rangle
      +\langle\nabla_x A_\theta(s,t,R_t,x_t),\dot{x}_t\rangle\Big)
      \label{eq:total_dOmega_dt}
\end{align}

    \end{proposition}
\begin{proof}
Fix $s,t$ and $(R_t,x_t)$. For convenience, we denote
$$
A_\theta:=A_\theta(s,t,R_t,x_t).
$$

By definition, $(t-s)A_\theta$ is a product of the scalar function $t-s$ and the vector-valued function
$A_\theta(s,t,R_t,x_t)$. Hence, by the product rule,
\begin{align}
\frac{d}{dt}(t-s)A_\theta
&=\frac{d}{dt}(t-s)\,A_\theta+(t-s)\frac{d}{dt}A_\theta \nonumber\\
&=A_\theta+(t-s)\frac{d}{dt}A_\theta. \label{pf:product_rule}
\end{align}

Next, we compute the total derivative of $A_\theta(s,t,R_t,x_t)$ with respect to $t$.
Since $A_\theta$ depends on $t$ explicitly and implicitly through both $R_t$ and $x_t$,
by the chain rule,
\begin{align}
&\frac{d}{dt}A_\theta(s,t,R_t,x_t)
=\partial_t A_\theta(s,t,R_t,x_t)+d_RA_\theta(s,t,R_t,x_t)[\dot R_t]
+d_xA_\theta(s,t,R_t,x_t)[\dot x_t]. \label{pf:chain_rule}
\end{align}

Substituting \eqref{pf:chain_rule} into \eqref{pf:product_rule} yields
\begin{align}
&\frac{d}{dt}(t-s)A_\theta(s,t,R_t,x_t)
=A_\theta+(t-s)\Big(\partial_t A_\theta
+\langle\nabla_RA_\theta,\dot R_t\rangle
+\langle\nabla_xA_\theta,\dot x_t\rangle\Big). \nonumber
\end{align}

This proves \eqref{eq:total_dOmega_dt}.
\end{proof}

\subsection{Practical implementation: $\SE^N$ adaptation details}
Recall the $\SE^N$ parametrization of the main text: $f_\theta(s,t,R^N_t,x^N_t)=[A_\theta^{\text{avg},N};v_\theta^{\text{avg},N}]$, with the $i$-th block $A_{\theta,i}^{\text{avg},N}\in\mathbb{R}^{3}$ parametrizing the $i$-th $\mathfrak{so}(3)$ component. We now write the two per-component losses explicitly.

\paragraph{Independent $\mathfrak{so}(3)^N$ loss.}
Extending \eqref{eq:loss_so3} to $\mathfrak{so}(3)^N$ by treating each component
independently gives
\begin{align}
&\mathcal{L}_{\mathfrak{so}(3)^N}
:=\mathbb{E}_{\substack{s<t,\\ q(R^N_0,R^N_1)}}
\sum_{i=1}^{N}
\left\|
\mathrm{sg}\Big(J\big((t-s)A_{\theta,i}^{\mathrm{avg}}\big)\Big)\Big(A_{\theta,i}^{\mathrm{avg}}
+(t-s)\,\mathrm{sg}\big(\tfrac{d}{dt}\hat{A}_{\theta,i}^{\mathrm{avg}}\big)\Big)
-\omega_{t,i}
\right\|_2^2,\label{eq:mf-so3N-loss}
\end{align}
where $\omega_{t,i}:=\Omega_i(t,R_{t,i},x_{t,i})^{\vee}$ is the data instantaneous
angular velocity (cf.\ Eq.~\eqref{eq:omega_t-def}) and $J(\cdot)$ is the $\SO$ right
Jacobian applied per component. The $J^{-1}$ form \eqref{eq:loss_so3_inv} extends in
the same way, with the $i$-th summand replaced by
$\big\|A_{\theta,i}^{\mathrm{avg}}
-\mathrm{sg}\big(J^{-1}\big((t-s)A_{\theta,i}^{\mathrm{avg}}\big)\omega_{t,i}
-(t-s)\tfrac{d}{dt}\hat A_{\theta,i}^{\mathrm{avg}}\big)\big\|_2^2$.

\paragraph{Translation $\mathbb{R}^{3\times N}$ loss.}
For the translation component, the loss in Euclidean space \eqref{eq:mf_loss} extends directly to $N$ independent particles. With the model prediction $v_{\theta,i}^{\text{avg}}\in\mathbb{R}^3$ for residue $i$, the translation loss is
{\small
\begin{align}
\mathcal{L}_{\mathbb{R}^{3\times N}}
&:=\mathbb{E}_{\substack{s<t,\\ q_{0,1}}}
\sum_{i=1}^{N}
\Big\|
v_{\theta,i}^{\text{avg}}+(t-s)\,\text{sg}\!\left(\frac{d}{dt}\hat{v}_{\theta,i}^{\text{avg}}\right)-v_i(t,x_{t,i})
\Big\|_2^2,\label{eq:mf-R3N-loss}
\end{align}}
where $v_i(t,x_{t,i})=x_{1,i}-x_{0,i}$ is the constant ground-truth translation velocity along the linear interpolation path $x_{t,i}=(1-t)x_{0,i}+t x_{1,i}$. The combined MeanFlow loss on $\SE^N$ is therefore $\mathcal{L}^{\SE^N}_{\mathrm{MF}} = \mathcal{L}_{\mathfrak{so}(3)^N} + \mathcal{L}_{\mathbb{R}^{3\times N}}$.

\section{Common Settings: Auxiliary Loss, Prior and Inference}

\subsection{Auxiliary Loss}\label{sec:auxiliary_loss}

We include auxiliary losses from \citet{yim2023se} to enforce geometric 
consistency at the atomic level. Specifically, let $\mathcal{A}_0 \in 
\mathbb{R}^{N \times 4 \times 3}$ denote the ground-truth backbone atom 
coordinates (in \AA) of the four heavy atoms $(\text{N}, \text{C}_\alpha, 
\text{C}, \text{O})$ per residue, and let $\hat{\mathcal{A}}_0$ be the 
corresponding coordinates predicted by the model. We define a direct regression 
loss on backbone atom positions and a pairwise distance loss in a local 
neighborhood,
{\small
\begin{equation}
    \mathcal{L}_{\mathrm{bb}} = \frac{1}{4N} \sum \|\mathcal{A}_0 - 
    \hat{\mathcal{A}}_0\|^2,
    \mathcal{L}_{\mathrm{2D}} = \frac{\|\mathbf{1}\{D < 6\text{\AA}\}
    (D - \hat{D})\|^2}{\sum \mathbf{1}_{D < 6\text{\AA}} - N},
\end{equation}}
where $D \in \mathbb{R}^{N \times N \times 4 \times 4}$ is the tensor of 
pairwise distances between heavy atoms, i.e.\ $D_{ijab} = \|\mathcal{A}_{ia} - 
\mathcal{A}_{jb}\|$, and $\hat{D}$ is defined analogously from 
$\hat{\mathcal{A}}_0$. The auxiliary loss is then
\begin{equation}
    \mathcal{L}_{\mathrm{aux}} = \mathbb{E}_{\mathcal{Q}}\!\left[\mathcal{L}_{
    \mathrm{bb}} + \mathcal{L}_{\mathrm{2D}}\right],
\end{equation}
where $\mathcal{Q}(t, x_0, x_1, \tilde{x}_t) := \mathcal{U}(0,1) \otimes 
\bar{\pi}(x_0, x_1) \otimes \rho_t(\tilde{x}_t \mid x_0, x_1)$ is the 
factorized joint distribution. And following the settings in \citet{bose2024se,yim2023se}, we only apply $\mathcal{L}_{\mathrm{aux}}$ for 
$t < 0.75$, scaling it by $\lambda_{\mathrm{aux}}$. 


\subsection{Prior Distribution}\label{sec:prior}
Following the notation in \citet{geng2026mean,geng2026improved}, we denote protein backbones by $(R_0^N,x_0^N)$ and sample $(R_1^N,x_1^N)\in \SO^N\times \mathbb{R}^{N\times 3}$ from a simple prior.

\paragraph{Translation prior.}
In Euclidean space, the natural analogue of a ``standard'' prior is an isotropic Gaussian,
\begin{equation}
  x_1^N \sim \mathcal{N}\bigl(0,\sigma_x^2 I_{N\times 3}\bigr),
\end{equation}
where we identify $x_1^N\in\mathbb{R}^{N\times 3}$ with a vector in $\mathbb{R}^{3N}$ and typically set $\sigma_x=1$.

\paragraph{Rotation prior: Gaussian analogue on $\SO$.}
On the rotation group, the closest analogue of an isotropic Gaussian is an \emph{isotropic (heat-kernel) Gaussian} on $\SO$, denoted $\mathrm{IGSO(3)}(\sigma_R)$. It is the transition density of Brownian motion on $\SO$ at time $\sigma_R^2$, and is isotropic in the sense that its density depends only on the geodesic rotation angle
\begin{equation}
  \theta(R) := \|\log(R)\|_2 \in [0,\pi],\log:\SO\to\mathfrak{so}(3)\cong\mathbb{R}^3,
\end{equation}
with respect to the Haar measure $\mathrm{d}R$.

Let $\mathrm{IGSO}(3)^N$ denote the product measure of $\mathrm{IGSO}(3)$. Concretely, we use the factorized prior
\begin{equation}
p(R_1^N,x_1^N)=\mathrm{IGSO}(3)(\sigma_R)^N\times \mathcal{N}(0,I_{N\times 3}).
\end{equation}
The typical choice of $\sigma_R$ is $1.5$. In addition, as $\sigma_R\to\infty$, $\mathrm{IGSO(3)}(\sigma_R)$ approaches the uniform (Haar) distribution on $\SO$, which is another typical choice of the prior.

\begin{algorithm}[t]
\caption{MeanFlow training for $\SE^N$ (idealized form). This is the
plain MeanFlow objective, stated to make the structure of the method
explicit; it is \emph{not} the recipe used for our reported results.
The trained objective is the Stage-2 loss~\eqref{eq:stage2-loss}---which
adds the endpoint anchor and the auxiliary term---with the rotation and
translation losses computed in the small-$t$ stabilized form of
Algorithm~\ref{alg:stable-derivative-B}, after the $\alpha$-Flow warm-up
of Algorithm~\ref{alg:af-training}. See Appendix~\ref{sec:training-details}.}
\label{alg:meanflow-training}
\begin{algorithmic}[1]
\Require Data distribution $p_{\text{data}}$ over $\SE^N$; MeanFlow network $f_\theta$; OT-coupling flag
\For{each training iteration}
  \State Sample a mini-batch $\{(R_0^N,x_0^N)\}_{b=1}^B \sim p_{\text{data}}$.
  \Statex \hspace{\algorithmicindent} $R_0^N\in\mathbb{R}^{B\times N\times 3\times 3}$, $x_0^N\in\mathbb{R}^{B\times N\times 3}$.
  \State Sample $(R_1^N,x_1^N)$ from the prior; see Section~\ref{sec:prior}.
  \If{OT-coupling}
    \State Solve $OT((R_0,x_0),(R_1,x_1))$ and couple $(R_0,x_0)$ and $(R_1,x_1)$ using the optimal transport plan.
  \EndIf
  \State Sample times $0\le s<t\le 1$.
  \State Forward-interpolate to time $t$ to obtain $(R_t^N,\Omega_t^N)$ and $(x_t^N,v_t^N)$.
  \State Evaluate the network: $[A_\theta^{\text{avg},N},v_\theta^{\text{avg},N}]\gets f_\theta(s,t,R_t^N,x_t^N)$.
  \State Compute $\mathcal{L}_{\mathfrak{so}(3)^N}$ via Eq.~\eqref{eq:mf-so3N-loss}.
  \State Compute $\mathcal{L}_{\mathbb{R}^{3\times N}}$ via Eq.~\eqref{eq:mf-R3N-loss}.
  \State Compute the auxiliary loss $\mathcal{L}_{\text{aux}}$; see Section~\ref{sec:auxiliary_loss}.
  \State $\mathcal{L}\gets \mathcal{L}_{\mathfrak{so}(3)^N}+\mathcal{L}_{\mathbb{R}^{3\times N}}+\lambda_{\text{aux}}\,\mathbb{1}[t<0.75]\,\mathcal{L}_{\text{aux}}$.
  \State Update $\theta$ using a gradient step (or any optimizer).
\EndFor
\end{algorithmic}
\end{algorithm}

\subsection{Inference process}\label{sec:inference}

Similar to the original MeanFlow method in Euclidean space, we can adapt multi-step inference for the $\SE^N$ MeanFlow framework.

In particular, let $T\in\mathbb{N}$ denote the number of inference steps and let $\Delta t=\frac{1}{T}$ be the step size. We apply the following updates on $\SO$ and $\mathbb{R}^3$, respectively:
\begin{align}
&R_{t-\Delta t}=R_t e^{-\Delta t A_\theta^\wedge(s,t,R_t,x_t)},\nonumber\\
&x_{t-\Delta t}=x_t-\Delta t  v^{\text{avg}}_\theta(s,t,R_t,x_t).\nonumber
\end{align}
where $t$ decreases from $1$ to $1/T$ and $s=t-\Delta t$. We summarize the procedure in Algorithm~\ref{alg:meanflow-inference}.
\begin{algorithm}[t]
\caption{$\SE^N$ MeanFlow inference}
\label{alg:meanflow-inference}
\begin{algorithmic}[1]
\Require Batch size $B$; pretrained network $f_\theta$; number of steps $T$
\Ensure $(\hat{R}_0^N,\hat{x}_0^N)$
\State $\Delta t=1/T$
\State Sample $(R_1^N,x_1^N)$ from the prior over $\SE^N$.
\State $(R_t^N,x_t^N)\gets (R_1^N,x_1^N)$.
\For{$t = 1,1-\Delta t,\ldots,\Delta t$}
  \State $s\gets t-\Delta t$.
  \State $[A_\theta^{\text{avg},N},v_\theta^{\text{avg},N}]\gets f_\theta(s,t,R_t^N,x_t^N)$.
  \State $R_t^N\gets R_t^N\exp(-\Delta t\,(A_\theta^{\text{avg},N})^\wedge)$.
  \State $x_t^N\gets x_t^N-\Delta t\,v_\theta^{\text{avg},N}$.
\EndFor
\State $\hat{R}^N_0=R_t^N, \hat{x}^N_0=x_t^N$.
\end{algorithmic}
\end{algorithm}

\section{Ablation Study: A Controlled $\mathrm{SO}(3)$ Benchmark for Few-Step Generation}
\label{sec:toy}

Designability on SCOPe confounds the generative objective with the protein-specific IPA
trunk, the auxiliary losses, the self-conditioning recipe, and the folding oracle, so a
few-step win there is hard to attribute. This appendix therefore tests the central claim
of the paper --- that it is the \emph{average-velocity parameterization} that enables
few-step generation, rather than the network, coupling, or protein-specific machinery ---
on a controlled $\mathrm{SO}(3)^2$ benchmark, where we hold all confounders fixed and vary
only the training loss.

A full ablation on real proteins is prohibitively expensive: the design space includes many
hyperparameter combinations (e.g., OT variants in sampling such as per-GPU OT, global OT, or
no OT; loss normalizations and clipping; different loss mixtures; architectural toggles such
as enabling/disabling AdaLN; and optimizer / learning-rate choices). Each setting requires
substantial compute (4$\times$80GB GPUs), long GPU-hour budgets, and careful checkpoint
selection to reliably judge performance. Under our limited training budget we therefore do
\emph{not} run exhaustive real-dataset ablations; instead, on the synthetic
$\mathrm{SO}(3)^2$ benchmark we fix the model, learning rate, $\SE$ interpolator, and
training steps so that performance differences are attributable to the losses themselves.

\subsection{Data}
\label{sec:toy-data}

A sample is a pair of rotations $(R^{(0)},R^{(1)})\in\mathrm{SO}(3)^2$: the rotational
part of a two-residue backbone, with the translation branch switched off. Throughout we
visualize a rotation at its axis-angle coordinate
\begin{align}
w=\log(R)\in\mathbb{R}^3,\qquad \|w\|\le\pi,
\end{align}
so a distribution on $\mathrm{SO}(3)$ becomes a point cloud inside the $\pi$-ball
(Figure~\ref{fig:toy-data}).

The two atoms carry deliberately different geometry.

\paragraph{Atom 0 (\texttt{cube6}): a mode-seeking task.}
A mixture of six isotropic components centered at the six cube-face rotations (rotations
by $\pi/2$ about $\pm e_x,\pm e_y,\pm e_z$, hence $\|w\|=\pi/2$). The modes are sharp and
well separated, so a sampler must \emph{commit} to one basin; hedging between two modes
is immediately visible as mass in the empty region between them.

\paragraph{Atom 1 (\texttt{moons3d}): a curved-manifold task.}
The classical two-moons distribution, lifted into the Lie algebra so that its support is
a pair of interlocking crescents. Here the model must trace a thin, curved,
one-dimensional structure rather than collapse onto a few points.

\begin{figure}[t]
\centering
\includegraphics[width=0.4\textwidth]{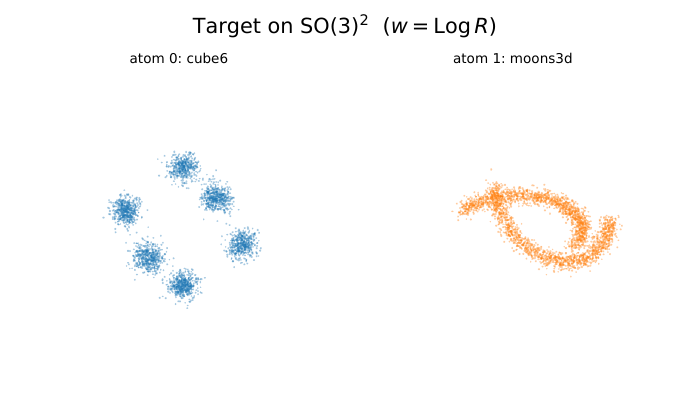}
\caption{The $\mathrm{SO}(3)^2$ target, drawn in axis-angle coordinates
$w=\log(R)$, $\|w\|\le\pi$. \textbf{Left}, atom~0 (\texttt{cube6}): six sharp
modes at the cube-face rotations --- a mode-seeking task. \textbf{Right}, atom~1
(\texttt{moons3d}): two interlocking crescents --- a curved-manifold task. Both must be
solved by the same network under the same objective.}
\label{fig:toy-data}
\end{figure}
\begin{figure*}[t]
\centering
\includegraphics[width=\textwidth]{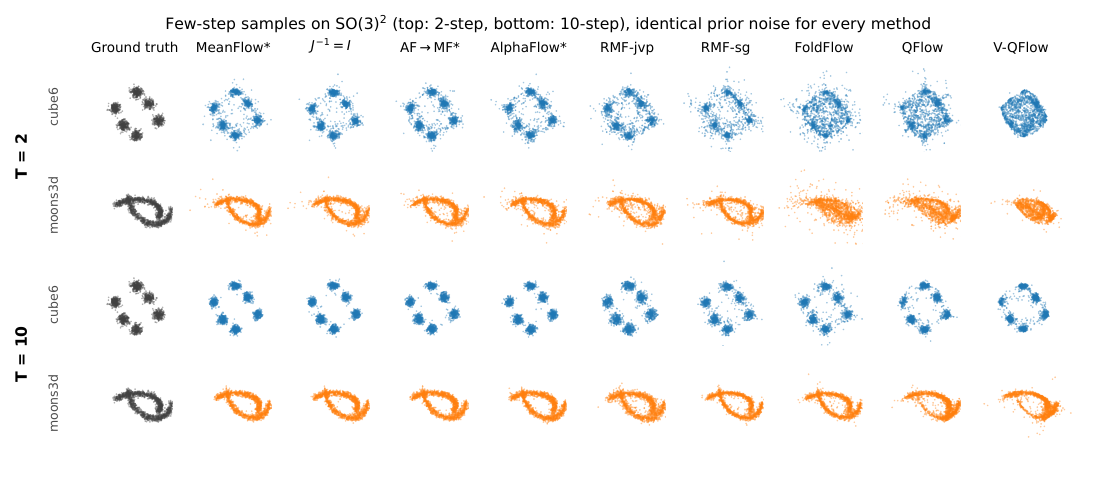}
\caption{Two-step ($T{=}2$) samples, in the same axis-angle coordinates as
Figure~\ref{fig:toy-data}, from the \emph{identical} frozen prior noise for every method.
Top row: atom~0 (\texttt{cube6}); bottom row: atom~1 (\texttt{moons3d}). The
average-velocity methods (ours and both RMF variants) already resolve the six modes and
the crescent geometry after two steps. FoldFlow and QFlow smear across the ball, and
V-QFlow contracts onto a shell.}
\label{fig:toy-samples}
\end{figure*}


\paragraph{Prior.}
The prior is the Haar (uniform) measure on $\mathrm{SO}(3)$, independently per atom. A
Haar draw sits a mean geodesic distance of $126.5^{\circ}$ from the identity, so the
transport distance is large and the prior carries no information about either target. We
draw $10{,}000$ training and $2{,}000$ held-out samples; the noise is resampled every
epoch, from a seeded stream (below).

\subsection{Experiment setting}
\label{sec:toy-controls}

All methods share the following. Only the training loss differs.

\begin{enumerate}
\item \textbf{Network.} The same $1.07$M-parameter $\SE$ MLP for every method,
      mapping $(R_t,t,s)\mapsto\omega\in\mathbb{R}^{2\times3}$, a body-frame angular
      velocity per atom. Single-time methods (FoldFlow) are fed $s{=}t$; endpoint
      methods (QFlow, V-QFlow) use the quaternion head of the same trunk, at matched
      rotation-branch capacity.
\item \textbf{Time sampler.} $t\sim\mathcal{U}(0,1)$ with \emph{no} lower cut-off
      ($t_\varepsilon=0$; the regular-$t$ losses never divide by $t$), and
      $s\sim\mathcal{U}(0,t)$ for the two-time methods, with a $50\%$ $s{=}t$ anchor
      branch. The three-time method (RMF semigroup) additionally draws an interior
      $m\in(s,t)$ and gets no $s{=}t$ branch, since a degenerate interval makes the
      semigroup identity vacuous.
\item \textbf{Budget.} $20$k steps, batch $500$, the same optimizer and learning rate,
      and EMA weights at evaluation.
\item \textbf{Data and noise stream.} \emph{Bit-identical} across methods. This is
      enforced by three independent RNG streams --- batch indices, coupling, and times ---
      so that methods which consume different numbers of random draws (OT draws a
      multinomial; a three-time method draws three uniforms where a single-time method
      draws one) nevertheless see the same $(R_0,R_1)$ pair at every step. With a single shared stream the sequences drift out of sync after the first step, silently, and the comparison is no longer controlled.
\end{enumerate}

\paragraph{No optimal-transport coupling.}
No method uses mini-batch OT. This is deliberate and it is the crux of the protocol: OT
is a flow-\emph{straightening} device. Under an OT coupling the learned probability path
is close to a straight geodesic, in which case the average velocity and the instantaneous
velocity nearly coincide --- so a plain flow-matching model gets few-step generation for
free, and the experiment can no longer distinguish the two parameterizations. Enabling OT
would hand the instantaneous-velocity baselines exactly the property under test. We
therefore run every method \emph{without} OT so that few-step behavior is attributable
to the objective alone.

\subsection{Evaluation Metric}
\label{sec:toy-metric}

We report the Wasserstein-2 distance between $2{,}000$ generated and $2{,}000$ held-out
samples, under the squared geodesic cost summed over the two atoms:
\begin{align}
&c\big((R^{(0)},R^{(1)}),(S^{(0)},S^{(1)})\big)
=\sum_{a\in\{0,1\}} d_{\mathrm{SO}(3)}^2\big(R^{(a)},S^{(a)}\big),\nonumber\\
&d_{\mathrm{SO}(3)}(R,S)=\big\|\log(R^\top S)\big\|,\nonumber
\end{align}
solved as an exact assignment and reported in degrees.


\paragraph{The floor.}
At finite sample size, $\mathcal{W}_2$ between two \emph{independent} draws of the target
is not zero. We measure this floor with the identical estimator and the identical sample
size, and report it in Table~\ref{tab:toy_so3}: $21.70^{\circ}$. No model can go below
it, and a model at the floor is indistinguishable from the target at this sample size.

\paragraph{Sampling.}
Models are integrated on a uniform Euler grid $t\in\mathrm{linspace}(1,0,T)$ for
$T\in\{1,2,5,10,20\}$, with no final-step special case and no $t_{\min}$ floor. Every
method is initialized from the same frozen prior-noise batch, so the sample sets --- and
hence the panels of Figure~\ref{fig:toy-samples} --- are directly comparable.

\begin{table*}[t]
\centering
\begingroup
\small
\setlength{\tabcolsep}{6pt}
\renewcommand{\arraystretch}{1.05}
\begin{tabular}{@{}lccccc@{}}
\toprule
\multirow{2}{*}{Method} & \multicolumn{5}{c}{$\mathcal{W}_2$ ($^{\circ}$) $\downarrow$ at $T$ steps} \\
\cmidrule(lr){2-6}
 & $T{=}1$ & $T{=}2$ & $T{=}5$ & $T{=}10$ & $T{=}20$ \\
\midrule
\multicolumn{6}{@{}l}{\emph{Average-velocity parameterization}}\\
\textbf{SE(3)-MeanFlow (ours)}                & \textbf{30.74} & \textbf{26.29} & 25.83 & 25.77 & 25.95 \\
\quad ablation: $J^{-1}\!:=\!I$ (log-map)  & 31.15 & 26.37 & 25.91 & 25.90 & 25.91 \\
\quad $\alpha$-Flow $\to$ MeanFlow (ours) & 31.13 & 26.67 & 25.30 & 25.31 & 25.35 \\
\quad $\alpha$-Flow (ours)       & 32.16 & 26.81 & 25.63 & 25.51 & 25.63 \\
RMF, JVP form~\cite{woo2026riemannian}        & 32.44 & 27.65 & \textbf{25.26} & 25.21 & 25.18 \\
RMF, semigroup$+$FM~\cite{woo2026riemannian}  & 31.52 & 29.87 & 25.36 & \textbf{24.75} & 24.83 \\
\midrule
\multicolumn{6}{@{}l}{\emph{Instantaneous-velocity / endpoint parameterization}}\\
FoldFlow~\cite{bose2024se}                    & 80.52 & 43.74 & 26.82 & 24.94 & \textbf{24.77} \\
QFlow~\cite{yue2025reqflow}                   & 77.99 & 42.68 & 29.44 & 31.28 & 31.96 \\
V-QFlow~\cite{yue2025reqflow}                 & 90.64 & 51.35 & 31.75 & 32.93 & 32.24 \\
\midrule
\texttt{Floor} & 21.70 & 21.70 & 21.70 & 21.70 & 21.70\\
\bottomrule
\end{tabular}
\endgroup
\caption{Few-step generation on the $\mathrm{SO}(3)^2$ toy benchmark.
$\mathcal{W}_2$ (degrees; lower is better) between $2{,}000$ generated and $2{,}000$
held-out samples, as the number of Euler steps $T$ falls from $20$ to $1$. All methods
share one network, one data/noise stream, one time sampler and one budget
(Section~\ref{sec:toy-controls}); \emph{none} uses an OT coupling. \texttt{Floor} is the
$\mathcal{W}_2$ between two independent draws of the target and is the lowest attainable
value. Best per column in bold. The separation is by \emph{parameterization}, not by
method family: average-velocity methods survive $T{=}1$; instantaneous-velocity and
endpoint-prediction methods collapse. Removing the $\mathrm{SO}(3)$ Jacobian ($J^{-1}\!:=\!I$) leaves the
average-velocity model unchanged within noise for $T\ge2$ and costs $0.4^{\circ}$
at $T{=}1$, where the sampler queries the largest interval.}
\label{tab:toy_so3}
\end{table*}

\subsection{Results}
\label{sec:toy-results}

Table~\ref{tab:toy_so3} shows a clean separation by \emph{parameterization}, not by
method family or by method quality.

\paragraph{20 steps: the fine-grid regime.}
At $T{=}20$ the two parameterizations are indistinguishable in quality: every
average-velocity method, together with FoldFlow, sits within a few degrees of the
$21.70^{\circ}$ floor, and the single best number is in fact FoldFlow's
($24.77^{\circ}$). With a fine enough Euler grid the instantaneous-velocity
parameterization is entirely adequate and our objective buys nothing. The
exceptions are the quaternion-parameterized methods, QFlow and V-QFlow, which
plateau roughly $10^{\circ}$ above the floor. Since QFlow optimizes a
flow-matching loss and V-QFlow an endpoint loss, both should in principle track
FoldFlow in this regime; we attribute the gap to the rotation-matrix backbone used
throughout this ablation, which forces repeated quaternion--matrix conversions
inside the model and the loss. Appendix~\ref{sec:q-practical} reports the same
effect at protein scale.

\paragraph{Few steps: MeanFlow methods outperform flow matching.}
At $T{=}2$ (and $T{=}5$), the MeanFlow-style average-velocity methods --- ours and both
RMF variants --- already approach the $21.70^{\circ}$ floor, with our model being
slightly better overall. In contrast, the flow-matching baselines FoldFlow, QFlow and
V-QFlow exhibit substantially higher errors at these low step counts.
Figure~\ref{fig:toy-samples} corroborates this qualitatively: after two steps our
samples already show six distinct modes and a recognizable pair of crescents, whereas
FoldFlow and QFlow do not.

At $T{=}1$, SE(3)-MeanFlow remains the best method, but all approaches are relatively far from the
floor. This is expected: in the extreme case where two rotations differ by an angle of
$\pi$, the shortest geodesic (and hence the induced velocity) is not well-defined due
to non-uniqueness (analogous to the ambiguity of shortest paths between the north and
south poles on a sphere). This suggests that $\SO$ is intrinsically ill-suited for
one-step generation. However, few-step generation methods can still be applied. 
\paragraph{Effect of the right Jacobian.}
Replacing $J^{-1}$ by the identity --- the log-map approximation
of~\citet{zhong2026riemannian} --- costs $0.41^{\circ}$ at $T{=}1$ and is within
$0.13^{\circ}$ of our model elsewhere (Table~\ref{tab:toy_so3}). This is the
expected pattern: since $J^{-1}(A)=I+\tfrac12A^{\wedge}+\cdots$ acts on
$A^{s\to t}=(t-s)A^{\mathrm{avg}}$, the correction scales with the interval, and a
$T$-step sampler queries only $t-s=1/T$. The same scaling holds on real training
data (Table~\ref{tab:jacobian-rho}). The Jacobian is closed-form and free, so we
retain it; its benefit is concentrated in the aggressive few-step regime.
\paragraph{The $\alpha$-Flow ablation isolates the mechanism.}
Our $\alpha$-Flow objective (Section~\ref{sec:alpha-flow}), trained alone with $\alpha$
annealed $1\to0.1$, performs competitively, but is consistently slightly worse than
SE(3)-MeanFlow at low step counts (Table~\ref{tab:toy_so3}). This is expected: $\alpha$-Flow
is a JVP-free surrogate that relaxes the full average-velocity consistency enforced by
MeanFlow. Appending a MeanFlow phase to the last $40\%$ of the same schedule --- same
network, same budget, same data --- improves the one-step error by about $1^{\circ}$,
with smaller gains at $T{=}2$, indicating that few-step capability is primarily supplied
by the MeanFlow objective.
Training with MeanFlow throughout gives the better one- and two-step error, while the
warm-up variant is marginally better for $T\ge5$; neither ordering is large relative to
the within-block spread. The few-step advantage reported in the main text is therefore
not an artifact of a short fine-tuning phase.

\paragraph{Concurrent work, Riemannian MeanFlow.}
Both RMF variants~\cite{woo2026riemannian} sit in the same block as our method and
perform comparably (RMF semigroup$+$FM reaches $31.52^{\circ}$ at $T{=}1$, statistically
indistinguishable from our $30.74^{\circ}$ at this sample size). We regard this as
supporting the paper's thesis rather than undercutting it: the thesis is about the
average-velocity parameterization, and RMF is an average-velocity method. The
$\SO$-specific contribution of our work --- the exact right-Jacobian identity of
Proposition~\ref{pro:A_identity_main} and the resulting $\SE^N$ objective ---
is evaluated on protein backbones in the main text, where the two methods do separate.

\paragraph{Caveat.}
The numbers in Table~\ref{tab:toy_so3} are from a single seed. The $\sim\!50^{\circ}$ gap
between the two blocks is far beyond any plausible seed variation, but small differences
\emph{within} a block (e.g.\ our $30.74^{\circ}$ versus RMF's $31.52^{\circ}$ at $T{=}1$)
should not be read as a ranking.

\section{Stable Training Framework, Part 1: MeanFlow Loss (Vanilla and Small-$t$ Stabilized)}\label{sec:stable-rescale}
\paragraph{Vanilla (no small-$t$ reparameterization).}
The MeanFlow identity on $\mathrm{SO}(3)$ (Proposition~\ref{pro:A_identity_main}, with full derivation in Appendix~\ref{sec:SE3MF-details}) yields a consistency target involving the trajectory derivative $\tfrac{d}{dt}A^{\mathrm{avg}}_\theta$. Our practical deep learning model is \emph{diff-frame}: it takes the current state $(R_t,x_t)$ and a time embedding (constructed from $(t,s)$) as input, and directly predicts the endpoints $(\hat{R}_0^\theta,\hat{x}_0^\theta)$. We then define the average velocities using the endpoint (``$x$-prediction'') parameterization
\begin{align}
A_{\theta}^{\mathrm{avg}}(s,t,R_t,x_t) &= \frac{1}{t}\,\log\!\big((\hat{R}_0^\theta)^{\top}R_t\big)^{\vee},\\
v_{\theta}^{\mathrm{avg}}(s,t,R_t,x_t) &= \frac{1}{t}\,(x_t-\hat{x}_0^\theta),
\end{align}
which is the $\SE$ analogue of pixel-space MeanFlow. This is the ``normal'' objective; if one samples times bounded away from $0$ and computes $\tfrac{d}{dt}A_{\theta}^{\mathrm{avg}}$ via a JVP, it can be used directly.

\paragraph{Motivation: small-$t$ numerical stability.}
Both branches contain explicit $1/t$ factors and the rotation target further requires differentiating through $\log((\hat R_0^\theta)^\top R_t)$, so naive autodiff becomes numerically fragile as $t\to 0$. Below we describe our implementation that preserves the same objective but avoids explicit $1/t$ during target construction.

\subsubsection{Positive $t_{\min}$ and last-step endpoint prediction}
Following prior work we enforce a strictly positive $t_{\min}>0$ and only sample training times from $[t_{\min},1]$. In our method we extend the lower bound to $t_{\min}=10^{-6}$.

At inference time, we integrate on a fixed grid $t\in\mathrm{linspace}(1,t_{\min},T)$. In the final iteration, we set $(t,s)=(t_{\min},0)$ and directly use the model's endpoint prediction $(\hat{R}_0^\theta,\hat{x}_0^\theta)$ as the generated sample rather than performing another Euler update. For the small-$t$ consistency losses we use a denominator clamp $t_\varepsilon$ when dividing by $t^2$ (only applied at the final loss normalization step).

\subsection{Auxiliary $B$-variables for stable time derivatives}
The MeanFlow consistency loss involves a time derivative term of the form $(t-s)\,\tfrac{d}{dt}A_{\theta}^{\mathrm{avg}}$ (and similarly for $v_{\theta}^{\mathrm{avg}}$). Directly differentiating $A_{\theta}^{\mathrm{avg}}=\tfrac{1}{t}\log((\hat{R}_0^\theta)^\top R_t)^{\vee}$ can lead to large gradients when $t$ is small.

We therefore introduce auxiliary variables that absorb the problematic factor $t$,
\begin{align}
B^A_{\theta}(s,t,R_t,x_t) &:= t\,A_{\theta}^{\mathrm{avg}}(s,t,R_t,x_t),\\
B^v_{\theta}(s,t,R_t,x_t) &:= t\,v_{\theta}^{\mathrm{avg}}(s,t,R_t,x_t),
\end{align}
and compute time-derivative information via $\tfrac{d}{dt}B^A_{\theta}$ and $\tfrac{d}{dt}B^v_{\theta}$. In implementation we avoid dividing by $t$ throughout the derivative-target construction; all $1/t$ factors are deferred, and we only apply the $t^{-2}$ normalization after forming the squared residuals (with $t$ clamped below by $t_\varepsilon$).
Using $B^A_{\theta}=tA_{\theta}^{\mathrm{avg}}$ and $h=t-s$, one can rewrite
\begin{align}
A_{\theta}^{s\to t} &= h\,A_{\theta}^{\mathrm{avg}} = \frac{h}{t}\,B^A_{\theta},\nonumber\\
 t\,\frac{d}{dt}A_{\theta}^{s\to t} &= B^A_{\theta} + h\Big(\frac{d}{dt}B^A_{\theta}-A_{\theta}^{\mathrm{avg}}\Big)
   \nonumber\\
   &= B^A_{\theta} + \Big(h\,\frac{d}{dt}B^A_{\theta}-\frac{h}{t}B^A_{\theta}\Big).\label{eq:Bt}
\end{align}
where $\frac{h}{t}\leq 1$. 
In implementation, we absorb the factor $h$ into the JVP tangents, so the computed JVP output corresponds to $h\,\tfrac{d}{dt}B^A_{\theta}$ directly (and analogously for $B^v_{\theta}$).
Here $\tfrac{d}{dt}B^A_{\theta}$ is the \emph{total} time derivative of $B^A_{\theta}(s,t,R_t,x_t)$ (with $s$ held fixed). By the chain rule,
\begin{align}
h\frac{d}{dt}B^A_{\theta} &= \frac{\partial B^A_{\theta}}{\partial t}
+ \Big\langle \frac{\partial B^A_{\theta}}{\partial R_t},\,h\frac{dR_t}{dt}\Big\rangle
+ \Big\langle \frac{\partial B^A_{\theta}}{\partial x_t},\,h\frac{dx_t}{dt}\Big\rangle,
\end{align}
where $\tfrac{\partial B^A_{\theta}}{\partial R_t}$ and $\tfrac{\partial B^A_{\theta}}{\partial x_t}$ include the implicit dependence through the network prediction $(\hat{R}_0^\theta,\hat{x}_0^\theta)=f_\theta(s,t,R_t,x_t)$. In practice we compute $\tfrac{d}{dt}B^A_{\theta}$ (and $\tfrac{d}{dt}B^v_{\theta}$) with a JVP given the instantaneous velocities $(\tfrac{dR_t}{dt},\tfrac{dx_t}{dt})$, 
which avoids explicitly forming $\tfrac{d}{dt}A_{\theta}^{\mathrm{avg}}$.
We apply the same construction to translation with $B^v_{\theta}=t\,v_{\theta}^{\mathrm{avg}}$.

\paragraph{Small-$t$ consistency losses (implementation form).}
To match the implementation, we form the residuals in a scaled way and defer all divisions by $t$ to the final normalization step. Let $(t-s)$ denote the step size and define
\begin{align}
&A_{\theta}^{s\to t}:=(t-s)A_\theta= \frac{t-s}{t}B^A_{\theta},\nonumber\\
&v_{\theta}^{s\to t} :=(t-s)v_\theta=\frac{t-s}{t}B^v_{\theta},\nonumber 
\end{align}
where $B^A_{\theta}=tA^{\mathrm{avg}}_{\theta}$ and $B^v_{\theta}=t v^{\mathrm{avg}}_{\theta}$.
With the $(t-s)$-absorbed JVP outputs $(t-s)\,\tfrac{d}{dt}B^A_{\theta}$ and $(t-s)\,\tfrac{d}{dt}B^v_{\theta}$, we construct the scaled derivative targets
\begin{align}
\begin{cases}
 t\,\frac{d}{dt}A_{\theta}^{s\to t} &:= B^A_{\theta}+\mathrm{sg}\!\left((t-s)\,\frac{d}{dt}B^A_{\theta}-\frac{t-s}{t}B^A_{\theta}\right),\\
 t\,\frac{d}{dt}v_{\theta}^{s\to t} &:= B^v_{\theta}+\mathrm{sg}\!\left((t-s)\,\frac{d}{dt}B^v_{\theta}-\frac{t-s}{t}B^v_{\theta}\right).    
\end{cases}
\label{eq:scaled_targets}
\end{align}
where the JVP-derived directional derivatives $ (t-s)\,\tfrac{d}{dt}B^A_{\theta}$ and $(t-s)\,\tfrac{d}{dt}B^v_{\theta}$ are computed as a single Jacobian--vector product of $f_\theta(s,t,R_t,x_t)$:
\begin{align}
&(B^A_{\theta},B^v_{\theta},\ldots),\ \big((t-s)\,\tfrac{d}{dt}B^A_{\theta},\ (t-s)\,\tfrac{d}{dt}B^v_{\theta},\ldots\big)
\notag\\
&\qquad :=\ \mathrm{JVP}_{(\dot{R}_t,\dot{x}_t,\dot{t},\dot{s})}\big[f_\theta(s,t,R_t,x_t)\big].\label{eq:jvp_scaledB}
\end{align}
with the (step-size absorbed) input tangent
\begin{equation}
\left\{
\begin{aligned}
\dot{R}_t &= R_t\,\big((t-s)\,\Omega_t\big), \\
\dot{x}_t &= (t-s)\,v_t
\end{aligned}
\right.
\qquad
(\text{MF})
\end{equation}
or, for IMF where the instantaneous velocities are taken from a single shared forward at $(t,t)$,
\begin{equation}
\left\{
\begin{aligned}
\dot{R}_t &= R_t\,\big(\frac{t-s}{t}\,(B_{\theta}^{\mathrm{A}}(t,t,R_t,x_t))^\wedge\big), \\
\dot{x}_t &= \frac{t-s}{t}B^v_{\theta}(t,t,R_t,x_t)
\end{aligned}
\right.
\qquad
(\text{IMF}).
\end{equation}
The time tangents are $\dot{t}=(t-s)$ in the small-$t$ parameterization (otherwise $\dot{t}=1$), and $\dot{s}=0$.
The per-sample small-$t$ losses (summing over residues) are computed from $B^A_{\theta},B^v_{\theta}$ and the JVP-derived directional derivatives $(t-s)\,\tfrac{d}{dt}B^A_{\theta}$ and $(t-s)\,\tfrac{d}{dt}B^v_{\theta}$ as
{\footnotesize
\begin{align}
\boxed{\,
\begin{aligned}
\mathcal{L}_{\mathrm{rot}}
&:=\frac{1}{\max(t,t_{\varepsilon})^{2}}\sum_{i=1}^{N}
\left\{\begin{aligned}
&\left\|\mathrm{sg}\!\left(J\!\left(A_{\theta,i}^{s\to t}\right)\right)
\left(t\,\tfrac{d}{dt}A_{\theta,i}^{s\to t}\right)-t\,\omega_{t,i}\right\|_2^2,&\texttt{jmode}=J,\\
&\left\|t\,\tfrac{d}{dt}A_{\theta,i}^{s\to t}
-\mathrm{sg}\!\left(J^{-1}\!\left(A_{\theta,i}^{s\to t}\right)\right)t\,\omega_{t,i}\right\|_2^2,&\texttt{jmode}=J^{-1}.
\end{aligned}\right.\\
\mathcal{L}_{\mathrm{trans}}
&:=\frac{1}{\max(t,t_{\varepsilon})^{2}}\sum_{i=1}^{N}
\left\|t\,v_{t,i}-t\,\tfrac{d}{dt}v_{\theta,i}^{s\to t}\right\|_2^2.
\end{aligned}
\,}\label{eq:smallt_losses}
\end{align}}
This is algebraically equivalent to the unscaled MeanFlow objectives but avoids explicit $1/t$ factors during target construction; the only division by $t$ is the final $\max(t,t_{\varepsilon})^{-2}$ normalization applied after squaring the residuals.  The two \texttt{jmode} branches share the
same JVP \eqref{eq:jvp_scaledB} and scaled targets \eqref{eq:scaled_targets}, differing
only in the final residual.

\begin{algorithm}[t]
\caption{Stable X-prediction style Mean-flow training loss}
\label{alg:stable-derivative-B}
\begin{algorithmic}[1]
\Require State $(R_t,x_t)$; times $s<t$; instantaneous velocities $(\Omega_t,v_t)$ (MF)
or $(B^{A}_{\theta}(t,t),B^{v}_{\theta}(t,t))$ (IMF); network $f_\theta$; clamp
$t_\varepsilon$; Jacobian placement $\texttt{jmode}\in\{J,J^{-1}\}$
\State $\Delta t\gets t-s$
\State Construct the (step-size absorbed) input tangent:
\Statex \quad (MF) $\dot{R}_t\gets R_t(\Delta t\,\Omega_t)$, $\dot{x}_t\gets \Delta t\,v_t$
\Statex \quad (IMF) $\dot{R}_t\gets R_t(\frac{\Delta t}{t}\,(B^{A}_{\theta}(t,t))^\wedge)$, $\dot{x}_t\gets \frac{\Delta t}{t}B^{v}_{\theta}(t,t)$ See (Eq. \eqref{eq:Bt}).
\State Set $\dot{t}\gets \Delta t$ in the small-$t$ parameterization (otherwise $\dot{t}\gets 1$), and $\dot{s}\gets 0$.
\State Compute a single JVP of $f_\theta(s,t,R_t,x_t)$ along $(\dot{R}_t,\dot{x}_t,\dot{t},\dot{s})$ to obtain
\Statex \quad primals $(B^A_{\theta},B^v_{\theta})$ and directional derivatives $\big(\Delta t\,\tfrac{d}{dt}B^A_{\theta},\Delta t\,\tfrac{d}{dt}B^v_{\theta}\big)$  (Eq. \eqref{eq:jvp_scaledB})
\State Form the scaled derivative targets via Eq.~\eqref{eq:scaled_targets}.
\State $\mathcal{J}\gets\mathrm{sg}\big(J(A^{s\to t}_\theta)\big)$
\If{$\texttt{jmode}=J$}
  \State $\mathrm{res}_{\mathrm{rot}}\gets\mathcal{J}\big(t\,\tfrac{d}{dt}A^{s\to t}_{\theta}\big)-t\,\omega_t$
\Else
  \State $\mathrm{res}_{\mathrm{rot}}\gets t\,\tfrac{d}{dt}A^{s\to t}_{\theta}-\mathcal{J}^{-1}\big(t\,\omega_t\big)$
        \Comment{Remark~\ref{rmk:jacobian-inverse}}
\EndIf
\State $\mathrm{res}_{\mathrm{trans}}\gets t\,\tfrac{d}{dt}v^{s\to t}_{\theta}-t\,v_t$
       \Comment{$J\equiv I$ on the abelian branch}
       
\State Normalize at the end by $\max(t,t_\varepsilon)^{-2}$ (Eq.~\eqref{eq:smallt_losses})
\end{algorithmic}
\end{algorithm}
\begin{algorithm}[t]
\caption{Stable inference with $t_{\min}>0$ and endpoint prediction
(linear and exponential rotation schedules)}
\label{alg:stable-inference-tmin}
\begin{algorithmic}[1]
\Require Prior sample $(R_1^N,x_1^N)$; network $f_\theta$; steps $T$; $t_{\min}>0$;
         rotation schedule $\in\{\texttt{linear},\texttt{exp}\}$, exp rate $c>0$
\Ensure $(\hat{R}_0^N,\hat{x}_0^N)$
\State $\{t_i\}_{i=0}^{T-1}\gets \mathrm{linspace}(1,t_{\min},T)$
\State $(R,x)\gets (R_1^N,x_1^N)$
\For{$i=0,\ldots,T-2$}
  \State $t\gets t_i,\ \ s\gets t_{i+1},\ \ \Delta t\gets t-s$
  \State $(\hat{R}_0^\theta,\hat{x}_0^\theta)\gets f_\theta(s,t,R,x)$
  \State $u_\theta\gets \log\big((\hat{R}_0^\theta)^\top R\big)^{\vee}$
         \Comment{log-map from $R$ toward endpoint $\hat{R}_0^\theta$}
  \If{schedule $=\texttt{linear}$}
     \State $A^{\mathrm{avg}}_\theta\gets \tfrac{1}{t}\,u_\theta$
            \Comment{rate $1/t$ (remaining time)}
  \Else \Comment{\texttt{exp}}
     \State $A^{\mathrm{avg}}_\theta\gets c\,u_\theta$
            \Comment{constant rate $c$}
  \EndIf
  \State $v^{\mathrm{avg}}_\theta\gets \tfrac{1}{t}(x-\hat{x}_0^\theta)$
         \Comment{translation: linear in both schedules}
  \State $R\gets R\exp\!\big(-\Delta t\,(A^{\mathrm{avg}}_\theta)^\wedge\big)$
  \State $x\gets x-\Delta t\,v^{\mathrm{avg}}_\theta$
\EndFor
\State $(\hat{R}_0^\theta,\hat{x}_0^\theta)\gets f_\theta(0,t_{\min},R,x)$
\State $(\hat{R}_0^N,\hat{x}_0^N)\gets (\hat{R}_0^\theta,\hat{x}_0^\theta)$
\end{algorithmic}
\end{algorithm}

\subsection{Rescaled MeanFlow: training and inference pseudocode}
The rescaled MeanFlow training objective and the corresponding stable inference procedure are summarized in
Algorithm~\ref{alg:stable-derivative-B} and Algorithm~\ref{alg:stable-inference-tmin}, respectively.
In inference, in addition to the standard linear rotation schedule, we also use an exponential (\texttt{exp})
rotation schedule inspired by ReQFlow~\cite{yue2025reqflow}. The main idea is to redefine the effective
rotation step sizes so that steps are larger on the noise side ($t\approx 1$) and become smaller and more
fine-grained near the data side ($t\approx 0$), which empirically improves few-step generation stability.

\section{Stable Training Framework, Part 2: JVP-Free Training via SE(3) $\alpha$-Flow}\label{sec:alpha-flow}

The stabilization in Section~\ref{sec:stable-rescale} mitigates, but does not
remove, the core difficulty of the \emph{differential} consistency target: it still
requires the trajectory derivative $\tfrac{d}{dt}A^{\mathrm{avg}}_\theta$ via a
Jacobian--vector product (JVP) of $f_\theta$. Since
$A^{\mathrm{avg}}_\theta=\tfrac1t\log((\hat R_0^\theta)^{\top}R_t)^\vee$ already carries
a $1/t$ factor, differentiating it injects a second $1/t$, so a head error
$\varepsilon$ on the rotation output propagates to the target at order
$\varepsilon/t^2$. This is benign for the flat translation branch but dominates the
instability of the rotation branch, where it compounds with $\mathrm{SO}(3)$ curvature
and fragile forward-mode autodiff through the IPA trunk. We adopt the $\alpha$-Flow
framework of \citet{zhang2025alphaflow}, which replaces the differential target by a
two-evaluation, JVP-free consistency target, reducing the amplification to
$\mathcal{O}(\varepsilon/t)$.

\subsection{Model parameterization: endpoint and average-velocity heads}
\label{sec:af-param}
The network $f_\theta$, evaluated at input $(s,t,R_t,x_t)$, may be read either as an \emph{endpoint}
($x$-)predictor returning $(\hat R_0^\theta,\hat x_0^\theta)$, or as an
\emph{average-velocity} ($u$-)predictor returning $(A^{\mathrm{avg}}_\theta,
v^{\mathrm{avg}}_\theta)$ --- the mean body angular velocity and mean translational
velocity of the predicted geodesic from the endpoint to $(R_t,x_t)$. The two views are equivalent and can be transferred to each other:
\begin{align}
&\text{(rotation)}\quad
A^{\mathrm{avg}}_\theta=\tfrac1t\log\!\big((\hat R_0^\theta)^\top R_t\big)^\vee \Longleftrightarrow\;\;
\hat R_0^\theta=R_t\exp\!\big(-t\,A^{\mathrm{avg}\wedge}_\theta\big),
\label{eq:af-convert-rot}\\[2pt]
&\text{(translation)}\quad
v^{\mathrm{avg}}_\theta=\tfrac1t\big(x_t-\hat x_0^\theta\big) \Longleftrightarrow\;\;
\hat x_0^\theta=x_t-t\,v^{\mathrm{avg}}_\theta.
\label{eq:af-convert-trans}
\end{align}
Both maps depend on $(s,t,R_t,x_t)$; we suppress the arguments where clear, and write
$A^{\mathrm{avg}}_\theta(s,t,R_t,x_t)$, $\hat R_0^\theta(s,t,R_t,x_t)$ etc.\ when
needed. We develop the $\alpha$-Flow targets and losses in the average-velocity
variables $(A^{\mathrm{avg}}_\theta,v^{\mathrm{avg}}_\theta)$, where the construction is
most transparent (the \emph{regular-$t$} form), and switch to the endpoint variables
only to expose and cancel the small-denominator factors, giving the numerically
stabilized \emph{small-$t$} form. Throughout, the data-side instantaneous quantities
are the body angular velocity $\omega_t=\Omega(t,R_t)^\vee$ and the translational velocity
$v_t=v(t,x_t)=x_1-x_0$.

\subsection{Time grid}
Let $\alpha\in[\alpha_{\min},1]$ be the consistency-step ratio, floored by a small
$\alpha_{\min}>0$. With $0\le s\le t\le1$, define the intermediate time and step
\begin{align}
&m=\alpha s+(1-\alpha)t,\qquad \delta=t-m=\alpha(t-s),\nonumber\\
&m-s=(1-\alpha)(t-s),\qquad s\le m\le t,\nonumber
\end{align}
so that $(m-s)+\delta=t-s$. The stepped-back (intermediate) state, obtained by
integrating the shift velocities backward over $[m,t]$, is
\begin{align}
x_m=x_t-\delta\,\tilde v_t,\qquad R_m=R_t\exp(-\delta\,\tilde\omega_t^{\wedge}),\nonumber
\end{align}
with the shift velocities $\tilde v_t,\tilde\omega_t$ fixed below.

\subsection{Translation target}\label{sec:af-trans}
Following \citet{zhang2025alphaflow}, the average velocity over $[s,t]$ decomposes across
the split point $m$ into a \emph{near} segment $[m,t]$ and a \emph{far} segment
$[s,m]$. The near segment uses the shift velocity
$$
\tilde{v}_t=
\begin{cases}
v_t, & \text{(data velocity)},\\[2pt]
u^{x}_\theta(m,t,x_t)=v^{\mathrm{avg}}_\theta(m,t,x_t), & \text{(model prediction)},
\end{cases}
$$
and we adopt the data velocity $\tilde v_t=v_t$. The far segment uses the
stop-gradient model average velocity at the stepped-back state $x_m=x_t-\delta\tilde v_t$,
\begin{align}
v^{\mathrm{avg}}_\theta(s,m,x_m)=\tfrac1m\big(x_m-\hat x_0^\theta(s,m,R_m,x_m)\big),\nonumber
\end{align}
and the $\alpha$-Flow target is the convex combination
\begin{align}
v^{\mathrm{avg}}_{\mathrm{tgt}}
=\alpha\,\tilde{v}_t+(1-\alpha)\,v^{\mathrm{avg}}_\theta(s,m,x_m).
\label{eq:af-trans-target}
\end{align}

\paragraph{Regular-$t$ loss.} Regressing the model average velocity directly,
\begin{align}
\mathcal{L}^{\mathrm{reg}}_{\mathrm{trans}}
=\frac{1}{\alpha}\sum_{i=1}^{N}
\big\|v^{\mathrm{avg}}_{\theta,i}(s,t,x_t)-\mathrm{sg}\big(v^{\mathrm{avg}}_{\mathrm{tgt},i}\big)\big\|_2^2 .
\label{eq:af-trans-loss-reg}
\end{align}
The $1/\alpha$ prefactor matches the $\alpha\to0$ scaling of the residual
(Prop.~\ref{prop:af-meanflow-limit}); no further $s,t$-dependent weight is needed
because the abelian target is a plain convex combination.

\paragraph{Small-$t$ (stabilized) form.} The average-velocity variables carry an
explicit $1/t$ (via $v^{\mathrm{avg}}=\tfrac1t(x_t-\hat x_0)$) and $1/m$ (in the far
term), which amplify head error as $t,m\to0$. Deferring the $1/t$ into the displacement
variable $B^{x}_\theta=t\,v^{\mathrm{avg}}_\theta=x_t-\hat x_0^\theta$ removes the
model-side division, and the corresponding displacement target is
\begin{align}
B^{x}_{\mathrm{tgt}}=t\,v^{\mathrm{avg}}_{\mathrm{tgt}}
=\alpha t\,\tilde{v}_t+\frac{(1-\alpha)t}{m}\big(x_m-\hat x_0^\theta(s,m,R_m,x_m)\big),\nonumber
\end{align}
where $0<\tfrac{(1-\alpha)t}{m}\le1$ is numerically bounded. Regressing on
displacements with a floored normalizer gives
\begin{align}
\mathcal{L}_{\mathrm{trans}}=\frac{1}{\alpha\,\max(t,t_\varepsilon)^{2}}
\sum_{i=1}^{N}\big\|B^{x}_{\theta,i}-\mathrm{sg}(B^{x}_{\mathrm{tgt},i})\big\|_2^2 .
\label{eq:af-trans-loss}
\end{align}
Since $\|v^{\mathrm{avg}}_\theta-v^{\mathrm{avg}}_{\mathrm{tgt}}\|_2^2
=t^{-2}\|B^{x}_\theta-B^{x}_{\mathrm{tgt}}\|_2^2$, \eqref{eq:af-trans-loss} coincides
with \eqref{eq:af-trans-loss-reg} for $t\ge t_\varepsilon$; the floor only regularizes the vanishing-$t$ limit.

\subsection{Rotation target}\label{sec:af-rot}
We mirror the translation branch on $\mathrm{SO}(3)$, replacing vector addition by
group composition. For $0\le a\le b\le1$, the accumulated relative rotation is
\begin{align}
D(a,b)&:=\mathcal{T}\exp\!\Big(\int_a^b\Omega(\tau)\,d\tau\Big)\nonumber\\
&=\exp\!\big((b-a)\,A^{\mathrm{avg}}(a,b)^{\wedge}\big)\nonumber\\
&=R_a^\top R_b\in\mathrm{SO}(3),
\label{eq:af-Dab}
\end{align}
where the middle equality collapses the time-ordered exponential to a single generator
and is exact under the geodesic (constant-body-velocity) assumption, which holds for
the data path and for the $x$-prediction geodesic to $\hat R_0^\theta$.
\begin{proposition}[Interval additivity]\label{prop:af-additivity}
For any $s\le m\le t$, $\;D(s,t)=D(s,m)\,D(m,t)$.
\end{proposition}
\begin{proof}
$D(s,t)=R_s^\top R_t=(R_s^\top R_m)(R_m^\top R_t)=D(s,m)D(m,t)$. The regrouping is exact
(independent of the geodesic assumption) and is the non-commutative $\mathrm{SO}(3)$
analogue of Euclidean displacement additivity: addition becomes group multiplication,
ordered far ($[s,m]$) before near ($[m,t]$).
\end{proof}

Mirroring \eqref{eq:af-trans-target}, the near segment $[m,t]$ uses the shift angular
velocity
$$
\tilde\omega_t=
\begin{cases}
\omega_t=\Omega(t,R_t)^\vee, & \text{(data angular velocity)},\\[2pt]
A^{\mathrm{avg}}_\theta(m,t,R_t,x_t), & \text{(model prediction; Shortcut)},
\end{cases}
$$
and we adopt the data angular velocity $\tilde\omega_t=\omega_t$. The far segment $[s,m]$ uses
the stop-gradient model average angular velocity at the stepped-back state
$(R_m,x_m)$, $R_m=R_t\exp(-\delta\,\tilde\omega_t^{\wedge})$:
\begin{align}
A_m&:=A^{\mathrm{avg}}_\theta(s,m,R_m,x_m)\nonumber\\
&=\tfrac1m\log\!\big((\hat R_0^\theta(s,m,R_m,x_m))^\top R_m\big)^\vee .
\nonumber
\end{align}
Then $D(s,m)=\exp\!\big((m-s)A_m^{\wedge}\big)$ and
$D(m,t)=\exp\!\big(\delta\,\tilde\omega_t^{\wedge}\big)=R_m^\top R_t$, with $x_m=x_t-\delta\tilde v_t$
as in Section~\ref{sec:af-trans}.

\paragraph{Regular-$t$ target and loss.} By Proposition~\ref{prop:af-additivity} and
\eqref{eq:af-Dab}, $\log D(s,t)^\vee=(t-s)\,A^{\mathrm{avg}}_{\mathrm{tgt}}$, so the
average-velocity target is the group-composition analogue of
\eqref{eq:af-trans-target}:
\begin{align}
A^{\mathrm{avg}}_{\mathrm{tgt}}
=\frac{1}{t-s}\,\log\!\Big(
\underbrace{\exp\!\big((m-s)A_m^{\wedge}\big)}_{D(s,m)}\;
\underbrace{\exp\!\big(\delta\,\tilde\omega_t^{\wedge}\big)}_{D(m,t)}\Big)^{\!\vee}.
\label{eq:af-rot-target-reg}
\end{align}
The loss mirrors \eqref{eq:af-trans-loss-reg}:
\begin{align}
\mathcal{L}^{\mathrm{reg}}_{\mathrm{rot}}
=\frac{1}{\alpha}\sum_{i=1}^{N}
\big\|A^{\mathrm{avg}}_{\theta,i}(s,t)-\mathrm{sg}\big(A^{\mathrm{avg}}_{\mathrm{tgt},i}\big)\big\|_2^2 .
\label{eq:af-rot-loss-reg}
\end{align}
The scalar $\tfrac{1}{t-s}$ must remain \emph{outside} $\log(\exp\cdot\exp)$: because
$A_m$ and $\tilde\omega_t$ do not commute, folding it into the two exponentials would
rescale each segment angle and alter the Baker--Campbell--Hausdorff (BCH) cross term,
no longer yielding $\log D(s,t)$. This is the non-commutative counterpart of the
translation branch, where the same normalization is instead absorbed into the linear
convex weights $\alpha,\,1-\alpha$.

\paragraph{Small-$t$ (stabilized) form.} Two small-denominator factors appear: the
overall $1/t$ shared with translation, and the rotation-specific $1/(t-s)$ together
with $\log(\cdot)$ near $I$. Deferring the $1/t$ into the displacement
$B^{A}_\theta=t\,A^{\mathrm{avg}}_\theta=\log((\hat R_0^\theta)^\top R_t)^\vee$ gives the
displacement target $B^{A}_{\mathrm{tgt}}=t\,A^{\mathrm{avg}}_{\mathrm{tgt}}$,
\begin{align}
\boxed{\;
\begin{aligned}
B^{A}_{\mathrm{tgt}}
&=\frac{t}{t-s}\,
\log\!\Big(
\underbrace{\exp\!\big((\tfrac{m-s}{m}B^A_m)^{\wedge}\big)}_{D(s,m)}\;\underbrace{\exp\!\big(\delta\,\tilde\omega_t^{\wedge}\big)}_{D(m,t)}
\Big)^{\!\vee},\\
B^{A}_m
&=\log\!\big((\hat R_0^\theta(s,m,R_m,x_m))^\top R_m\big)^\vee.
\end{aligned}
\;}
\label{eq:af-rot-target}
\end{align}
using $(m-s)A_m=\tfrac{m-s}{m}B^A_m$. Both segment angles scale as $O(t-s)$ and are
bounded: $\tfrac{m-s}{m}B^A_m$ has $\tfrac{m-s}{m}\le1$ and main log $\le\pi$, while
$\delta\,\tilde\omega_t=\alpha(t-s)\,\omega_t$ with $\|\omega_t\|\le\pi$. Hence
$\|\log D(s,t)^\vee\|=O(t-s)$, the $\tfrac{1}{t-s}$ factor is cancelled at the same
rate, and
\begin{align}
\big\|B^{A}_{\mathrm{tgt}}\big\|
\le\pi\Big(\tfrac{(1-\alpha)t}{m}+\alpha t\Big)\le 2\pi, \nonumber
\end{align}
mirroring the bounded weights $\tfrac{(1-\alpha)t}{m},\,\alpha t\le1$ of the translation
target: the target is analytically free of small-denominator blow-up. For numerical
stability when $t-s$ is below machine tolerance $t_\varepsilon$ (where $\log(\cdot)$
near $I$ loses precision), we avoid forming $\tfrac{1}{t-s}$ and use the first-order
BCH limit, whose $O(t-s)$ correction is then negligible:
\begin{align}
B^{A}_{\mathrm{tgt}}=
\begin{cases}
\dfrac{t}{t-s}\,\log\!\big(D(s,m)\,D(m,t)\big)^\vee, & t-s\ge t_\varepsilon,\\[8pt]
t\big[(1-\alpha)A_m+\alpha\,\tilde\omega_t\big], & t-s< t_\varepsilon.
\end{cases}\nonumber
\end{align}
The rotation loss is then identical in form to \eqref{eq:af-trans-loss},
\begin{align}
\mathcal{L}_{\mathrm{rot}}=\frac{1}{\alpha\,\max(t,t_\varepsilon)^{2}}
\sum_{i=1}^{N}\big\|B^{A}_{\theta,i}-\mathrm{sg}(B^{A}_{\mathrm{tgt},i})\big\|_2^2 ,
\label{eq:af-rot-loss}
\end{align}
and, via $\|A^{\mathrm{avg}}_\theta-A^{\mathrm{avg}}_{\mathrm{tgt}}\|_2^2
=t^{-2}\|B^{A}_\theta-B^{A}_{\mathrm{tgt}}\|_2^2$, coincides with
\eqref{eq:af-rot-loss-reg} whenever $t\ge t_\varepsilon$ and $t-s\ge t_\varepsilon$.
The target requires one logarithm and two exponentials and no differentiation of
$f_\theta$.

\begin{remark}[Fallback at $\alpha=1$ and choice of normalization]\label{rmk:af-limits}
At $\alpha=1$ the intermediate time hits the far end, $m=s$ and $\delta=t-s$, so the
far segment vanishes ($D(s,m)=I$) and
$B^{A}_{\mathrm{tgt}}=\tfrac{t}{t-s}\log D(m,t)^\vee=t\,\tilde\omega_t$: the objective
reduces to flow matching \eqref{eq:fw-loss}, anchoring the branch to the data
velocity and supplying the boundary condition that prevents collapse. The opposite
limit $\alpha\to0$ recovers the differential MeanFlow objective and is analyzed in
Section~\ref{sec:af-meanflow-limit}.

The group composition $\log(\exp\cdot\exp)$ is essential throughout: the naive Lie
sum $(m-s)A_m+\delta\,\tilde\omega_t$ drops the BCH/Jacobian correction and is not
equivalent.
\end{remark}

\subsection{The $\alpha\to 0$ Limit Recovers MeanFlow}
\label{sec:af-meanflow-limit}

In this subsection we show that the $\alpha$-Flow loss converges to the differential
MeanFlow loss as $\alpha\to0$. We treat the rotation branch; the translation branch
is the abelian case ($J\equiv I$, no BCH correction) and reduces to Euclidean
MeanFlow \eqref{eq:mf_loss} verbatim.
\begin{remark}[Alphaflow normalization]\label{rmk:af-normalization}
If $\alpha$-Flow is to interpolate exactly between flow matching and MeanFlow, the
prefactor should be $1/\alpha^2$ rather than $1/\alpha$, since the residual
regressed in \eqref{eq:af-rot-loss-reg} is $O(\alpha)$
(Proposition~\ref{prop:af-meanflow-limit}). We therefore analyze
\begin{align}
\widetilde{\mathcal{L}}_{\mathrm{rot}}
:=\frac{1}{\alpha^{2}}\sum_{i=1}^{N}
\big\|A^{\mathrm{avg}}_{\theta,i}(s,t,R_t,x_t)
-\mathrm{sg}\big(A^{\mathrm{avg}}_{\mathrm{tgt},i}\big)\big\|_2^2 ,
\label{eq:af-rot-loss-theory}
\end{align}
with $A^{\mathrm{avg}}_{\mathrm{tgt}}$ from \eqref{eq:af-rot-target-reg}; the
$1/\alpha$ prefactor of \eqref{eq:af-rot-loss-reg} follows the style of
\citet{zhang2025alphaflow} and is what we train with. We also use only the first case
of \eqref{eq:af-rot-target}, the second being a small-$(t-s)$ fallback introduced
for numerical convenience.
\end{remark}

\begin{proposition}[$\alpha\to0$ recovers the MeanFlow loss]
\label{prop:af-meanflow-limit}
Let $g(\tau):=(\tau-s)A^{\mathrm{avg}}_\theta(s,\tau,R_\tau,x_\tau)$ denote the model
log-displacement along the interpolation path, so that
$g(t)=(t-s)A^{\mathrm{avg}}_\theta(s,t)$ and, by the product rule,
$\dot g(t)=A^{\mathrm{avg}}_\theta+(t-s)\tfrac{d}{dt}A^{\mathrm{avg}}_\theta$ is the
bracketed quantity of \eqref{eq:omgea_identity_main}. If $g$ is $C^2$ on $t$, then
the target \eqref{eq:af-rot-target-reg} expands as
\begin{align}
A^{\mathrm{avg}}_{\mathrm{tgt}}
=A^{\mathrm{avg}}_\theta(s,t)+\alpha\Big(J\big(g(t)\big)^{-1}\omega_t-\dot g(t)\Big)+O(\alpha^2)
\label{eq:af-tgt-expand}
\end{align}
and consequently, per residue,
\begin{align}
\widetilde{\mathcal{L}}_{\mathrm{rot}}=\big\|\dot g(t)-J^{-1}\omega_t\big\|_2^{2}+O(\alpha),
\label{eq:af-loss-limit}
\end{align}
i.e.\ exactly the residual of the $J^{-1}$ form \eqref{eq:mf-se3-loss2} of the MeanFlow identity.
\end{proposition}

\begin{proof}
Since \eqref{eq:af-rot-loss-theory} decouples across residues, we take $N=1$ and
drop the index $i$. Write $m=t-\delta$ with $\delta=\alpha(t-s)$, and $J:=J(g(t))$.

Since the data velocities $\omega_t$ and $v_t$ are constant along the interpolation
path (Appendix~\ref{sec:bg-fm}), the stepped-back state satisfies
$R_t\exp(-\delta\,\omega_t^{\wedge})=R_m$ and $x_t-\delta v_t=x_m$ \emph{exactly}. The
two factors of \eqref{eq:af-rot-target-reg} are therefore
$D(s,m)=\exp(g(m)^{\wedge})$ and $D(m,t)=\exp(\delta\,\omega_t^{\wedge})$, with $g$ the
same function appearing in the statement; this is what lets the finite difference
below capture the \emph{total} derivative \eqref{eq:dA/dt} rather than $\partial_t$
alone.

The first-order BCH formula gives
$\log(e^{X}e^{Y})=X+\frac{\mathrm{ad}_X}{1-e^{-\mathrm{ad}_X}}Y+O(\|Y\|^2)$. On
$\mathfrak{so}(3)$ one has $\mathrm{ad}_{\phi^{\wedge}}\psi^{\wedge}
=(\phi^{\wedge}\psi)^{\wedge}$, so in vector coordinates
$\big(\tfrac{1-e^{-\mathrm{ad}_X}}{\mathrm{ad}_X}\big)^{\vee}
=\int_0^1e^{-u\phi^{\wedge}}du=J(\phi)$ by \eqref{eq:right_jacobian}. This is
invertible: $J(\phi)$ is a polynomial in the skew matrix $\phi^{\wedge}$, so its
eigenvalues are $1$ and $\tfrac{1-e^{\mp i\theta}}{\pm i\theta}$ with
$\theta=\|\phi\|\le\pi$, all nonzero. Taking $X=g(m)^{\wedge}$ and
$Y=\delta\,\omega_t^{\wedge}=O(\alpha)$,
\begin{align}
\log\!\big(D(s,m)D(m,t)\big)^{\vee}
=g(m)+\delta\,J\big(g(m)\big)^{-1}\omega_t+O(\delta^2).
\nonumber
\end{align}
Substituting $g(m)=g(t)-\delta\dot g(t)+O(\delta^2)$ and
$J(g(m))^{-1}=J^{-1}+O(\delta)$, we have: 
\begin{align}
A^{\mathrm{avg}}_{\mathrm{tgt}}
&=\frac{1}{t-s}\log\!\big(D(s,m)\,D(m,t)\big)^{\vee}
\nonumber\\
&=\frac{1}{t-s}\Big(g(m)+\delta\,J\big(g(m)\big)^{-1}\omega_t+O(\delta^{2})\Big)
\nonumber\\
&=\frac{1}{t-s}\Big(g(t)-\delta\,\dot g(t)
+\delta\,J^{-1}\omega_t+O(\delta^{2})\Big)
\nonumber\\
&=\frac{g(t)}{t-s}
+\frac{\alpha(t-s)}{t-s}\Big(J^{-1}\omega_t-\dot g(t)\Big)
+\frac{O\big(\alpha^{2}(t-s)^{2}\big)}{t-s}
\nonumber\\
&=A^{\mathrm{avg}}_\theta(s,t)
-\alpha\Big(\dot g(t)-J^{-1}\omega_t\Big)+O(\alpha^{2}),
\nonumber
\end{align}
which is \eqref{eq:af-tgt-expand}. Substituting into \eqref{eq:af-rot-loss-theory},
\begin{align}
\widetilde{\mathcal{L}}_{\mathrm{rot}}
&=\frac{1}{\alpha^{2}}
\Big\|A^{\mathrm{avg}}_\theta(s,t)
-\mathrm{sg}\big(A^{\mathrm{avg}}_{\mathrm{tgt}}\big)\Big\|_2^{2}
\nonumber\\
&=\frac{1}{\alpha^{2}}
\Big\|A^{\mathrm{avg}}_\theta(s,t)
-\Big(A^{\mathrm{avg}}_\theta(s,t)-\alpha\,\big(\dot g(t)-J^{-1}\omega_t\big)+O(\alpha^{2})\Big)\Big\|_2^{2}
\nonumber\\
&=\frac{1}{\alpha^{2}}
\Big\|\alpha\,\big(\dot g(t)-J^{-1}\omega_t\big)+O(\alpha^{2})\Big\|_2^{2}
\nonumber\\
&=\big\|\big(\dot g(t)-J^{-1}\omega_t\big)\big\|_2^{2}+O(\alpha),
\nonumber
\end{align}
which is \eqref{eq:af-loss-limit}.
\end{proof}
\begin{remark}[Relation to the two loss forms]
\label{rmk:af-reweight}
The limit \eqref{eq:af-loss-limit} is the $J^{-1}$ form \eqref{eq:mf-se3-loss2},
which is exactly the objective used in Stage~2
(Appendix~\ref{sec:training-details}): the $\alpha$-Flow warm-up and the MeanFlow
phase therefore optimize the same target up to $O(\alpha)$, rather than switching
objectives mid-training. The $J$ form \eqref{eq:mf-se3-loss} shares its zero set
but differs by the reweighting $J^{\top}J$, which is the identity on the
translation branch and on the diagonal $s=t$.
\end{remark}

\begin{figure}[t]
  \centering
  \includegraphics[width=0.62\linewidth]{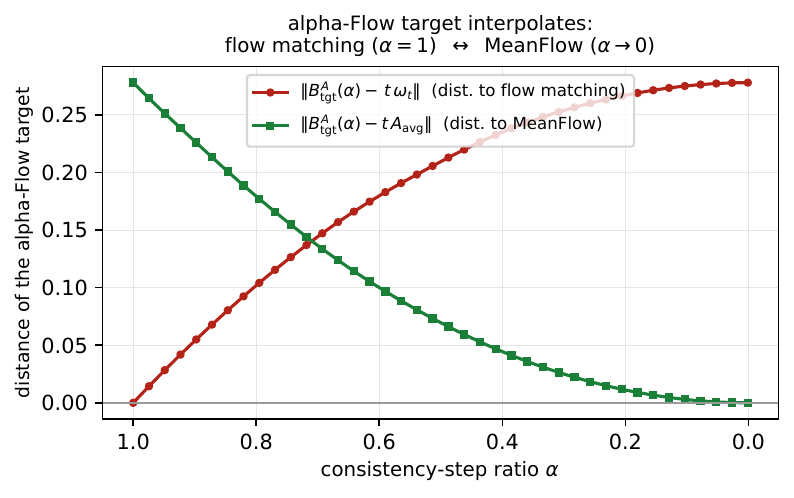}
  \caption{The $\alpha$-Flow target interpolates between flow matching and MeanFlow.
  On a smooth generator field $g(\tau)$ we form the rotation target
  $B^{A}_{\mathrm{tgt}}(\alpha)=tA^{\text{avg}}_{\mathrm{tgt}}$ of \eqref{eq:af-rot-target} over a range of $\alpha$
  and plot its distance to the two endpoints: to the flow-matching target
  $t\,\omega_t$ (red), which vanishes at $\alpha=1$ (Remark~\ref{rmk:af-limits}), and
  to the MeanFlow target $t\,A^{\mathrm{avg}}$ (green), which vanishes as
  $\alpha\to0$ (Proposition~\ref{prop:af-meanflow-limit}). The plot is drawn in the
  displacement variable $B^{A}=t\,A^{\mathrm{avg}}$ of Section~\ref{sec:af-rot}; the
  common factor $t$ affects both curves equally.}
  \label{fig:af-limit-numerical}
\end{figure}

We verify Proposition~\ref{prop:af-meanflow-limit} numerically. We take a smooth
generator field $g(\tau)$ and a data velocity $\omega_t$ drawn \emph{independently} of
$g$, so that $J(g(t))^{-1}\omega_t\neq\dot g(t)$ in general and the first-order term of
\eqref{eq:af-tgt-expand} is non-degenerate. For a range of $\alpha$ we form the
$\alpha$-Flow rotation target
$A^{\mathrm{avg}}_{\mathrm{tgt}}(\alpha)=\tfrac{1}{t-s}\log\!\big(
e^{g(m)^{\wedge}}e^{\delta\,\omega_t^{\wedge}}\big)^{\vee}$ with
$m=\alpha s+(1-\alpha)t$ and $\delta=\alpha(t-s)$, and compare it to the zeroth- and
first-order predictions of \eqref{eq:af-tgt-expand}, computing $\dot g$ by
forward-mode AD and $J(g(t))$ from its closed form. The
$\mathcal{O}(\alpha)$/$\mathcal{O}(\alpha^2)$ scalings
(Figure~\ref{fig:af-limit-numerical}) confirm the expansion and its coefficient;
consistently, $\widetilde{\mathcal{L}}_{\mathrm{rot}}$ converges to
$\lVert\dot g-J(g(t))^{-1}\omega_t\rVert_2^2$.

\begin{algorithm}[t]
\caption{JVP-free SE(3) $\alpha$-Flow training step}
\label{alg:af-training}
\begin{algorithmic}[1]
\Require $(R_t^N,x_t^N)$; times $s\le t$; data $(\Omega_t^N,v_t^N)$; head $f_\theta$;
         $\alpha\in[\alpha_{\min},1]$; clamp $t_\varepsilon$
\State $m\gets\alpha s+(1-\alpha)t$,\quad $\delta\gets t-m$,\quad $\omega_t:=\omega_t^N\gets(\Omega_t^N)^\vee$
\State $R_m^N\gets R_t^N\exp(-\delta\omega_t^{\wedge})$,\quad $x_m^N\gets x_t^N-\delta\,v_t^N$
\State $(\hat R_0^m,\hat x_0^m)\gets \mathrm{sg}\,f_\theta(s,m,R_m^N,x_m^N)$
\State $mA_m\gets\log((\hat R_0^m)^\top R_m^N)^\vee$,\quad $mu_m\gets x_m^N-\hat x_0^m$ \Comment{displacements}
\State $B^A_{\mathrm{tgt}}\gets\frac{t}{t-s}\log\big(\exp((\tfrac{m-s}{m}\,mA_m)^{\wedge})\exp(\delta\,\omega_t^{\wedge})\big)^\vee$
\State $B^x_{\mathrm{tgt}}\gets \alpha t\,v_t^N+\tfrac{(1-\alpha)t}{m}\,mu_m$
\State $(\hat R_0^\theta,\hat x_0^\theta)\gets f_\theta(s,t,R_t^N,x_t^N)$ \Comment{only graph path}
\State $B^A_\theta\gets\log((\hat R_0^\theta)^\top R_t^N)^\vee$,\quad $B^x_\theta\gets x_t^N-\hat x_0^\theta$
\State Form $\mathcal{L}_{\mathrm{rot}},\mathcal{L}_{\mathrm{trans}}$ (Eqs.~\eqref{eq:af-rot-loss},\eqref{eq:af-trans-loss});
\end{algorithmic}
\end{algorithm}


\section{Stable Training Framework, Part 3: Semigroup Loss of SE(3)-MeanFlow}
\label{sec:semigroup}

The MeanFlow identity \eqref{eq:omgea_identity_main} is the \emph{differential}
form of our average velocity: it is obtained by differentiating
\eqref{eq:avg_omega} in $t$, and its training target carries both a
Jacobian--vector product (for $\tfrac{d}{dt}A^{\mathrm{avg}}_\theta$) and the
right Jacobian $J$. In this section, we show that the \emph{same} definition
\eqref{eq:avg_omega} also admits a \emph{finite}, derivative-free characterization:
the average-velocity flow map is a genuine two-parameter semigroup on
$\SE$, and the consistency with this semigroup yields a JVP-free,
$J$-free training objective. Throughout, let $s\le m\le t\in[0,1]$. We write
\begin{align}
&\omega_t:=\Omega(t,R_t,x_t)^\vee=\log(R_0^\top R_1)^\vee,\nonumber\\
&v_t:=v(t,x_t)=x_1-x_0\nonumber
\end{align}
for the data-side instantaneous (constant body-frame) velocities used in flow matching
\eqref{eq:fw-loss}.

\subsection{The average-velocity flow map}

Recall from \eqref{eq:avg_omega} that the average angular velocity over $[s,t]$ is
defined by
\begin{align}
\exp\!\big((t-s)\,\Omega^{\mathrm{avg}}(s,t,R_t,x_t)\big)
&=\mathcal{T}\exp\!\Big(\int_s^t \Omega(\tau,R_\tau,x_\tau)\,d\tau\Big)=:D(s,t)\in\SO,
\label{eq:sg-D}
\end{align}
where the last equality $D(s,t)=R_s^\top R_t$ holds by integrating the
left-trivialized ODE $\dot R_\tau=R_\tau\Omega(\tau,R_\tau)$; this is exactly
\eqref{eq:af-Dab}. We write $A^{\mathrm{avg}}=(\Omega^{\mathrm{avg}})^\vee$, and define the
rotation log-displacement by
\begin{align}
&A^{s\to t}
:=(t-s)A^{\mathrm{avg}}(s,t,R_t,x_t)\in\mathbb{R}^{3},\nonumber\\
&\exp\!\big((A^{s\to t})^{\wedge}\big)
=D(s,t).\nonumber
\end{align}
For the translation branch on $\mathbb{R}^3$, the decoupled product metric \eqref{eq:dec_met} gives the
displacement $p(s,t):=x_t-x_s=(t-s)\,v^{\mathrm{avg}}(s,t,x_t)$. Based on the two notations of displacement above, we give the following definition.

\begin{definition}[Average-velocity flow map]\label{def:flowmap-sg}
For $s\le t$, we define $\Psi_{s\to t}:\SE\to\SE$ by
\begin{align}
\Psi_{s\to t}(R,x)
&:=\Big(\,R\,D(s,t),\ \ x+p(s,t)\,\Big)\nonumber\\
&=\Big(\,R\,\exp\!\big((A^{s\to t})^{\wedge}\big),\ \ x+(t-s)\,v^{\mathrm{avg}}(s,t)\,\Big).\nonumber
\end{align}
By construction, $\Psi_{s\to t}(R_s,x_s)=(R_t,x_t)$ along the interpolation path.
\end{definition}

\subsection{The semigroup property}
Having defined the average-velocity flow map, we next show that it forms a semigroup under time composition.
\begin{proposition}[Semigroup property of the average-velocity flow]\label{prop:semigroup}
For all $s\le m\le t$, the family $\{\Psi_{s\to t}\}_{0\le s\le t\le1}$ defined in Definition~\ref{def:flowmap-sg} satisfies
\begin{align}
\Psi_{m\to t}\circ\Psi_{s\to m}=\Psi_{s\to t},
\qquad
\Psi_{t\to t}=\mathrm{id}.
\label{eq:sg-flow}
\end{align}
Equivalently, in terms of the average velocities,
{\small
\begin{align}
&\textbf{(rotation)}\quad
\exp\!\big((A^{s\to m})^{\wedge}\big)\,\exp\!\big((A^{m\to t})^{\wedge}\big)
=\exp\!\big((A^{s\to t})^{\wedge}\big),
\label{eq:sg-rot-identity}\\[2pt]
&\textbf{(translation)}\quad
(m-s)\,v^{\mathrm{avg}}(s,m)+(t-m)\,v^{\mathrm{avg}}(m,t)=(t-s)\,v^{\mathrm{avg}}(s,t).
\label{eq:sg-trans-identity}
\end{align}}
\end{proposition}
\begin{proof}
On $\mathrm{SO}(3)$, the transition operators compose by inserting
$R_m R_m^\top=I$:
\[
D(s,t)=R_s^\top R_t=(R_s^\top R_m)(R_m^\top R_t)=D(s,m)\,D(m,t),
\]
which is Proposition~\ref{prop:af-additivity}. (Equivalently, this is the
multiplicativity of the time-ordered exponential
$\mathcal{T}\exp(\int_s^t)=\mathcal{T}\exp(\int_s^m)\,\mathcal{T}\exp(\int_m^t)$,
so the statement holds for an arbitrary velocity field, not only for the geodesic
path.) Substituting $D(a,b)=\exp((A^{a\to b})^{\wedge})$ gives
\eqref{eq:sg-rot-identity}. On the flat translation factor, displacements add,
$x_t-x_s=(x_t-x_m)+(x_m-x_s)$, which is \eqref{eq:sg-trans-identity}. Combining the
two factors and using the product law \eqref{eq:dec_met} gives
$\Psi_{m\to t}(\Psi_{s\to m}(R,x))=(R\,D(s,m)D(m,t),\,x+p(s,m)+p(m,t))
=(R\,D(s,t),\,x+p(s,t))=\Psi_{s\to t}(R,x)$, i.e.\ \eqref{eq:sg-flow}.
\end{proof}

\begin{remark}[Where $\mathrm{SO}(3)$ curvature enters]\label{rmk:sg-bch}
The composition \eqref{eq:sg-rot-identity} is \emph{exact} and independent of
curvature: it is the associativity of group multiplication. Curvature appears only
if one tries to collapse the product of exponentials into a \emph{single} Lie-algebra
sum. By the Baker--Campbell--Hausdorff (BCH) formula,
\begin{align}
A^{s\to t}
&=\log\!\Big(\exp\!\big((A^{s\to m})^{\wedge}\big)\exp\!\big((A^{m\to t})^{\wedge}\big)\Big)^{\vee}\nonumber\\
&=A^{s\to m}+A^{m\to t}+\tfrac12\,A^{s\to m}\times A^{m\to t}+\cdots, \nonumber 
\end{align}
so the naive additive law $A^{s\to t}=A^{s\to m}+A^{m\to t}$ holds only when the two
axes are collinear ($A^{s\to m}\times A^{m\to t}=0$); in general it drops the BCH
cross terms. The translation identity \eqref{eq:sg-trans-identity} has no such
correction because $\mathbb{R}^3$ is abelian. This is the single place where the
rotation branch departs from the Euclidean MeanFlow semigroup.
\end{remark}
\subsection{The semigroup-consistency loss}

Proposition~\ref{prop:semigroup} characterizes the correct average-velocity field
\emph{without} any time derivative: the field is consistent iff its one-step
prediction on $[s,t]$ equals the two-step composition through any intermediate
$m\in[s,t]$. We turn this into a regression objective. Let $f_\theta$ output
$A^{\mathrm{avg}}_\theta(\cdot)$ and $v^{\mathrm{avg}}_\theta(\cdot)$ (in the
endpoint parameterization of Section~\ref{sec:stable-rescale},
$A^{\mathrm{avg}}_\theta(s,t,R_t)=\tfrac1t\log((\hat R_0^\theta)^\top R_t)^\vee$).
Given a triple $s<m<t$, form the intermediate state by stepping back the near
segment $[m,t]$,
\begin{align}
\begin{cases}
&R_m=R_t\exp\!\big(-(t-m)A^{\mathrm{avg}}_\theta(m,t,R_t,x_t)^{\wedge}\big),\\
&x_m=x_t-(t-m)\,v^{\mathrm{avg}}_\theta(m,t,x_t),
\end{cases}
\label{eq:sg-intermediate}
\end{align}
and define the composed (two-step) targets by \eqref{eq:sg-rot-identity}--\eqref{eq:sg-trans-identity}:
\begin{align}
A^{s\to t}_{\mathrm{tgt}}
&:=\log\!\Big(\underbrace{\exp\!\big((m-s)A^{\mathrm{avg}}_\theta(s,m,R_m,x_m)^{\wedge}\big)}_{D(s,m)}
\times\underbrace{\exp\!\big((t-m)A^{\mathrm{avg}}_\theta(m,t,R_t,x_t)^{\wedge}\big)}_{D(m,t)}\Big)^{\!\vee},
\label{eq:sg-rot-target}\\[2pt]
p^{s\to t}_{\mathrm{tgt}}
&:=(m-s)\,v^{\mathrm{avg}}_\theta(s,m,x_m)+(t-m)\,v^{\mathrm{avg}}_\theta(m,t,x_t).
\label{eq:sg-trans-target}
\end{align}
The model regresses its \emph{direct} one-step prediction on $[s,t]$ onto these
stop-gradient targets:
\begin{align}
\mathcal{L}^{\mathrm{semi}}_{\mathrm{rot}}
&=\mathbb{E}_{s<m<t}\,\Big\|(t-s)A^{\mathrm{avg}}_\theta(s,t,R_t,x_t)
-\mathrm{sg}\big(A^{s\to t}_{\mathrm{tgt}}\big)\Big\|_2^2,
\label{eq:sg-rot-loss}\\
\mathcal{L}^{\mathrm{semi}}_{\mathrm{trans}}
&=\mathbb{E}_{s<m<t}\,\Big\|(t-s)v^{\mathrm{avg}}_\theta(s,t,x_t)
-\mathrm{sg}\big(p^{s\to t}_{\mathrm{tgt}}\big)\Big\|_2^2.
\label{eq:sg-trans-loss}
\end{align}
The semigroup constraint alone admits trivial (collapsed) minimizers; it must be
anchored by the $s\to t$ boundary, where \eqref{eq:sg-D} degenerates to the
instantaneous velocity. This boundary term is exactly flow matching:
\begin{align}
\begin{cases}
\mathcal{L}^{\mathrm{bd}}_{\mathrm{rot}}
=\mathbb{E}_t\big\|A^{\mathrm{avg}}_\theta(t,t,R_t,x_t)-\omega_t\big\|_2^2,\\
\mathcal{L}^{\mathrm{bd}}_{\mathrm{trans}}
=\mathbb{E}_t\big\|v^{\mathrm{avg}}_\theta(t,t,x_t)-v_t\big\|_2^2.    
\end{cases}
\label{eq:sg-bd-loss}
\end{align}
The total objective is
\begin{align}
\mathcal{L}^{\mathrm{semi\text{-}MF}}
=\underbrace{\mathcal{L}^{\mathrm{bd}}_{\mathrm{rot}}+\mathcal{L}^{\mathrm{bd}}_{\mathrm{trans}}}_{\text{flow matching (anchor)}}
+\underbrace{\mathcal{L}^{\mathrm{semi}}_{\mathrm{rot}}+\mathcal{L}^{\mathrm{semi}}_{\mathrm{trans}}}_{\text{semigroup consistency}} .
\label{eq:sg-total}
\end{align}
Every term uses only forward evaluations of $f_\theta$ together with
$\exp/\log$ on $\mathrm{SO}(3)$ and addition on $\mathbb{R}^3$: there is no
Jacobian--vector product and no explicit right Jacobian $J$.

\begin{remark}[The normalization scalar must stay outside $\log(\exp\cdot\exp)$]
If one prefers to regress the average velocity $A^{\mathrm{avg}}_\theta(s,t)$ itself
rather than the log-displacement, the target is
$A^{\mathrm{avg}}_{\mathrm{tgt}}(s,t)=\tfrac{1}{t-s}A^{s\to t}_{\mathrm{tgt}}$ with
$A^{s\to t}_{\mathrm{tgt}}$ from \eqref{eq:sg-rot-target}. The scalar
$\tfrac{1}{t-s}$ must remain \emph{outside} $\log(\exp\cdot\exp)$: because the two
segment generators do not commute, absorbing it into the two exponentials would
rescale each segment angle and corrupt the BCH cross term, no longer producing
$\log D(s,t)$. On the translation branch the same normalization is harmlessly
absorbed into the linear weights. This is the finite-interval counterpart of the
observation made for \eqref{eq:af-rot-target}.
\end{remark}
\subsection{Consistency: zero loss implies exact reconstruction}

The boundary term fixes the instantaneous limit and the semigroup term propagates it
to all intervals; together they pin down the unique correct flow map. The following
is the finite (JVP-free) analogue of Proposition~\ref{pro:loss_zero}.

\begin{proposition}[Uniqueness of the semigroup minimizer]\label{prop:sg-consistency}
Suppose $A^{\mathrm{avg}}_\theta$ is continuous and, along the interpolation path,
satisfies the boundary condition $A^{\mathrm{avg}}_\theta(t,t,R_t,x_t)=\omega_t$ for all
$t$ together with the rotation semigroup identity
\begin{align}
&\exp\!\big((A^{s\to m}_\theta)^{\wedge}\big)\,\exp\!\big((A^{m\to t}_\theta)^{\wedge}\big)
=\exp\!\big((A^{s\to t}_\theta)^{\wedge}\big),\qquad \forall\,s\le m\le t,\nonumber 
\end{align}
where $A^{a\to b}_\theta:=(b-a)A^{\mathrm{avg}}_\theta(a,b,R_b,x_b)$. Then
$\exp((A^{s\to t}_\theta)^{\wedge})=R_s^\top R_t$ for all $s\le t$; in particular,
with $s=0,t=1$, $\hat R_0:=R_1\exp(-(A^{0\to1}_\theta)^{\wedge})=R_0$. The analogous
statement holds for translation with $v^{\mathrm{avg}}_\theta(t,t)=v_t$.
\end{proposition}

\begin{proof}
Fix $t$ and set $G(s):=\exp((A^{s\to t}_\theta)^{\wedge})\in\mathrm{SO}(3)$, so
$G(t)=I$. For $h>0$ the semigroup identity gives
$G(s)=\exp((A^{s\to s+h}_\theta)^{\wedge})\,G(s+h)$, hence
\[
G(s+h)=\exp\!\big(-(A^{s\to s+h}_\theta)^{\wedge}\big)\,G(s).
\]
By the boundary condition and continuity,
$A^{s\to s+h}_\theta=h\,A^{\mathrm{avg}}_\theta(s,s+h)=h\,\omega_s+o(h)$, so
$\exp(-(A^{s\to s+h}_\theta)^{\wedge})=I-h\,\omega_s^{\wedge}+o(h)$ and therefore
\[
\frac{d}{ds}G(s)=-\,\omega_s^{\wedge}\,G(s),\qquad G(t)=I.
\]
On the other hand $D(s,t)=R_s^\top R_t$ obeys, using
$\dot R_s=R_s \omega_s^{\wedge}$ and skew-symmetry,
\begin{align}
&\frac{d}{ds}D(s,t)=\big(\partial_s R_s^\top\big)R_t=-\,\omega_s^{\wedge}R_s^\top R_t
=-\,\omega_s^{\wedge}D(s,t),\nonumber\\
&D(t,t)=I.    \nonumber
\end{align}

$G$ and $D(\cdot,t)$ solve the same linear ODE with the same terminal condition, so
by uniqueness $G(s)=D(s,t)=R_s^\top R_t$ for all $s\le t$. Taking $s=0,t=1$ gives
$\exp((A^{0\to1}_\theta)^{\wedge})=R_0^\top R_1$, i.e.\ $\hat R_0=R_0$. For
translation, the additive cocycle $p_\theta(s,t)$ with $p_\theta(t,t)$-derivative
$v^{\mathrm{avg}}_\theta(t,t)=v_t$ integrates to $p_\theta(s,t)=x_t-x_s$, giving
$\hat x_0=x_0$.
\end{proof}

Proposition~\ref{prop:sg-consistency} shows \eqref{eq:sg-total} is a valid training
objective: its global minimizer is the exact average-velocity field, and the two
ingredients are both necessary---the boundary term supplies the instantaneous data
velocity, and the semigroup term is the (curvature-exact) propagation rule that
extends it to every interval.

\subsection{Relation to the differential identity and to $\alpha$-Flow}

\begin{remark}[Equivalence with the right-Jacobian identity]\label{rmk:sg-equiv}
The semigroup loss \eqref{eq:sg-total} and the differential MeanFlow loss
\eqref{eq:mf-so3N-loss} have the same minimizer but realize the right Jacobian
differently. Differentiating the boundary-anchored semigroup identity
\eqref{eq:sg-rot-identity} in $t$ at $s\to t$ reproduces the differential identity
\eqref{eq:omgea_identity_main}, in which $J$ appears explicitly as the Fr\'echet
derivative of $\exp$. In the finite form, $J$ never appears as a separate factor:
its entire effect is absorbed into the BCH series of $\log(\exp\cdot\exp)$ in
\eqref{eq:sg-rot-target}. Concretely, a first-order BCH expansion of
\eqref{eq:sg-rot-target} reinstates $J$ and recovers \eqref{eq:mf-so3N-loss} in
gradient (cf.\ Remark~\ref{rmk:af-limits}). The two formulations are thus
gradient-equivalent to first order, but the finite form is JVP-free and, because all
segment angles are bounded by $\pi$, avoids the $\mathcal{O}(\varepsilon/t^2)$
amplification of the differential target discussed in Section~\ref{sec:alpha-flow}.
\end{remark}

\begin{remark}[$\alpha$-Flow as an instance]\label{rmk:sg-alpha}
The JVP-free $\alpha$-Flow objective of Section~\ref{sec:alpha-flow} is a particular
instantiation of \eqref{eq:sg-total}. Choosing the split point $m=\alpha s+(1-\alpha)t$,
replacing the near segment $[m,t]$ model prediction by the \emph{data} velocity
$\tilde\omega_t=\omega_t$ (so that $D(m,t)=\exp(\delta\,\omega_t^{\wedge})$ with $\delta=\alpha(t-s)$),
and keeping the far segment as the stop-gradient model evaluation, turns
\eqref{eq:sg-rot-target} into the $\alpha$-Flow target \eqref{eq:af-rot-target} and
\eqref{eq:sg-trans-target} into its translation counterpart. The limit $\alpha=1$
makes the far segment vanish and reduces \eqref{eq:sg-total} to the flow-matching
boundary \eqref{eq:sg-bd-loss}; the limit $\alpha\to0$ recovers, via BCH, the
right-Jacobian differential loss (Remark~\ref{rmk:af-limits}). The ``pure'' semigroup
form \eqref{eq:sg-rot-target}--\eqref{eq:sg-trans-target}, using the model on both
segments, corresponds to the Shortcut choice $\tilde\omega_t=u^{A}_\theta(m,t,R_t)$.
\end{remark}
\noindent\textbf{Connection to Riemannian MeanFlow semigroup.}
The semigroup objective \eqref{eq:sg-total} can be viewed as a variant of the semigroup formulation proposed in Riemannian MeanFlow~\cite{woo2026riemannian}: their construction is based on the endpoint (geodesic) loss, while ours is derived from and anchored by the flow-matching boundary condition \eqref{eq:sg-bd-loss}. We provide code for this semigroup objective in our repository; however, in our experiments we found that with a limited training budget (e.g., 10--20k steps) semigroup training yields weaker few-step improvements than our JVP-based MeanFlow loss, and therefore we do not use the semigroup loss in our final training recipe. We include it here for theoretical completeness.

\section{Quaternion Formulation of SE(3)-MeanFlow}
\label{sec:quaternion}

The SE(3)-MeanFlow objectives admit an equivalent formulation with unit quaternions
on $\mathbb{S}^3$, following the rotation parametrization of
ReQFlow~\cite{yue2025reqflow}. We derive it here for completeness; an
implementation is provided in our repository alongside the rotation-matrix
formulation used throughout the paper.

\subsection{Quaternion algebra and conventions}

We follow the standard conventions; see \citet{hanson2005visualizing} for a fuller
treatment. A quaternion is a pair
\begin{align}
q=(q_0,\mathbf{q})\in\mathbb{R}\times\mathbb{R}^3,\qquad \|q\|=1 ,\nonumber
\end{align}
with multiplication and inverse
\begin{align}
&q\otimes p=\big(q_0p_0-\mathbf{q}^\top\mathbf{p},\;
q_0\mathbf{p}+p_0\mathbf{q}+\mathbf{q}\times\mathbf{p}\big),\nonumber\\
&q^{-1}=(q_0,-\mathbf{q}).\nonumber
\end{align}
The unit quaternions form the manifold $\mathbb{S}^3$, a double cover of $\SO$: the
quaternions $q$ and $-q$ represent the same rotation
\begin{align}
&R(q)=\begin{bmatrix}
1-2(y^2+z^2) & 2(xy-wz) & 2(xz+wy)\\
2(xy+wz) & 1-2(x^2+z^2) & 2(yz-wx)\\
2(xz-wy) & 2(yz+wx) & 1-2(x^2+y^2)
\end{bmatrix}, \qquad q=(w,x,y,z).\nonumber
\end{align}
Recovering $q$ from $R$ is standard (trace-based) up to this sign; we fix it by
enforcing $w\ge0$.

\paragraph{Exponential and logarithm.}
We define $\mathrm{Exp}:\mathbb{R}^3\to\mathbb{S}^3$ with the half-angle absorbed,
\begin{align}
\mathrm{Exp}(\omega):=\Big(\cos\tfrac{\|\omega\|}{2},\;
\sin\tfrac{\|\omega\|}{2}\,\tfrac{\omega}{\|\omega\|}\Big),
\qquad
\mathrm{Exp}(0)=(1,\mathbf{0}),
\label{eq:q-exp}
\end{align}
and let $\log$ be its inverse on the hemisphere $w\ge0$, so
$\|\log(q)\|\le\pi$. With this convention the covering map is compatible with
the matrix exponential,
\begin{align}
R\big(\mathrm{Exp}(\omega)\big)=\exp(\omega^{\wedge}),
\label{eq:q-cover}
\end{align}
so $\omega$ carries the same axis-angle meaning as in the rest of the paper. (The
ReQFlow subsection above writes the same object as $\exp(\tfrac12\omega)$; the
factor of two is bookkeeping in where the half-angle is placed.) The tangent space
is $T_q\mathbb{S}^3=\{v\in\mathbb{R}^4:q^\top v=0\}$.

\paragraph{Kinematics.}
Let $\omega(t)\in\mathbb{R}^3$ be the \emph{body-frame} angular velocity, i.e.\ the
same quantity as in \eqref{eq:so3_ode}. The left-trivialized kinematics are
\begin{align}
\dot q_t=\tfrac12\,q_t\otimes(0,\omega(t)),
\label{eq:q-kinematics}
\end{align}
whose solution is the time-ordered exponential in $\mathbb{S}^3$,
\begin{align}
&q_t=q_s\otimes\mathcal{T}\mathrm{Exp}\Big(\int_s^t\omega(\tau)\,d\tau\Big),
\qquad\text{and}\qquad
q_t=q_s\otimes\mathrm{Exp}\big((t-s)\,\omega\big)\ \text{ if }\omega\text{ is constant}.
\nonumber
\end{align}
The half-angle in \eqref{eq:q-exp} absorbs the factor $\tfrac12$ of
\eqref{eq:q-kinematics}, so no stray factors of two appear below.

\paragraph{Interpolation and data velocity.}
Mirroring Appendix~\ref{sec:bg-fm}, for a data rotation $q_0$ and a prior rotation
$q_1$ (sign-aligned so that $q_0^\top q_1\ge0$, which selects the shorter arc) the
conditional path and its constant body angular velocity are
\begin{align}
q_t=q_0\otimes\mathrm{Exp}\big(t\,\omega_t\big),
\qquad
\omega_t:=\log\big(q_0^{-1}\otimes q_1\big)\in\mathbb{R}^3 .
\label{eq:q-interp}
\end{align}

\subsection{Average velocity and the MeanFlow identity}

Exactly as in \eqref{eq:avg_omega}, we define the average angular velocity over
$[s,t]$ through the time-ordered exponential. Writing the state as
$z_t=(x_t,q_t)\in\mathbb{R}^3\times\mathbb{S}^3$ to make the conditioning explicit,
\begin{equation}
\mathrm{Exp}\big((t-s)\,\omega^{\mathrm{avg}}(s,t,x_t,q_t)\big)
:=\mathcal{T}\mathrm{Exp}\Big(\int_s^t\omega(\tau,q_\tau,x_\tau)\,d\tau\Big)
= q_s^{-1}\otimes q_t.
\label{eq:q_avg_def}
\end{equation}

\begin{proposition}[Quaternion MeanFlow identity]
\label{prop:q-identity}
Differentiating \eqref{eq:q_avg_def} with respect to $t$ yields
{\small\begin{align}
&J\big((t-s)\,\omega^{\mathrm{avg}}\big)
\Big(\omega^{\mathrm{avg}}(s,t,x_t,q_t) +(t-s)\tfrac{d}{dt}\omega^{\mathrm{avg}}(s,t,x_t,q_t)\Big)=\omega_t ,
\label{eq:q_meanflow_identity}
\end{align}}
with $J$ the \emph{same} right Jacobian \eqref{eq:right_jacobian} as in the
rotation-matrix formulation, and
\begin{align}
\frac{d}{dt}\omega^{\mathrm{avg}}
=\frac{\partial}{\partial t}\omega^{\mathrm{avg}}
+\big\langle\nabla_q\omega^{\mathrm{avg}},\dot q_t\big\rangle
+\big\langle\nabla_x\omega^{\mathrm{avg}},\dot x_t\big\rangle .
\nonumber
\end{align}
\end{proposition}

\begin{proof}
The covering map $\mathbb{S}^3\to\SO$ of \eqref{eq:q-cover} is a two-to-one Lie group
homomorphism and a local diffeomorphism, and under the identification
$(0,\omega)\leftrightarrow\omega^{\wedge}$ it induces the identity on Lie algebras
($\mathfrak{su}(2)\cong\mathfrak{so}(3)$). Applying it to \eqref{eq:q_avg_def}
returns \eqref{eq:avg_omega} verbatim, so every step of the proofs of
Propositions~\ref{pro:Omega_identity} and~\ref{pro:dA/dt} carries over unchanged.
\end{proof}

\begin{remark}[Quaternions do not remove the right Jacobian]
\label{rmk:q-no-free-lunch}
Since the two Lie algebras are isomorphic, the quaternion identity
\eqref{eq:q_meanflow_identity} is \emph{identical} to the rotation-matrix identity
\eqref{eq:omgea_identity_main}, right Jacobian included. Changing the
representation therefore cannot simplify the MeanFlow target; it can only change
the numerical conditioning of $\mathrm{Exp}$, $\log$ and the
Jacobian--vector product used to evaluate it.
\end{remark}

\subsection{Parametrization, loss and inference}

As in Section~\ref{sec:stable-rescale}, the network is an endpoint predictor
returning $(\hat q_0^\theta,\hat x_0^\theta)$, converted to the average-velocity
head by
\begin{align}
\omega^{\mathrm{avg}}_\theta=\tfrac1t\,\log\big((\hat q_0^\theta)^{-1}\otimes q_t\big)
\iff
\hat q_0^\theta=q_t\otimes\mathrm{Exp}\big(-t\,\omega^{\mathrm{avg}}_\theta\big).
\nonumber
\end{align}
Regressing the left-hand side of \eqref{eq:q_meanflow_identity} onto the data
velocity $\omega_t$ of \eqref{eq:q-interp} gives the rotation loss
\begin{align}
\mathcal{L}^{\mathbb{S}^3}_{\mathrm{rot}}
=\mathbb{E}_{s<t,\,q(z_0,z_1)}\Big[\big\|\,J\big((t-s)\omega^{\mathrm{avg}}_\theta\big)\Big(\omega^{\mathrm{avg}}_\theta+(t-s)\,\mathrm{sg}\big(\tfrac{d}{dt}\omega^{\mathrm{avg}}_\theta\big)\Big)-\omega_t\,\big\|_2^2\Big],
\label{eq:q-loss}
\end{align}
the exact analogue of \eqref{eq:mf-se3-loss}; the translation branch and the
$\alpha$-Flow variant of Section~\ref{sec:alpha-flow} are unchanged, since neither
touches the rotation representation. Inference steps backwards by
$q_s\gets q_t\otimes\mathrm{Exp}\big(-(t-s)\,\omega^{\mathrm{avg}}_\theta\big)$.

\subsection{Practical note}
\label{sec:q-practical}

Empirically the quaternion implementation consistently underperformed the
rotation-matrix one, and we use the latter for every result in this paper. Two
causes are consistent with what we observed. First, the surrounding pipeline
(dataset, IPA trunk, auxiliary losses, evaluation) operates on rotation matrices,
so the quaternion path inserts repeated matrix$\,\leftrightarrow\,$quaternion
conversions whose error accumulates. Second, and more specific to MeanFlow, the
double cover $q\sim-q$ means the enforced sign convention $w\ge0$ can flip along a
trajectory; the endpoint prediction $\hat q_0^\theta$ then jumps discontinuously,
and the forward-mode derivative $\tfrac{d}{dt}\omega^{\mathrm{avg}}_\theta$ in
\eqref{eq:q-loss} is corrupted at exactly those steps. Rotation matrices have no
such ambiguity. Together with Remark~\ref{rmk:q-no-free-lunch}---the quaternion
target is not analytically simpler---this left no reason to prefer the quaternion
branch, and we report no quaternion-based results.

\section{Training Curriculum: Pre-training and Post-training}
\label{sec:training-details}

Our final model is obtained in three phases: a two-stage \emph{pre-training}
that builds a strong multi-step backbone generator, followed by a
\emph{post-training} (rectification) stage that sharpens few-step generation.
All phases share the same network (Section~\ref{sec:architecture}) and the same
decoupled $\SE=\SO\times\mathbb{R}^3$ objective; they differ only in which
consistency target is used and in how the data--noise pairs are coupled.

\subsection{Stage 1: JVP-free $\alpha$-Flow warm-up}
\label{sec:stage1}

Training the endpoint-parameterized MeanFlow target directly from initialization
is fragile: as discussed in Sections~\ref{sec:stable-rescale}
and~\ref{sec:alpha-flow}, the differential rotation target amplifies head error
as $\mathcal{O}(\varepsilon/t^2)$ and requires a forward-mode derivative (JVP)
through the IPA trunk. We therefore warm up with the JVP-free $\alpha$-Flow
objective of Section~\ref{sec:alpha-flow}, annealing the consistency-step ratio
$\alpha$ from a cap $\alpha_{\max}=1$ toward a floor $\alpha_{\min}=0.1$ along a
logistic schedule:
\begin{align}
&\alpha(k)=
\begin{cases}
\alpha_{\max}, & k\le k_s,\\[2pt]
\alpha_{\min}+(\alpha_{\max}-\alpha_{\min})\,\sigma_\gamma(\tau_k), & k_s<k<k_e,\\[2pt]
\alpha_{\min}, & k\ge k_e,
\end{cases}\\
&\qquad \text{where }\sigma_\gamma(\tau)=\frac{1}{1+e^{\,\gamma(\tau-\tfrac12)}},\nonumber\\
&\tau_k=\frac{k-k_s}{k_e-k_s},
\label{eq:alpha-schedule}
\end{align}
where $k$ is the optimizer step and $\tau_k\in[0,1]$ is the normalized progress
through the anneal window. Here $k_s$ is the hold phase---the ratio is pinned at
the cap for the first $k_s$ steps---$k_e$ is the step by which the floor is
reached, and $\gamma$ sets the steepness of the logistic transition, centered at
the window midpoint $k_s+\tfrac12(k_e-k_s)$. We use $k_s\!=\!2000$,
$k_e\!=\!150\mathrm{k}$, and $\gamma\!=\!8$.

By Remark~\ref{rmk:af-limits}, $\alpha=1$ reduces the objective exactly to flow
matching~\eqref{eq:fw-loss}, so during the hold phase training is anchored to the
data velocity and cannot collapse; decreasing $\alpha$ thereafter progressively
injects the few-step average-velocity consistency that underlies fast sampling.
The anneal was scheduled over $150\mathrm{k}$ steps, but we stopped Stage~1 early
at $100\mathrm{k}$---where $\alpha\!\approx\!0.29$---because validation
designability had already plateaued; $\alpha$ therefore never reaches its floor
$\alpha_{\min}=0.1$ during pre-training. This checkpoint initializes Stage~2.
Full settings are in Table~\ref{tab:train-config}.

\paragraph{Self-conditioning.}
In all phases we use $50\%$ self-conditioning, following
FrameFlow~\cite{frameflow}/ReQFlow~\cite{yue2025reqflow}. With probability $0.5$ a
detached preliminary forward pass produces an endpoint prediction
$\hat x_0^\theta$ whose pairwise $\mathrm{C}_\alpha$ distogram is appended as an
extra edge feature to a second, gradient-carrying forward pass; the remaining
half of each batch passes zeros in that slot. At inference the previous step's
endpoint prediction is used as the self-conditioning input
(Algorithm~\ref{alg:stable-inference-tmin}). This improves sample quality at no
additional inference cost.

\paragraph{Global-OT coupling.}
Both pre-training stages instantiate the mini-batch OT coupling of
Appendix~\ref{sec:bg-ot}, but solve a single plan over the \emph{entire} global
batch rather than one plan per rank. At each step a
\texttt{torch.distributed.all\_gather} assembles all $M=\sum_{r=1}^{W} B_r$
backbones across the $W$ ranks, and the independent data--noise pairing is
replaced by the assignment
\begin{align}
\pi^\star=\arg\min_{\pi\in\mathcal{S}_M}\sum_{k=1}^{M}
c\big(z_0^{k},\,z_1^{\pi(k)}\big),
\label{eq:global-ot}
\end{align}
with $c$ the decoupled $\SE^N$ transport cost of Appendix~\ref{sec:bg-ot} and
$(\lambda_R,\lambda_x)=(0.5,0.5)$. The exact Hungarian assignment is solved on
rank~0 with POT~\cite{flamary2021pot} and $\pi^\star$ broadcast back; the
$O(M^3)$ solve is negligible against a forward/backward pass at these batch
sizes. Solving \eqref{eq:global-ot} once over all $M$ backbones uses the full
pairing information, whereas $W$ per-rank plans search a block-diagonal subset
of $\mathcal{S}_M$ and under-use it by a factor $W$; empirically the global
coupling roughly halves the early-plateau time on SCOPe. 
Since mini-batch OT is a biased estimate of the true coupling whose bias decreases with batch size, the global plan yields a straighter induced path than $W$ independent per-rank plans at the same per-GPU memory cost.

Post-training disables
it (its self-reflow pairs are already matched; Section~\ref{sec:posttrain}).

\subsection{Stage 2: endpoint $+$ MeanFlow objective}
\label{sec:stage2}

Starting from the Stage-1 checkpoint, we continue with the endpoint-anchored
small-$t$ MeanFlow objective for a maximum budget of $10\mathrm{k}$ steps, and
select the released checkpoint at $6.1\mathrm{k}$ steps by validation
designability. We retain the $50\%$ self-conditioning and the global-OT coupling
of Stage~1; the only changes are the switch to the differential MeanFlow loss and
the added endpoint anchor. The loss combines three terms (see \eqref{eq:stage2-loss}). The endpoint distance is defined as: 
{\small
\begin{align}
&\mathcal{L}_{\mathrm{end}}
:=\sum_{i=1}^N\Big[\,
D_{\SO}^{2}\!\big(R^{i}_{\theta,0},\,R^{i}_{0}\big)
+\big\|x^{i}_{\theta,0}-x^{i}_{0}\big\|_2^{2}
\,\Big],\label{eq:endpoint-loss}\\
&D_{\SO}(R,Q):=\big\|\log(R^{\top}Q)^{\vee}\big\|_2
=\arccos\!\Big(\tfrac{\operatorname{tr}(R^{\top}Q)-1}{2}\Big),\nonumber
\end{align}
}

\begin{remark}
We emphasize that the split of the training budget between the $\alpha$-Flow
warm-up and the MeanFlow stage is a design choice rather than a requirement: the
two can be freely re-balanced (e.g.\ a shorter warm-up with a longer MeanFlow
phase). Because protein-scale training is expensive, we do not sweep this ratio
at the protein scale; on the low-dimensional $\mathrm{SO}(3)^2$ benchmark
(Appendix~\ref{sec:toy}) different warm-up/MeanFlow ratios reach the few-step
regime with little difference, so we simply report the single configuration used
for the released checkpoint.
\end{remark}


\paragraph{Loss normalization and stability.}
Our optional loss-magnitude normalization (\texttt{our\_loss\_norm}) is
\emph{disabled} in all phases. For stability we instead rely on three
mechanisms: a global gradient-norm clip at $1.0$, per-residue rotation- and
translation-loss clamps ($50$ and $5$), and the MeanFlow near-data reweighting
$1/\max(t,t_\varepsilon)^2$ ($t_\varepsilon=0.1$, Stage~2 onward). The rotation
and translation branches are weighted by $(w_R,w_x)=(1.0,1.0)$.

\begin{table}[t]
\centering
\begingroup
\small
\setlength{\tabcolsep}{8pt}
\renewcommand{\arraystretch}{1.15}
\begin{tabular}{@{}lccc@{}}
\toprule
Term & Raw & Weight & Weighted \\
\midrule
MeanFlow rot ($\mathcal{L}_{\mathrm{rot}}$)     & $3.87$ & $\times 0.05$ & $0.193$ \\
MeanFlow trans ($\mathcal{L}_{\mathrm{trans}}$) & $2.63$ & $\times 0.05$ & $0.131$ \\
\textbf{MeanFlow total}                         & $\mathbf{6.50}$ & & $\mathbf{0.324}$ \\
\midrule
Endpoint rot                                    & $0.136$ & $\times 1.0$ & $0.136$ \\
Endpoint trans                                  & $0.060$ & $\times 1.0$ & $0.060$ \\
\textbf{Endpoint total}                         & $\mathbf{0.196}$ & & $\mathbf{0.196}$ \\
\bottomrule
\end{tabular}
\endgroup
\caption{Per-term loss magnitudes at $\lambda_{\mathrm{MF}}=0.05$ (median over
$24$ SCOPe training batches; \texttt{our\_loss\_norm} disabled, right-Jacobian
inverse form). The raw MeanFlow residual is $\sim\!33\times$ larger than the endpoint
term---owing to the near-data reweighting
$1/\max(t,t_\varepsilon)^2$---and $\lambda_{\mathrm{MF}}=0.05$ brings the two to
the same order ($\approx 1.6{:}1$). }
\label{tab:loss-magnitudes}
\end{table}
\paragraph{Total objective.}
The Stage-2 loss is the weighted sum
\begin{align}
\lambda_{\mathrm{end}}\,\mathcal{L}_{\mathrm{end}}
+\lambda_{\mathrm{MF}}\,\big(\mathcal{L}^{\text{MF}}_{\mathrm{rot}}+\mathcal{L}^{\text{MF}}_{\mathrm{trans}}\big)
+\lambda_{\mathrm{aux}}\,\mathbf{1}[t<0.75]\,\mathcal{L}_{\mathrm{aux}},
\label{eq:stage2-loss}
\end{align}
with $(\lambda_{\mathrm{end}},\lambda_{\mathrm{MF}},\lambda_{\mathrm{aux}})=(1.0,\,0.05,\,2.0)$;
the endpoint term anchors the model and stabilizes the MeanFlow loss, and all
remaining settings are listed in Table~\ref{tab:train-config}. Here
$\lambda_{\mathrm{MF}}=0.05$ is a scale normalizer rather than a tuned trade-off.
Because the loss-magnitude normalization is disabled in all phases (see above),
the terms of \eqref{eq:stage2-loss} enter at their raw magnitudes, and these
differ by more than an order of magnitude: the small-$t$ MeanFlow residuals
\eqref{eq:smallt_losses} carry the near-data reweighting
$1/\max(t,t_\varepsilon)^{2}$, which reaches $t_\varepsilon^{-2}=100$ for
$t\le t_\varepsilon=0.1$, whereas the endpoint term is unweighted. With
$\lambda_{\mathrm{MF}}=0.05$ the weighted MeanFlow contribution is comparable in
magnitude to $\lambda_{\mathrm{end}}\mathcal{L}_{\mathrm{end}}$ throughout
training---a ratio of roughly $1.6{:}1$ (Table~\ref{tab:loss-magnitudes}). The
small numerical value therefore reflects the scale of the raw residual, not a
small role in the objective; equivalently, one may enable the loss normalization
and set $\lambda_{\mathrm{MF}}=1$.

\begin{table*}[t]
\centering
\small
\setlength{\tabcolsep}{4pt}
\renewcommand{\arraystretch}{1.03}
\begin{tabular}{lcc}
\toprule
Hyperparameter & Stage 1 ($\alpha$-Flow) & Stage 2 and Post-train (endpt+MF)  \\
\midrule
Objective & $\alpha$-Flow (JVP-free) & endpoint\,+\,MeanFlow \\
Initialization & from scratch & Stage-1 ckpt, Stage-2 ckpt \\
Maximum Steps & $150$k & $10$k, $10$k \\
\midrule
$\alpha$ schedule (cap$\to$floor) & $1\!\to\!0.1$ & --- \\
\quad $(k_s,k_e,\gamma)$ & $(2\text{k},150\text{k},8)$ & --- \\
Loss weight & (alpha,aux)=1,1 &(Endpoint,MF,aux)=1,0.05,2 \\ 
MeanFlow config & --- & (version,jmode)=(mf,$J^{-1}$)\\ 
Auxiliary time gate & $t<0.75$ & $t<0.75$ \\
loss normalization & off & off \\
Rotation-loss clamp & 50 & $50$ \\
Trans-loss clamp & $5$ & $5$ \\
\midrule
Min.\ time $t_{\min}$ & $10^{-6}$ & $10^{-6}$ \\
Denom.\ clamp $t_\varepsilon$ & $0.1$ & $0.1$ \\
Rotation schedule & exp (rate $10$) & exp (rate $10$) \\
$(w_R,\,w_x)$ & $(1.0,1.0)$ & $(1.0,1.0)$ \\
$s\mid t$ sampler & $\mathcal{U}[t_{\min},t]$ & $\mathcal{U}[t_{\min},t]$ \\
$s{=}t$ anchor fraction & $0$ & $0$ \\
Two-time input ($s<t$) & yes & yes \\
Self-conditioning & yes ($50\%$) & yes ($50\%$) \\
JVP & --- (JVP-free) & single combined \\
Coupling & global SE(3) OT & global SE(3) OT \\
\midrule
Optimizer, lr & Adam, $10^{-4}$ & Adam, $10^{-4}$ \\
Gradient clip & $0$ (none) & $1.0$ \\
Gradient accumulation & $2$ & $2$ \\
Batch cap ($N^2_{\max}$) & $\propto 1/N^2$ ($5{\times}10^{5}$) & $\propto 1/N^2$ ($5{\times}10^{5}$) \\
Devices & $4\times$ H100 (80GB) GPU & $4\times$ H100 (80GB) GPU \\
\bottomrule
\end{tabular}
\caption{Training configuration for the two pre-training stages. Post-training
(self-reflow) reuses the Stage-2 column verbatim, changing only the data
coupling: the OT coupling is replaced by the model's own deterministic
self-reflow pairs, with OT disabled (Section~\ref{sec:posttrain}).}
\label{tab:train-config}
\end{table*}

\subsection{Post-training: self-reflow rectification}
\label{sec:posttrain}

Few-step designability is further improved by a rectification stage that
replaces the random data--noise coupling with the model's own deterministic
transport, following the rectified-flow strategy of ReQFlow~\cite{yue2025reqflow}.

\paragraph{Self-reflow dataset.}
Using the Stage-2 model, for each residue length $N\in[60,128]$ we draw prior
samples $z_1\sim p_{\mathrm{prior}}$ and integrate the model to obtain coupled
pairs $(z_1,\hat z_0)$ (noise $\to$ generated backbone). Each generated backbone
is scored with the self-consistency pipeline (ProteinMPNN followed by ESMFold),
and we retain only \emph{designable} samples ($\mathrm{scRMSD}<2\text{\AA}$). We
keep $50$ designable pairs per length---over-sampling $100$ generations per
length to meet the quota---yielding $69\times50=3450$ coupled pairs over the
SCOPe length range.

\paragraph{Rectification objective.}
We then continue training from the Stage-2 checkpoint on these pairs using the
\emph{identical} Stage-2 objective~\eqref{eq:stage2-loss}, but with the
independent coupling $q(z_0,z_1)$ replaced by the model-induced deterministic
coupling $(z_1,\hat z_0)$ and with the OT re-coupling disabled (the pairs are
already transport-consistent). Because the target transport is now (approximately)
straight, the average-velocity field it must match is closer to constant along
each trajectory, which is precisely the regime in which few-step MeanFlow
sampling is exact. Checkpoints are saved densely and selected on validation
designability and secondary-structure content.

\subsection{Checkpoint selection}
Rather than selecting the final model by a grid search over training
hyper-parameters, we retain the training checkpoint according to a fixed,
pre-specified rule on validation-time generation quality. During training we
periodically sample unconditional backbones at two inference budgets,
$T{=}100$ and $T{=}20$ integration steps, and monitor (i) the C$\alpha$--C$\alpha$
bond-geometry validity ($\mathrm{ca\_ca}$) and (ii) the secondary-structure
composition (helix and strand fractions). We keep the checkpoint that
(a)~attains $\mathrm{ca\_ca} > 0.97$ at \emph{both} $T{=}100$ and $T{=}20$, so that
backbones remain geometrically valid in both the high- and low-step regimes;
(b)~keeps the strand fraction near or above $0.2$, matching the reference distribution and
guarding against strand collapse; and (c)~among the checkpoints satisfying
(a)--(b), maximizes the helix fraction. This criterion was fixed before final
evaluation and applied identically across all runs.

\section{Architecture}
\label{sec:architecture}

Our network is the ReQFlow~\cite{yue2025reqflow} trunk left \emph{unchanged},
augmented with only two additions required to turn a single-time flow-matching
model into a two-time MeanFlow model:
(i)~a shared two-time embedding for the MeanFlow inputs $(t,s)$
(\S\ref{sec:time-embed}), and
(ii)~a per-block, zero-initialized AdaLN-Zero gate on each block's rigid update
(\S\ref{sec:blockadaln}). The trunk itself---the invariant point attention (IPA)
stack of \citet{jumper2021highly} as instantiated by ReQFlow, comprising per
block an IPA module, a two-layer sequence transformer, node/edge transitions and
a backbone frame update---is not modified. The full model has $\sim\!16.8$M
parameters, of which the AdaLN additions account for only $0.33\%$
($\sim\!55$k); it is therefore $\sim\!26\times$ smaller than RMF
(\citealt{woo2026riemannian}, $437$M) while remaining checkpoint-compatible with
ReQFlow/QFlow (Table~\ref{tab:trunk-config}).

\subsection{Backbone trunk}
The trunk follows the ReQFlow configuration: node and edge embedding sizes
$256$ and $128$, IPA hidden width $c_\text{hidden}=128$ with $8$ attention heads,
$8$ query/key points and $12$ value points, a $4$-head / $2$-layer sequence
transformer, and $6$ IPA blocks. The input embedder is ReQFlow's:
a self-conditioned pairwise distogram ($22$ bins), a relative-position encoding
($\mathrm{relpos}\_k=64$), and a diffuse-mask embedding. Because the trunk and
embedder are inherited unchanged, our model loads directly from our pretrained checkpoint (in stage 2).

\begin{table}[t]
\centering
\small
\begin{tabular}{lccc}
\toprule
& \textbf{Ours} & ReQFlow & RMF \\
\midrule
node embed size          & 256          & 256 & 768 \\
edge embed size          & 128          & 128 & 384 \\
\(c_\text{hidden}\)      & 128          & 128 & 48 \\
\# attention heads       & 8            & 8   & 16 \\
\# query/key points      & 8            & 8   & --- \\
\# value points          & 12           & 12  & 16 \\
\# IPA blocks            & 6            & 6   & 16 \\
seq.\ tfmr (heads/layers)& 4 / 2        & 4 / 2 & 12 / 3 \\
per-block AdaLN gate     & \textbf{yes} & no  & yes \\
\midrule
total parameters         & \(\sim\!16.8\)M & \(\sim\!16.8\)M & \(\sim\!437\)M \\
\bottomrule
\end{tabular}
\caption{Trunk configuration. ``Ours'' is the ReQFlow~\cite{yue2025reqflow}
trunk with the only architectural addition being a per-block AdaLN-Zero gate
($+0.055$M, $0.33\%$ of parameters). Both are $\sim\!26\times$ smaller than
RMF~\cite{woo2026riemannian} while recovering most of the structural fidelity.}
\label{tab:trunk-config}
\end{table}

\subsection{Block-wise AdaLN-Zero gate}
\label{sec:blockadaln}
The sole structural change to the trunk is a per-block gate on the rigid update.
Each IPA block emits a six-vector backbone update (three rotational, three
translational) that is composed onto the running frame. We modulate this update
with an AdaLN-Zero~\cite{peebles2023scalable} scale,
\[
\mathrm{update}_b \;\leftarrow\; \big(1+\gamma_b(\mathbf{c}_b)\big)\odot \mathrm{update}_b ,
\]
where the gate $\gamma_b$ is produced by a per-block linear head from the
conditioning vector
$\mathbf{c}_b=[\,\mathbf{h}_b\;\|\;\mathbf{c}^{(t)}\,]$, formed by concatenating
(i)~the block's current \emph{node features} $\mathbf{h}_b$---which are
$\SE$-invariant and state-aware (they depend on the current
$(R_t,x_t)$ through IPA)---and (ii)~a joint time/interval embedding
$\mathbf{c}^{(t)}$ built from $t$ and $h=t-s$. Every gate head is
zero-initialized, so at initialization $\gamma_b\equiv 0$, the update is scaled
by $1$, and the model is \emph{forward-identical} to the unmodified ReQFlow
trunk. This lets the two-time MeanFlow conditioning enter each block's geometry
update per residue and per step, without perturbing the pretrained trunk at the
start of training. In implementation the gate is attached via a forward hook, so
the trunk's forward code is untouched.

\subsection{Shared two-time embedding}
\label{sec:time-embed}
A MeanFlow prediction is conditioned on the pair $(t,s)$ with $0\le s\le t\le 1$;
we let $h=t-s$. Both $t$ and $h$ are passed through a \emph{single} shared
sinusoidal-plus-MLP embedder $\phi(\cdot)$ and concatenated in a fixed order,
$\mathbf{c}^{(t)}=[\phi(t)\,\|\,\phi(h)]$. Concatenation preserves the $(t,h)$
ordering; the same embedding feeds both the node/edge input features and the
per-block AdaLN gate of \S\ref{sec:blockadaln}. This is the only place the
single-time ReQFlow trunk is made time-pair aware.
\begin{figure}[t]
  \centering
  \includegraphics[width=0.85\linewidth]{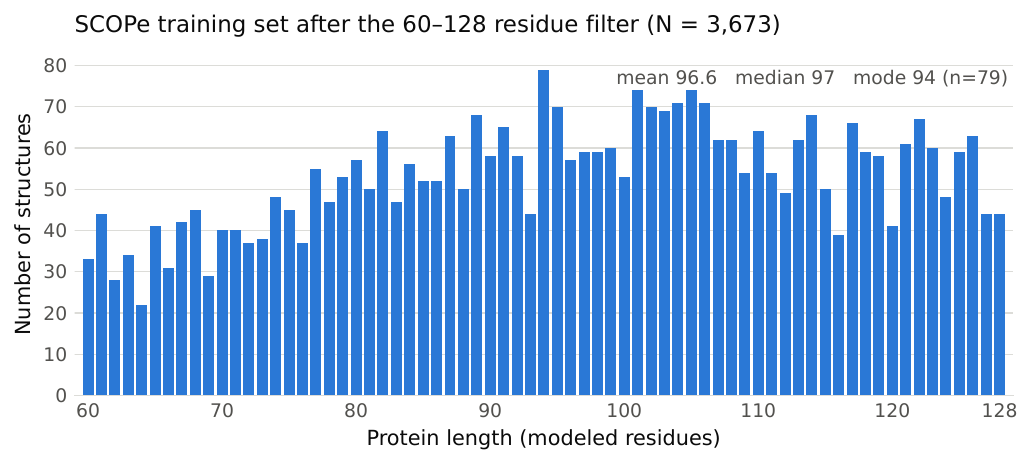}
  \caption{Residue-length distribution of the SCOPe backbone dataset used in our experiments (3,673 backbones with lengths in $[60,128]$).}
  \label{fig:scope-length-hist}
\end{figure}

\section{Extra Results/Settings in the Protein Experiment}
\label{sec:protein-extra}
\begin{figure}[t]
  \centering
  \includegraphics[width=\linewidth]{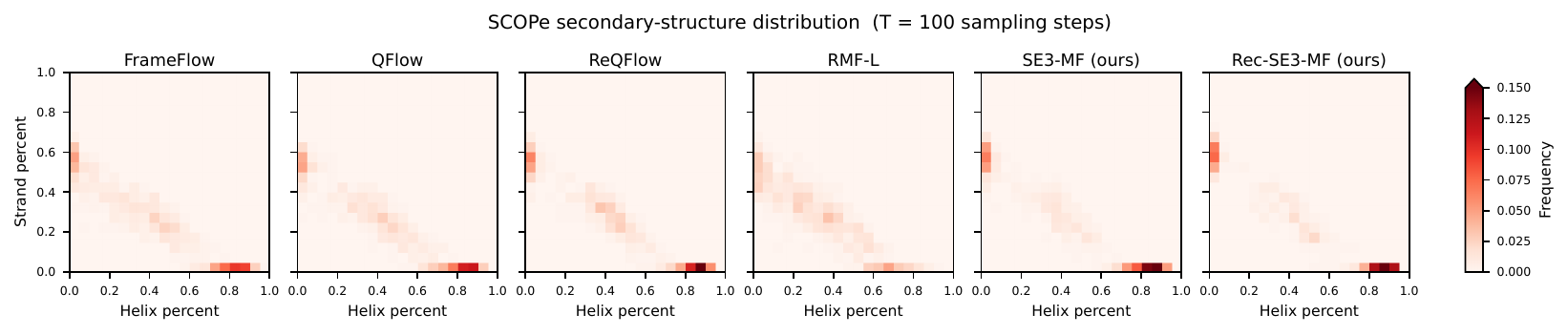}\\  \includegraphics[width=\linewidth]{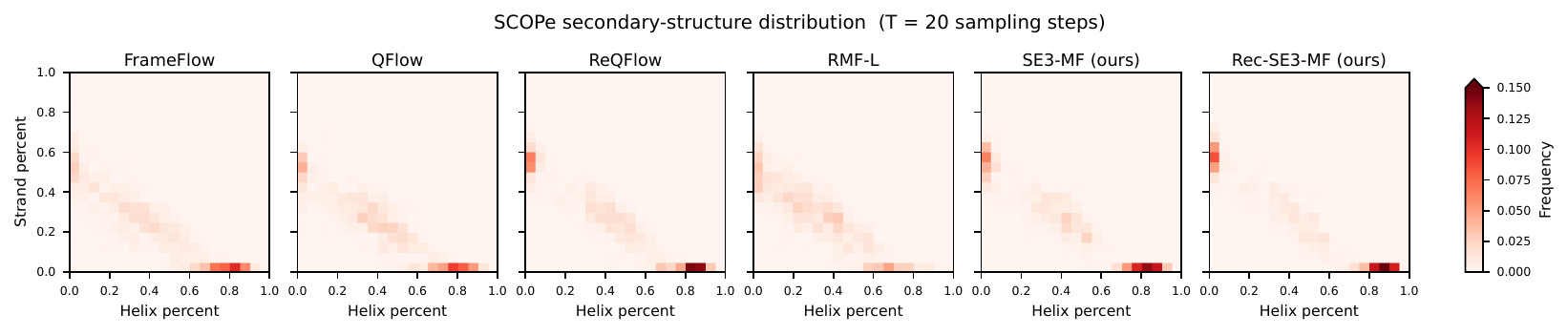}
  \caption{Secondary-structure statistics (top: $T{=}100$, bottom: $T{=}20$).}
  \label{fig:ss-scope-T100-T20}
\end{figure}

We provide additional visualizations in this section. Figure~\ref{fig:ss-scope-T100-T20} reports the joint distribution of
per-sample helix and strand content (via mdtraj DSSP) for the SCOPe-generated
backbones of each method, at $T{=}100$ (top) and $T{=}20$ (bottom). Each panel is
a $2$-D histogram over the helix-fraction ($x$) and strand-fraction ($y$) plane,
sharing a common color scale; the anti-diagonal reflects the intrinsic
helix--strand trade-off within a single chain. Across all methods, the mass
concentrates along this trade-off with an additional helix-rich mode, consistent
with the SCOPe length regime. Our models (SE3MF and RecSE3MF) recover a
distribution comparable to the flow-matching baselines rather than collapsing to
a single motif. Importantly, the distribution is largely preserved when the
sampling budget is reduced from $T{=}100$ to $T{=}20$: the few-step regime does
not visibly distort the secondary-structure statistics, indicating that the
average-velocity consistency learned during training transfers to aggressive
step reduction without a mode shift. 
\paragraph{Visualization of SCOPe.}
Figure~\ref{fig:scope-length-hist} shows the residue-length distribution of the SCOPe backbones used in our experiments.

\paragraph{Random seed in evaluation.}
All evaluations use a fixed random seed for exact reproducibility: for each sampled
backbone the RNG seed is set deterministically to $12345 + 10^{5}\,T + 10^{3}\,N + i$,
where $T$ is the number of sampling steps, $N$ the chain length, and $i$ the sample
index within that length.

\paragraph{How large is the Jacobian correction on the training distribution?}
Figure~\ref{fig:omega-identity-numerical} shows that dropping the Jacobian breaks the
identity by $\mathcal{O}(t-s)$ on an analytic path. To check that this regime is
actually visited during training, we evaluate
$\rho=\|(J^{-1}(A^{s\to t}_\theta)-I)\,\omega_t\|_2/\|\omega_t\|_2$---the relative
change the Jacobian makes to the regression target of \eqref{eq:mf-se3-loss2}---over
$1.2\times10^{5}$ residue-level samples from held-out SCOPe backbones under the
Stage-2 time sampler (SE(3)-OT coupling, Table~\ref{tab:jacobian-rho}). The correction
grows monotonically with the interval: median $\rho$ is $0.8\%$ for $t-s<0.1$, $7.2\%$
for $t-s\in[0.25,0.5)$, and $28\%$ for $t-s>0.75$ (P90: $2.6\%$, $19\%$, $58\%$). Since
the sampler draws $s$ across the full range below $t$, $22\%$ of training targets fall
in the regime where the correction exceeds $10\%$. The right Jacobian is therefore not
a negligible term on the distribution the objective is trained over, even though the
intervals queried at inference ($t-s=1/T$) are short enough that $J^{-1}\approx I$
there; the log-map approximation $J^{-1}:=I$ used by RMF-PT~\cite{zhong2026riemannian}
discards a correction of this size during training.
\begin{table}[t]
\centering
\begingroup
\small
\begin{tabular}{@{}lrcccc@{}}
\toprule
\multirow{2}{*}{$t-s$} & \multirow{2}{*}{$n$}
 & \multicolumn{2}{c}{$\rho$} & \multicolumn{2}{c}{$\phi$ (rad)} \\
\cmidrule(lr){3-4}\cmidrule(lr){5-6}
 & & median & P90 & median & P90 \\
\midrule
$[0,0.1)$      & 39{,}044 & 0.008 & 0.026 & 0.061 & 0.175 \\
$[0.1,0.25)$   & 32{,}212 & 0.031 & 0.077 & 0.283 & 0.500 \\
$[0.25,0.5)$   & 29{,}502 & 0.072 & 0.187 & 0.589 & 0.993 \\
$[0.5,0.75)$   & 14{,}957 & 0.165 & 0.392 & 0.891 & 1.497 \\
$[0.75,1.0]$   &  4{,}320 & 0.280 & 0.581 & 0.841 & 1.582 \\
\midrule
overall        & 120{,}035 & 0.033 & 0.194 & 0.289 & 0.955 \\
\bottomrule
\end{tabular}
\endgroup
\caption{The inverse right Jacobian's relative change to the rotation target,
$\rho=\|(J^{-1}(A^{s\to t}_\theta)-I)\,\omega_t\|_2/\|\omega_t\|_2$, and the driving
angle $\phi=\|A^{s\to t}_\theta\|_2$ (rad), binned by the interval $t-s$ over
$1.2\times10^{5}$ residue-level samples from held-out SCOPe backbones under the
Stage-2 (SE(3)-OT) sampler. Here $\rho$
and $\phi$ are summarized by their median and 90th percentile (P90) within each bin.}
\label{tab:jacobian-rho}
\end{table}

\subsection{Training budget: steps versus epochs}

We report the training budget in optimizer steps, not epochs: the QFlow/ReQFlow loader
(from FrameFlow) oversamples each epoch to $\approx\!460$ steps, while ours does one pass
at $\approx\!35$ steps---a $\sim\!13\times$ gap on the identical SCOPe set that makes epoch
counts incomparable. Importantly, both QFlow/ReQFlow and our method use the same hardware
(4$\times$80GB GPUs) and the same batch cap (see $N^2_{\max}$ 
in Table~\ref{tab:train-config}), so each optimizer step processes the same amount of data; the
only difference is that under our loader each sample is seen once per epoch, whereas
QFlow/ReQFlow repeats samples multiple times due to oversampling (which depends on the
number of GPUs). Steps are directly comparable; as a check, the QFlow checkpoint's
$90{,}160$ steps equal its reported $195$ epochs~\citep{yue2025reqflow}.

\begin{table*}[t]
\begingroup
\small
\setlength{\tabcolsep}{4pt}
\begin{tabular}{@{}lcccccccc@{}}
\toprule
\multirow{2}{*}{Objective}
& \multirow{2}{*}{\shortstack{Samp.\\Steps}}
& \multirow{2}{*}{\shortstack{Train\\Steps}}
& \multicolumn{3}{c}{Designability}
& \multirow{2}{*}{Div.\,$\downarrow$}
& \multicolumn{2}{c}{Secondary structure} \\
\cmidrule(lr){4-6}\cmidrule(lr){8-9}
& & & Frac.\,$\uparrow$ & scRMSD\,$\downarrow$ & scTM\,$\uparrow$ & & Helix / Strand & ca-valid \\
\midrule

FM $+$ semigroup (RMF)
& 100 & 6.8k & 0.886 & $1.285{\pm}0.962$ & $0.892{\pm}0.078$ & $\mathbf{0.376}$ & $0.498 / 0.199$ & $0.976$ \\
& 50  & -    & 0.880 & $1.337{\pm}1.027$ & $0.889{\pm}0.079$ & $\mathbf{0.369}$ & $0.466 / 0.221$ & $0.975$ \\
& 20  & -    & 0.781 & $1.665{\pm}1.366$ & $0.856{\pm}0.101$ & $\mathbf{0.370}$ & $0.474 / 0.197$ & $0.975$ \\
& 10  & -    & 0.533 & $2.825{\pm}2.327$ & $0.765{\pm}0.148$ & $\mathbf{0.368}$ & $0.446 / 0.171$ & $0.983$ \\

\midrule

endpoint $+$ semigroup
& 100 & 6.2k & 0.852 & $1.391{\pm}1.321$ & $0.887{\pm}0.095$ & 0.396 & $0.549 / 0.160$ & $0.971$ \\
& 50  & -    & 0.835 & $1.405{\pm}1.245$ & $0.885{\pm}0.093$ & 0.391 & $0.521 / 0.180$ & $0.972$ \\
& 20  & -    & 0.813 & $1.534{\pm}1.489$ & $0.873{\pm}0.096$ & 0.392 & $0.561 / 0.152$ & $0.977$ \\
& 10  & -    & 0.712 & $1.975{\pm}1.783$ & $\mathbf{0.835{\pm}0.120}$ & 0.389 & $0.551 / 0.147$ & $0.984$ \\

\midrule

\textbf{endpoint $+$ MeanFlow (ours)}
& 100 & 6.1k & $\mathbf{0.936}$ & $\mathbf{1.106{\pm}0.870}$ & $\mathbf{0.909{\pm}0.067}$ & 0.415 & $0.536 / 0.200$ & $0.971$ \\
& 50  & -    & $\mathbf{0.906}$ & $\mathbf{1.246{\pm}1.190}$ & $\mathbf{0.902{\pm}0.078}$ & 0.413 & $0.517 / 0.213$ & $0.971$ \\
& 20  & -    & $\mathbf{0.867}$ & $\mathbf{1.369{\pm}1.144}$ & $\mathbf{0.883{\pm}0.087}$ & 0.408 & $0.505 / 0.210$ & $0.970$ \\
& 10  & -    & $\mathbf{0.728}$ & $\mathbf{1.940{\pm}1.741}$ & $0.832{\pm}0.118$ & 0.402 & $0.497 / 0.193$ & $0.975$ \\

\bottomrule
\end{tabular}
\endgroup
\caption{\textbf{Controlled objective ablation on SCOPe.} All rows start from the
same Stage-1 $\alpha$-Flow checkpoint (step~$100$k) and share the IPA trunk, data,
coupling, self-conditioning and optimizer; only the Stage-2 consistency target is
swapped. \emph{Train Steps} counts Stage-2 steps only. The semigroup rows
instantiate the objectives of~\citet{woo2026riemannian} inside our pipeline; they
are \emph{not} a reproduction of that work, whose hyperparameters, model size and
training budget we do not replicate. Best per performance metric within each
sampling-step group in bold.}
\label{tab:obj_ablation}
\end{table*}

\subsection{Ablation Study: A Controlled Comparison of Consistency Objectives}
\label{sec:abl-controlled}

The RMF rows of Table~\ref{tab:scope_comparison} come from the released
checkpoint, which uses a different trunk ($437$M versus our $16.8$M) and a much
larger budget ($598$k versus our Stage-2 budget), so that comparison does not
isolate the objective. Table~\ref{tab:obj_ablation} removes the confound: all
three rows start from the \emph{same} Stage-1 $\alpha$-Flow checkpoint and share
the trunk, data, coupling, self-conditioning and optimizer, so the only variable
is the Stage-2 consistency target.

Under this matched budget the semigroup target is a weaker self-consistency
signal than the differential MeanFlow target. The FM$+$semigroup variant tracks
plain flow matching closely at every budget --- its designable fraction
$0.886/0.880/0.781/0.533$  is within noise of QFlow's
$0.885/0.870/0.778/0.559$ in Table~\ref{tab:scope_comparison}, including the same
collapse at $T{=}10$, which is the signature of the instantaneous-velocity
parameterization. The endpoint$+$semigroup variant is stronger in the aggressive
regime but still falls short of MeanFlow at every budget. Swapping in our
MeanFlow target, on the same trunk and from the same initialization, recovers the
few-step regime ($0.867$ at $T{=}20$, $0.728$ at $T{=}10$).

We read this as a matter of training budget rather than of correctness: the two
objectives share a minimizer (Remark~\ref{rmk:sg-equiv}), but the semigroup
constraint propagates the boundary condition only through the interval triples it
happens to sample, whereas the differential target imposes the same consistency
pointwise at every $(s,t)$. The semigroup route therefore appears to need a
substantially longer schedule to reach the few-step regime --- consistent with
RMF's $\sim\!598\mathrm{k}$ steps, and with our observation in
Appendix~\ref{sec:semigroup} that semigroup training gave weaker few-step gains
than the JVP-based loss at $10$--$20$k steps.
\begin{table*}[t]
\begingroup
\setlength{\tabcolsep}{4pt}
\hspace{-3em}
{\footnotesize
\begin{tabular}{lccccccccc}
\toprule
Method & Steps & Fraction $\uparrow$ & scRMSD $\downarrow$ & scTM $\uparrow$ & Diversity $\mathrm{TM}\downarrow$ & Novelty $\mathrm{TM}\downarrow$ & $\alpha$-Helix  & $\beta$-Strand  & $\mathrm{C}\alpha$-valid $\uparrow$\\
\midrule
\multirow{4}{*}{V-QFlow} 
& 100 & 0.990 & $0.654 \pm 0.446$ & $0.957 \pm 0.039$ & 0.384 & 0.846 & $0.740\pm0.260$ & $0.091\pm0.200$ & $0.997\pm0.006$\\
& 50  & 0.980 & $0.702 \pm 0.423$ & $0.952 \pm 0.045$ & 0.373 & 0.837 & $0.714\pm0.268$ & $0.109\pm0.206$ & $0.987\pm0.016$\\
& 20  & 0.964 & $0.869 \pm 0.697$ & $0.937 \pm 0.053$ & 0.348 & 0.803 & $0.710\pm0.245$ & $0.099\pm0.184$ & $0.892\pm0.069$\\
& 10  & 0.903 & $1.160 \pm 1.183$ & $0.908 \pm 0.088$ & 0.332 & 0.772 & $0.687\pm0.227$ & $0.093\pm0.164$ & $0.807\pm0.081$\\
\midrule
\multirow{4}{*}{V-ReQFlow}
& 100 & 0.990 & $0.675 \pm 0.318$ & $0.955 \pm 0.033$ & 0.403 & 0.851 & $0.808\pm0.199$ & $0.050\pm0.154$ & $0.984\pm0.016$\\
& 50  & 0.993 & $0.726 \pm 0.625$ & $0.951 \pm 0.036$ & 0.396 & 0.840 & $0.791\pm0.216$ & $0.064\pm0.170$ & $0.974\pm0.021$\\
& 20  & 0.975 & $0.855 \pm 0.638$ & $0.936 \pm 0.047$ & 0.383 & 0.812 & $0.785\pm0.217$ & $0.064\pm0.167$ & $0.940\pm0.068$\\
& 10  & 0.957 & $1.083 \pm 0.854$ & $0.910 \pm 0.065$ & 0.377 & 0.782 & $0.788\pm0.199$ & $0.052\pm0.142$ & $0.883\pm0.121$\\
\bottomrule
\end{tabular}
}
\endgroup
\caption{Evaluation results for V-QFlow and V-ReQFlow. These methods use a different data/evaluation pipeline from the other baselines, so we report them separately.}

\label{tab:vqflow_vreqflow_appendix}
\end{table*}

\subsection{Additional results for V-QFlow and V-ReQFlow}
V-QFlow and V-ReQFlow~\cite{yue2025reqflow} do not provide checkpoints trained on SCOPe; therefore, the scores in Table~\ref{tab:vqflow_vreqflow_appendix} are computed using the released checkpoints trained on PDB. Because this evaluation uses a different training dataset and pipeline from the SCOPe-based baselines, the results are not directly comparable; we thus report them in this separate section. Notably, while the designable fraction is close to 0.99, the helix fraction is high, whereas the strand fraction is around 0.1 or lower. This imbalance is biologically concerning because $\beta$-strands typically assemble into $\beta$-sheets via inter-strand hydrogen bonding and often depend on non-local sequence interactions; correspondingly, strand prediction is generally harder than helix prediction, and a very low strand fraction may indicate that the model fails to capture such long-range constraints~\cite{pauling1951structure,zhang2015secondary}.

\section{Computational Resources}
\label{sec:compute}

All experiments are conducted on a single GPU node with the following
configuration. Each node is equipped with $4\times$ NVIDIA H100 80\,GB HBM3 GPUs
and dual AMD EPYC 9654 96-core CPUs. Training uses all four GPUs with PyTorch
Lightning distributed data-parallel (DDP), $14$ CPU worker threads per GPU process
($56$ cores in total), and gradient accumulation of $2$; the effective batch is
set by a length-based sampler with cap $N^2_{\max}=5\times10^{5}$ residues$^2$
(Table~\ref{tab:train-config}). 

\end{document}